\documentclass{article}

\PassOptionsToPackage{numbers, compress}{natbib}

\usepackage[main, final]{neurips_2026}

\usepackage[utf8]{inputenc} %
\usepackage[T1]{fontenc}    %
\usepackage[colorlinks=true, linkcolor=blue!80!black, citecolor=blue!80!black, urlcolor=blue!80!black]{hyperref} %
\usepackage{url}            %
\usepackage{booktabs}       %
\usepackage{amsfonts}       %
\usepackage{nicefrac}       %
\usepackage{microtype}      %
\usepackage{xcolor}         %

\usepackage{csquotes}

\newcommand{\Assref}[1]{Assumption~\ref{#1}}

\usepackage{algorithm}
\usepackage{algorithmic}

\usepackage{amsmath}
\usepackage{amssymb}
\usepackage{mathtools}
\usepackage{amsthm}
\usepackage[capitalize,noabbrev]{cleveref}

\theoremstyle{plain}
\newtheorem{theorem}{Theorem}[section]
\newtheorem{proposition}[theorem]{Proposition}
\newtheorem{lemma}[theorem]{Lemma}
\newtheorem{corollary}[theorem]{Corollary}
\theoremstyle{definition}
\newtheorem{definition}[theorem]{Definition}

\newtheorem{assumption}[theorem]{Assumption}
\theoremstyle{remark}
\newtheorem{remark}[theorem]{Remark}

\usepackage[utf8]{inputenc}
\usepackage[T1]{fontenc}
\usepackage[ngerman,english]{babel}

\usepackage{comment}
\usepackage{enumerate}
\usepackage[shortlabels]{enumitem}
\setlist[itemize]{leftmargin=1.2em,labelsep=0.5em}
\setlist[enumerate]{leftmargin=2.4em,labelsep=0.5em}

\usepackage{amsmath}
\usepackage{mathtools}
\usepackage{amsfonts}
\usepackage{mathrsfs}
\usepackage{amssymb}
\usepackage{dsfont} %
\usepackage{amsthm}

\usepackage{lipsum}

\usepackage{aligned-overset}
\usepackage{nicefrac}

\usepackage{todonotes}

\usepackage{booktabs}
\usepackage{xcolor}

\definecolor{darkblue}{rgb}{0,0,.5}
\definecolor{darkorange}{rgb}{0.8, 0.4, 0}
\definecolor{okabeItoBlack}{RGB}{0,0,0}
\definecolor{okabeItoOrange}{RGB}{230,159,0}
\definecolor{okabeItoSkyBlue}{RGB}{86,180,233}
\definecolor{okabeItoBluishGreen}{RGB}{0,158,115}
\definecolor{okabeItoYellow}{RGB}{240,228,66}
\definecolor{okabeItoBlue}{RGB}{0,114,178}
\definecolor{okabeItoVermillion}{RGB}{213,94,0}
\definecolor{okabeItoReddishPurple}{RGB}{204,121,167}
\definecolor{grayETH}{rgb}{0.33, 0.33, 0.33}
\definecolor{blueETH}{rgb}{0.129, 0.361, 0.686}
\definecolor{purpleETH}{rgb}{0.506757, 0.111486, 0.381757}
\definecolor{greenETH}{rgb}{0.398406, 0.454183, 0.14741}
\definecolor{bronzeETH}{rgb}{0.5568627, 0.40392, 0.07450980}

\definecolor{tabblue}{RGB}{31,119,180}
\definecolor{tabred}{RGB}{214,39,40}

\definecolor{redETH}{rgb}{0.597173, 0.226148, 0.176678}

\newcommand{\markerBlue}[1]{\textcolor{tabblue}{#1}}
\newcommand{\markerOrange}[1]{\textcolor{orange}{#1}}

\newcommand{\gray}[1]{\textcolor{grayETH}{#1}}

\newcommand{\okSkyBlue}[1]{\textcolor{okabeItoSkyBlue}{#1}}
\newcommand{\okBluishGreen}[1]{\textcolor{okabeItoBluishGreen}{#1}}

\newcommand{\okReddishPurple}[1]{\textcolor{okabeItoReddishPurple}{#1}}

\usepackage{subcaption}
\usepackage{bbm}
\usepackage{tcolorbox}

\usepackage{afterpage}
\usepackage{wrapfig}

\newcommand{\nfr}{\nicefrac}

\def\RR{{\mathbb{R}}}
\def\PP{{\mathbb{P}}}
\def\NN{{\mathbb{N}}}

\def\SS{{\mathcal{S}}}
\def\AA{{\mathcal{A}}}
\def\VV{{\mathcal{V}}}

\DeclareMathOperator{\OOTildeSym}{\widetilde{\mathcal{O}}}
\DeclareMathOperator{\OOSym}{\mathcal{O}}
\newcommand{\OOTilde}[1]{{\OOTildeSym\!\left(#1\right)}}
\newcommand{\OO}[1]{{\OOSym\!\left(#1\right)}}

\newcommand{\card}[1]{\left|#1\right|}

\def\epsilon{{\varepsilon}}

\newcommand{\Exp}[1]{{ \mathbb{E}}\!\left[#1\right]}

\newcommand{\Expu}[2]{{\mathbb{E}}_{#1}\!\left[#2\right]}

\newcommand{\innerEmpty}{\,\cdot\,}

\DeclarePairedDelimiter{\paren}{(}{)}
\DeclarePairedDelimiter{\brac}{[}{]} %
\DeclarePairedDelimiter{\cbrac}{\{}{\}} %
\DeclarePairedDelimiter{\abs}{\lvert}{\rvert}
\DeclarePairedDelimiter{\norm}{\lVert}{\rVert}
\DeclarePairedDelimiterX{\IP}[2]{\langle}{\rangle}{#1, #2} %

\DeclarePairedDelimiter{\plus}{[}{]_{+}}

\newcommand{\pen}[1]{{\plus{#1}^2}}
\newcommand{\penGenNot}{{\mathcal{L}}}
\newcommand{\penGen}[1]{{\penGenNot\paren{#1}}}

\newcommand{\eqdef}{\coloneqq}
\newcommand{\define}{\coloneqq}
\newcommand{\defineRev}{\eqqcolon}

\DeclareMathOperator*{\argmax}{arg\,max}
\DeclareMathOperator*{\argmin}{arg\,min}

\DeclareMathOperator{\range}{range}     %
\DeclareMathOperator{\Image}{Im}        %

\DeclareMathOperator{\proj}{proj}

\usepackage{pdflscape}

\usepackage{tikz}
\usetikzlibrary{
 positioning,
 calc,
 arrows.meta,
 decorations.pathreplacing,
 fit,
 backgrounds,
 patterns
}

\newcommand{\imgwithtopcaption}[2]{
 \begin{tikzpicture}
  \node[inner sep=0] (img) {\includegraphics[width=\linewidth]{#1}};
  \node[
   anchor=north,
   font=\footnotesize\bfseries,
   fill=white,
   fill opacity=0.85,
   text opacity=1,
   rounded corners=1pt,
   inner sep=2pt
  ] at ($(img.north)+(0,-2pt)$) {#2};
 \end{tikzpicture}
}

\DeclareRobustCommand*{\circled}[1]{\tikz[baseline={([yshift=-0.7ex]char.center)}]{
  \node[shape=circle,draw,inner sep=0.5pt,line width=0pt,
   font=\small\bfseries] (char) {#1};}}

\newcommand{\circledPurple}[1]{\okReddishPurple{\circled{#1}}}

\newif\ifshowchanges
\showchangesfalse %
\definecolor{changePurple}{RGB}{196,123,228}

\def\uu{{u}}

\def\xx{{\theta}}

\def\xxOpt{\xx^{*}}

\def\xxZero{\xx^{(0)}}
\def\xxK{\xx^{(k)}}

\def\xxN{\xx^{(N)}}

\def\xxKnext{\xx^{(k+1)}}
\def\xxNnext{\xx^{(N+1)}}

\def\gK{g^{(k)}}
\def\gHatK{\hat{g}^{(k)}}
\def\gKnext{g^{(k+1)}}
\def\gHatKnext{\hat{g}^{(k+1)}}

\def\xiK{\xi^{(k)}}

\def\rhoHat{{\hat{\rho}}}

\def\fTilde{{\tilde{f}}}

\def\lambdaHat{{\widehat{\lambda}}}
\def\nuOpt{{\nu^{*}}}

\def\epsLam{\epsilon_\lambda}
\def\piOpt{\pi^{\star}}

\def\optLabel{\mathsf{opt}}
\newcommand{\penOpt}[1]{P^\optLabel_{#1}}

\def\phi{{\varphi}}
\def\refPol{{\mathrm{ref}}}

\def\vv{{v}}
\def\ww{{w}}
\def\ss{{s}}

\def\NN{{\mathbb{N}}}
\def\MM{{\mathbb{M}}}
\def\HH{{\mathcal{H}}}

\def\setN{{[N]}}

\newcommand{\setm}[1]{[#1]}

\def\phiKpen{{\phi^{(k)}_{\mathrm{pen}}}}
\newcommand{\phiK}[1]{{\phi^{(k)}_{#1}}}

\newcommand{\KL}[2]{{\klDiv\left(#1 \lVert #2 \right)}}

\def\XX{{\Theta}}
\def\YY{{\mathcal{Y}}}

\def\DDX{{\mathcal{D}_\XX}}
\def\UU{{\Lambda}}
\def\DDU{{\mathcal{D}_\UU}}
\def\DDY{{\mathcal{D}_\YY}}
\def\DDUsq{{\mathcal{D}^2_\UU}}

\def\xxApp{{\theta}}

\def\xxAppHat{\hat{\xxApp}}

\def\xxAppOpt{\xxApp^{*}}

\def\xxAppZero{\xxApp^{(0)}}
\def\xxAppK{\xxApp^{(k)}}

\def\xxAppN{\xxApp^{(N)}}

\def\xxAppKnext{\xxApp^{(k+1)}}
\def\xxAppNnext{\xxApp^{(N+1)}}

\def\xxAppHatK{\xxAppHat^{(k)}}
\def\xxAppHatKnext{\xxAppHat^{(k+1)}}
\def\XXApp{{\Theta}}
\def\DDXAppsq{{\mathcal{D}^2_\XXApp}}
\def\DDXApp{{\mathcal{D}_\XXApp}}
\def\UUApp{{\Lambda}}
\def\DDUApp{{\mathcal{D}_\UUApp}}
\def\DDUAppsq{{\mathcal{D}^2_\UUApp}}

\def\yy{{\tilde{\theta}}}

\def\lamAppLipSq{{\mu^2_c}}

\def\Fopt{{F^*}}
\def\dOpt{{d^*}}
\def\ConstrViolZero{{\mathcal{V}_0}}

\def\Tinner{{T_{\mathrm{in}}}}

\newcommand{\TinnerUp}[1]{T_{\mathrm{in}}^{\mathrm{#1}}}
\def\Ttotal{{T_{\mathrm{tot}}}}
\newcommand{\TtotalUp}[1]{T_{\mathrm{tot}}^{\mathrm{#1}}}
\def\nonSmoothLabel{{\mathrm{n}\text{-}sm}}
\def\smoothLabel{{sm}}

\def\Oracle{{\mathcal{A}}}
\def\epsTilde{{\tilde{\epsilon}}}
\def\epsin{{\epsilon_{\mathrm{in}}}}

\DeclareMathOperator{\BothPPPM}{IPPPM}

\newcommand{\algo}{PGP}

\def\thetaHat{{\hat{\theta}}}
\def\thetaK{{\theta^{(k)}}}

\def\thetaN{{\theta^{(N)}}}

\def\gApprox{{\hat{g}}}
\def\Fmax{{F_{\mathsf{UB}}}}
\def\RegMax{{\mathcal{R}_{\mathsf{UB}}}}

\def\aa{{a}}
\def\ss{{s}}
\def\rew{{r}}

\def\Reg{{\mathcal{R}}}

\def\rMax{{\mathcal{R}_{\max}}}
\def\cMax{{c_{\max}}}
\def\gradPsi{{\ell_{\psi}}}
\def\jacPsi{{L_{\psi}}}
\def\DDLam{{\mathcal{D}_\Lambda}}
\def\DDLamSq{{\mathcal{D}^2_\Lambda}}
\def\lamLip{{\mu_\lambda}}
\def\lamLipSq{{\mu^2_\lambda}}
\def\wVec{{\textbf{w}}}
\DeclareMathOperator{\Obs}{Obs}
\def\Tmax{{T_{\max}}}
\newcommand*\Dd{\mathrm{d}}
\newcommand*\dt{\: \mathrm{d}t}

\def\muLipSq{{\mu^2_c}}

\newcommand{\figext}{.pdf}

\usepackage[hyperpageref]{backref}

\title{Global Optimality for Constrained Exploration \\ via Penalty Regularization}
\author{%
  Florian Wolf\thanks{Corresponding author: \texttt{fwolf@caltech.edu}}\,\,\footnotemark[2] \And Ilyas Fatkhullin\footnotemark[3] \And Niao He\footnotemark[3]\\
  \AND \textnormal{\footnotemark[2]\,\,California Institute of Technology \footnotemark[3]\,\,ETH Zürich 
  \& ETH AI Center} %
}

\begin{document}

\maketitle

\begin{abstract}
 Efficient exploration is a central problem in reinforcement learning and is
often formalized as maximizing the entropy of the state-action occupancy
measure. While unconstrained maximum-entropy exploration is relatively well
understood, real-world exploration is often constrained by safety,
resource, or imitation requirements. This constrained setting is
particularly challenging because entropy maximization lacks additive
structure, rendering Bellman-equation-based methods inapplicable. Moreover,
scalable approaches require policy parameterization, inducing non-convexity
in both the objective and the constraints. To our knowledge, the only prior
model-free policy-gradient approach for this setting under general policy
parameterization is due to Ying et al. 2025
\cite{yingPolicybasedPrimalDualMethods2025}. Unfortunately, their
guarantees are limited to weak regret and ergodic averages, which do not
imply that the final output is a single deployable policy that is
near-optimal and nearly feasible. In this work we take a different approach
to this problem, and propose Policy Gradient Penalty (PGP) method, a
single-loop policy-space method that enforces general convex
occupancy-measure constraints via quadratic-penalty regularization. PGP
constructs pseudo-rewards that yield gradient estimates of the penalized
objective, subsequently exploiting the classical Policy Gradient Theorem.
We further establish the regularity of the penalized objective, providing
the smoothness properties needed to justify the convergence of PGP.
Leveraging hidden convexity and strong duality, we then establish global
last-iterate convergence guarantees, attaining an $\epsilon$-optimal
constrained entropy value with $\epsilon$-bounded constraint violation
despite policy-induced non-convexity. We validate PGP through ablations on
a grid-world benchmark and further demonstrate scalability on two
challenging continuous-control tasks.\footnote{Code available under:
 \href{https://github.com/Flo-Wo/penalty-based-constrained-exploration
 }{\texttt{https://github.com/Flo-Wo/penalty-based-constrained-exploration}}.}

\end{abstract}

\section{Introduction}
Efficient exploration is a fundamental challenge in reinforcement learning
(RL), particularly in the absence of an external reward signal, where the
goal is to acquire broad and informative coverage of the environment rather
than optimize task-specific returns. A principled line of work formulates
pure exploration as maximizing the entropy of the state-action occupancy
measure induced by a policy, encouraging diverse visitation of the
environment \cite{hazanProvablyEfficientMaximum2019}. This viewpoint has
led to substantial progress in entropy-based exploration, with both
theoretical and practical advances
\cite{muttiTaskAgnosticExplorationPolicy2021, tiapkinFastRatesMaximum2023}.

In realistic systems, however, exploration is inherently constrained not
only by safety limits and resource budgets, but crucially by imitation
requirements, where policies must remain close to a given prior or expert
behavior while still exploring novel regions of the state-action space
\cite{dulac-arnoldChallengesRealworldReinforcement2021}. Despite their
importance, existing approaches to constrained entropy-based exploration
are largely heuristic \cite{yangCEMConstrainedEntropy2023,
 tiboniDomainRandomizationEntropy2024}. The model-based approaches of
\citet{agarwalConcaveUtilityReinforcement2022a,
 baiAchievingZeroConstraint2023a} operate with explicit representations of
the state-action occupancy measure, which limits their applicability in
large-scale or continuous settings where function approximation is
required. In contrast, \citet{yingPolicybasedPrimalDualMethods2025} provide
only ergodic and weak-regret guarantees under restrictive assumptions.

A natural way to model safety, cost, and imitation requirements is to
express them as convex constraints on the state-action occupancy measure,
yielding a unified and expressive formulation for both finite- and
infinite-horizon Markov decision processes. Achieving scalability, however,
necessitates policy parameterization, which obscures this convex structure
and eliminates the additive Bellman structure underlying standard
dynamic-programming and actor-critic methods. Together with enforcing
constraints under stochastic, truncated trajectory sampling with biased
gradients, this leads to highly challenging optimization problems.

\paragraph{Contributions.} We develop a penalty framework for constrained maximum-entropy exploration
that operates directly in policy space while retaining the hidden-convex
structure of occupancy-measure optimization.\footnote{ In
 \Cref{sec:AppendixNonSmoothAndSmoothPenaltyApproach}, we further extend our
 framework beyond RL to constrained hidden-convex optimization problems with
 possibly non-smooth objectives and constraints, which could be of
 independent interest.} Our contributions are:
\begin{enumerate}[start=1,label={(\bfseries C\arabic*)}]
 \item We introduce \algo{}, a \emph{single-loop, primal-only} Policy Gradient
       Penalty method for constrained maximum-entropy exploration that operates
       directly in policy space and enforces non-convex constraints via a
       quadratic penalty. By explicitly utilizing the RL structure of the problem,
       we construct \emph{pseudo-rewards}, enabling single batch penalty gradient
       estimation.
 \item We prove \emph{global, non-asymptotic last-iterate convergence} of \algo{}
       under biased stochastic gradients, despite the non-convexity induced by
       policy parametrization. The analysis quantifies the smoothness of the
       penalty formulation, a key property needed to establish global convergence,
       and provides a principled choice of the penalty regularization parameter.
 \item In the finite-MDP setting covered by our theory, we validate PGP on a grid
       world, demonstrating robustness to penalty tuning and gradient noise and
       providing practical guidance for choosing the penalty parameter $\beta$.
       Beyond our theory, continuous state-action experiments demonstrate
       scalability and stable late-iterate feasibility.
\end{enumerate}

\section{Related Work}
\paragraph{Entropy Maximization.}
Maximizing the entropy of the state-action occupancy measure was introduced
by \citet{hazanProvablyEfficientMaximum2019} as a principled formulation of
reward-free exploration. Subsequent work has focused on improving
computational efficiency and practical performance, including
non-parametric entropy estimation
\cite{muttiIntrinsicallyMotivatedApproachLearning2020,
 muttiTaskAgnosticExplorationPolicy2021}, representation learning
\cite{seoStateEntropyMaximization2021,yaratsReinforcementLearningPrototypical2021},
and successor-based approaches \cite{liuAPSActivePretraining2021,
 liuBehaviorVoidUnsupervised2021}. On the theoretical side,
\citet{tiapkinFastRatesMaximum2023} established fast convergence rates for
maximum-entropy exploration, while \citet{zamboniHowExploreBelief2024,
 zamboniLimitsPureExploration2024} extended entropy-based exploration to
partially observable Markov decision processes (MDPs). These works focus
almost exclusively on the unconstrained setting. While they exploit the
convex structure of the entropy objective in the occupancy-measure space,
they do not address constraints and therefore cannot guarantee feasibility
in safety- or resource-constrained environments.

\paragraph{Constrained Entropy-Based Exploration.}
Several recent works have studied constrained entropy maximization
empirically. \citet{yangCEMConstrainedEntropy2023} propose constrained
entropy maximization for safe, task-agnostic exploration, and
\citet{tiboniDomainRandomizationEntropy2024} apply entropy maximization
under performance chance constraints to domain randomization for
sim-to-real transfer in robotics. Related ideas appear in transfer and
imitation settings \cite{blessingInformationMaximizingCurriculum2023a,
 yangReinforcementLearningGuided2023,
 kimAcceleratingReinforcementLearning2023}, observer gain tuning
\cite{klinkTrackingControlSpherical2023b, lutterInductiveBiasesMachine2021}
and system identification under performance constraints
\cite{wolfInterpretableEfficientDatadriven2025,
 zolmanSINDyRLInterpretableEfficient2025}. However, these methods are
heuristic in nature and do not provide (global) convergence guarantees,
even in tabular settings.

More broadly, constrained and safe RL has been extensively studied under
the Constrained Markov decision process (CMDP) formulation
\cite{altmanConstrainedMarkovDecision2021}, for instance, using primal-dual
methods \cite{dingNaturalPolicyGradient2020,
 efroniExplorationExploitationConstrainedMDPs2020,
 yingDualApproachConstrained2022, liuPolicyOptimizationConstrained2022,
 weiTripleQModelFreeAlgorithm2022, ghoshProvablyEfficientModelFree2022,
 dingPolicyGradientPrimaldual2022, baiAchievingZeroConstraint2023,
 dingLastIterateConvergentPolicy2023, liFasterAlgorithmSharper2024,
 mondalSampleEfficientConstrainedReinforcement2024,
 dingConvergenceSampleComplexity2025}, primal-only approaches
\cite{achiamConstrainedPolicyOptimization2017,
 chowLyapunovbasedApproachSafe2018, dalalSafeExplorationContinuous2018,
 yangProjectionBasedConstrainedPolicy2020, xuCRPONewApproach2021,
 wachiSafeExplorationReinforcement2023, niSafeExplorationApproach2025,
 buckleyPrimalDualSampleComplexity2025}, and in Bayesian settings
\cite{asConstrainedPolicyOptimization2022, wendlSafeExplorationPolicy2025,
 asActSafeActiveExploration2025}.

\paragraph{Convex Reinforcement Learning.}
Unconstrained convex, or general-utility, RL generalizes classical RL by
allowing the objective to be a convex functional of the state-action
occupancy measure, instead of standard additive rewards
\cite{zahavyRewardEnoughConvex2021}. In the unconstrained setting,
\citet{zhangVariationalPolicyGradient2020} establish global convergence
guarantees in tabular MDPs, with subsequent extensions to function
approximation and variance reduction
\cite{zhangConvergenceSampleEfficiency2021,
 barakatReinforcementLearningGeneral2023a, santosNumberTrialsMatters2025a},
as well as to large-scale problems
\cite{huangOccupancybasedPolicyGradient2024,
 barakatScalableGeneralUtility2025}. Finite-trial formulations
\cite{muttiChallengingCommonAssumptions2022} have also been studied
\cite{muttiImportanceNonMarkovianityMaximum2022,
 muttiConvexReinforcementLearning2023}. Related directions include
multi-agent convex RL \cite{yingScalableMultiAgentReinforcement2023}, a
mean-field perspective \cite{geistConcaveUtilityReinforcement2022}, and
extensions to risk-sensitive MDPs \cite{wuRisksensitiveMarkovDecision2024},
submodular objectives \cite{prajapatSubmodularReinforcementLearning2023,
 desantiGlobalReinforcementLearning2024}, and robust convex RL
\cite{chenRobustReinforcementLearning2025}. Online variants of
unconstrained convex RL have been investigated in
\cite{morenoEfficientModelBasedConcave2024,
 marinmorenoMetaCURLNonstationaryConcave2024,
 morenoOnlineEpisodicConvex2025, barakatOnlineLearningHiddenConvex2026}.
Closely related in spirit,
\citet{kalogiannisLearningEquilibriaAdversarial2024} study policy gradient
convergence in imperfect-information extensive-form games. In the
optimization literature, \citet{fatkhullinStochasticOptimizationHidden2025}
analyze algorithms for unconstrained stochastic optimization problems under
hidden convexity and \citet{fatkhullinGlobalSolutionsNonConvex2025} develop
algorithms for hidden convex constrained problems. We refer to
\Cref{sec:AppendixDetailedComparison} for additional related work and a
detailed technical comparison of our contributions to the two most closely
related works, \citet{zhangVariationalPolicyGradient2020} and
\citet{fatkhullinStochasticOptimizationHidden2025}.

In contrast, significantly fewer works address convex RL under complex
(e.g., safety or imitation) constraints.
\citet{agarwalConcaveUtilityReinforcement2022a} study constrained convex
MDPs (CCMDPs) using a model-based approach over the occupancy-measure space
with count-based approximations and direct policy parametrization, relying
on a strong form of Slater's condition.
\citet{baiAchievingZeroConstraint2023a} relax this requirement to the
classical Slater condition and propose a Kullback Leibler (KL) regularized
primal-dual algorithm, still operating in a model-based occupancy-measure
formulation. More recently, \cite{yingPolicybasedPrimalDualMethods2025} is
the first work to study CCMDPs under a general softmax policy
parametrization using a policy-gradient approach in policy parameter space.
However, the obtained guarantees are limited to weak regret and ergodic
averages, i.e., $\frac{1}{N}\sum_{k=1}^{N} F_1(\theta^{(k)}) - F^\star \le
 \varepsilon$ and $\bigl[\frac{1}{N}\sum_{k=1}^{N} F_2(\theta^{(k)})\bigr]_+
 \le \varepsilon$, see for example the discussion in
\cite{mullerTrulyNoRegretLearning2024}. Such guarantees permit
cancellations in constraint violations and, moreover, are not operational
in practice: averaging or sampling from iterates in parameter space is
generally ill-defined for policy networks. Relatedly,
\citet{zhangSafeEfficientPrimalDual2024} study convex constrained MDPs in
an offline, model-based setting using primal-dual methods under different
assumptions. Overall, existing approaches to constrained convex RL suffer
from at least one of the following limitations: they rely on model-based or
count-based formulations that do not scale beyond small discrete
state-action spaces; they require (forms of) Slater's condition which might
fail to hold in practice; and they provide only ergodic guarantees,
precluding \emph{at-any-time} convergence of the last iterate.

More generally, it was observed for constrained MDPs
\cite{xuCRPONewApproach2021,islamovSafeEFErrorFeedback2025} that
primal-only methods can be preferable in practice, due to fewer
hyperparameters, simpler implementations, and faster convergence resulting
from the absence of delayed dual updates.

\section{Preliminaries}
\paragraph{Notation.}
For a finite set $A$, we denote with $\abs{A}$ the cardinality and with
$\Delta(A)$ the $(\abs{A}-1)$-dimensional simplex, i.e. the space of
probability distributions over $A$. If not stated differently,
$\norm{\innerEmpty{}}$ denotes the standard $\ell^2$-norm for vectors and
the spectral norm for matrices, and $\IP{\innerEmpty{}}{\innerEmpty{}}$ is
the Euclidean inner product. A differentiable function $F: \XX \to \RR$ is
$L$-smooth if the gradient is $L$-Lipschitz continuous, $L \in \RR_{\geq
  0}$. For an integer $k\in \NN$ define the set $\setm{k} \define \{1,
 \ldots, k\}$. A notation $\OO{\innerEmpty{}}$ captures all dependencies
except for the accuracy $\epsilon$ and $\OOTilde{\innerEmpty{}}$ hides
logarithmic terms in $\nicefrac{1}{\epsilon}$.

\paragraph{Markov Decision Process (MDP).}
A discrete-time infinite-horizon discounted Markov Decision Process is a
tuple $\MM \define (\SS, \AA, r, \PP, \gamma, \mu_0)$, where $\SS$ and
$\AA$ are finite sets of states and actions, respectively. At a discrete
timestep $t \in \NN_0$, being in state $\ss_t \in \SS$ taking an action
$\aa_t \in \AA$ yields a reward $r_t = r(\ss_t, \aa_t)$, according to the
reward function $r: \SS \times \AA \to [-r_{\max}, r_{\max}]$, and a new
state $\ss_{t+1} \sim \PP(\innerEmpty{}\vert \ss_t, \aa_t)$ according to
the transition probability kernel $\PP: \SS \times \AA \to \Delta(\SS)$.
$r_{\max} > 0$ gives an upper and lower bound on the reward. We define the
\emph{value function} of a trajectory $\tau \define (\ss_t, \aa_t)_{t \in
  \NN_0}$ as the cumulative discounted reward $ V(\tau) \define (1-\gamma)
 \cdot \sum_{t=0}^{\infty} \gamma^t r(\ss_t, \aa_t),$ where $\mu_0$ denotes
the initial state distribution, i.e. $\ss_0 \sim \mu_0$, and $\gamma \in
 (0,1)$ is the so-called discount factor. A mapping $\pi: \SS \to
 \Delta(\AA)$ is called a (stationary) \emph{policy}, and $\Pi$ denotes set
of all stationary policies. Combined with the initial distribution $\mu_0$,
it induces a probability measure over finite-length trajectories $\tau =
 (\ss_t, \aa_t)_{t=0}^{T-1}$ by sampling $\aa_t \sim \pi(\innerEmpty{}\vert
 \ss_t)$ at each timestep $t = 0,\ldots, T$, $T \in \NN$, i.e. $p(\tau)
 \define \mu_0(\ss_0) \pi(\aa_0\vert \ss_0) \cdot \prod_{t=1}^{T-1}
 \PP(\ss_t\vert \ss_{t-1}, \aa_{t-1}) \pi(\aa_t\vert \ss_t)$.
\paragraph{Policy Parametrization.} Throughout this work, we will assume the common softmax policy
parameterization
\begin{align*}
 \pi_\theta(\aa \vert \ss) \define \frac{\exp(\psi(\ss, \aa; \theta))}{
  \sum_{\aa^\prime \in \AA}\exp(\psi(\ss, \aa^\prime; \theta))
 },
\end{align*}
where $\theta \in \Theta \subset \RR^p$, $\ss \in \SS$, $\aa \in \AA$, and
a smooth function $\psi:\SS\times \AA \times \Theta \to \RR$.
\paragraph{State-action Occupancy Measure (OM).}
For any policy $\pi_\theta$, we define the (state-action) occupancy measure
$\lambda^{\pi_\theta}: \SS \times \AA \to [0,1]$ as
\begin{align*}
 \lambda^{\pi_\theta}(\ss, \aa)
 \define (1-\gamma)\sum_{t=0}^{\infty} \gamma^t \PP_{\mu_0,
  \pi_{\theta}}\paren{\ss_t = \ss, \aa_t = \aa},
\end{align*}
where $\PP_{\mu_0, \pi_{\theta}}$ denotes the probability distribution of the
Markov chain $(\ss_t, \aa_t)_{t\in \NN}$ induced by the policy $\pi_\theta$
and the initial state distribution $\mu_0$ under the transition probability
kernel $\PP$. We denote by
$\Lambda \define \{\lambda^{\pi} \; : \; \pi \in \Pi\}$
the set of all state-action occupancy measures induced by the stationary policies
$\Pi$.
Since $\abs{\SS}, \abs{\AA} < \infty$, we will interpret a $\lambda \in \Lambda$
as an element of $[0,1]^{\abs{\SS} \times \abs{\AA}}$. For a parametrized
policy, we will use the notations $\lambda(\theta)$ and $\lambda^{\pi_{\theta}}$
interchangeably.
Intuitively, since the OM captures the state-action visitation frequencies
induced by a policy, it provides a principled foundation for defining
exploration as maximizing the entropy of the induced state-action
distribution.

\section{Problem Formulation}
\label{sec:ProblemFormulation}
For a convex, $L$-smooth constraint $\Reg: \Lambda \to \RR$ formulated in
the state-action-occupancy measure, we formulate the following
constrained-max-entropy problem:
\begin{align}
 \max_{\theta \in \Theta} F_1(\theta)  \define \HH(\lambda^{\pi_\theta})
 = - \Expu{(\ss, \aa) \sim \lambda^{\pi_\theta}}{\log \lambda^{\pi_\theta}(\ss, \aa)} \;\;
 \mathrm{s.t.} \;\; F_2(\theta)        \define
 \Reg(\lambda^{\pi_\theta}) \leq 0.
 \tag{CME}
 \label{eq:MainProblemRL}
\end{align}
We assume only that there exists a stationary policy
$\pi_{\rm feas}\in\Pi$ satisfying
$\Reg(\lambda^{\pi_{\rm feas}})\leq0$. Since $\Lambda$ is a compact
polytope in a finite MDP, and $\HH$ and $\Reg$ are continuous, the
corresponding problem over all stationary policies admits an optimizer
$\pi^*\in\Pi$. We denote its value by
$\HH^*\define\HH(\lambda^{\pi^*})$ and set $\Fopt\define\HH^*$.
We call $\thetaHat \in \Theta$
an $(\epsilon,\epsilon)$-optimal solution if
$\abs{\Exp{F_1(\thetaHat)- \Fopt}} \leq \epsilon$
and $\Exp{\plus{F_2(\thetaHat)}} \leq \epsilon$,
where $\epsilon> 0$. At first sight, optimizing
\eqref{eq:MainProblemRL} over $\theta\in\Theta$, rather than directly over
all stationary policies, may appear restrictive. The following proposition
shows that the (anchored) tabular softmax family is dense in the set of
stationary policies. \Cref{thm:FinalComplexityAnchoredSoftmax} (Appendix)
then shows that this density is sufficient for the stated approximate objective and
feasibility guarantees.
\begin{proposition}[Anchored Softmax Approximation of Stationary Policies]
 \label{prop:SoftmaxApproximationStationaryPolicies}
 Let $\MM$ be a finite MDP with an initial distribution
 $\min_{\ss \in \SS}\mu_0(\ss) > 0$, and fix a reference action
 $\aa_0\in\AA$. Consider the anchored tabular logits
 $\psi(\ss,\aa_0;\theta)=0$ and
 $\psi(\ss,\aa;\theta)=\theta_{\ss,\aa}$ for $\aa\neq\aa_0$, with
 parameter space
 $\Theta_R=[-R,R]^{\SS\times(\AA\setminus\{\aa_0\})}$.
 For every stationary policy $\pi^* \in \Pi$, including deterministic
 ones, and every $\epsilon \in (0,1]$, there exist
 $R = \OO{\log(\card{\AA} / \epsilon)}$ and a parameter
 $\theta \in \Theta_R$ such that
 \begin{align*}
  \max_{\ss \in \SS}\;\norm{\pi_{\theta}(\innerEmpty{}\vert\ss)
   - \pi^*(\innerEmpty{}\vert \ss)}_{1} \leq \epsilon,
 \end{align*}
 i.e. the family of anchored tabular softmax classes
 $\bigcup_{R>0}\{\pi_\theta:\theta\in\Theta_R\}$ is dense in the set
 of stationary policies.
\end{proposition}
\begin{proof}
 Full details in \Cref{sec:AppendixConcretePolicyClass},
 we show that fixing the reference logit makes the parameter-to-occupancy
 map invertible with a Lipschitz inverse on $\Theta_R$.
\end{proof}

The problem formulation \eqref{eq:MainProblemRL} includes, but is not
limited to, the following \textbf{examples}:
\begin{itemize}
 \item \textbf{(Cumulative Cost)} A safety constraint can be given
       by the discounted cumulative cost
       $F_2(\theta) \define
        (1-\gamma)\Expu{\mu_0,\pi_\theta}{\sum_{t=0}^{\infty}
         \gamma^t \, c(s_t, a_t)}-\cMax$, for a cost vector
       $c \in \RR^{\SS \times \AA}$. In this case \eqref{eq:MainProblemRL}
       corresponds to finding the policy with
       the maximum entropy among all the policies solving the environment
       $\Reg(\lambda^{\pi_\theta}) \define \IP{c}{\lambda^{\pi_\theta}} - \cMax$.
 \item \textbf{(Apprenticeship Learning)} Closeness to an expert's or prior trajectory
       can be encoded using the KL divergence
       $\Reg(\lambda^{\pi_\theta}) \define \KL{\lambda^{\pi_\theta}}{\lambda^{\pi_\refPol}} - \rMax$,
       with respect to the state-action-occupancy measure of a reference policy
       $\lambda^{\pi_\refPol}$.
\end{itemize}
Naturally, our framework can be extended to multiple constraints, by
considering a vector-valued $\Reg$ and dealing with the Jacobian instead
of the gradient.

\section{Policy Gradient Penalty Method}
\label{sec:PgpAlgorithm}
Our approach proceeds by transforming \eqref{eq:MainProblemRL}
into an unconstrained optimization problem through a quadratic
penalty reformulation:
\begin{align}
 \min_{\theta\in\Theta} P(\theta) \define -F_1(\theta) + \beta \penGen{F_2(\theta)},
 \tag{PEN}
 \label{eq:MainPenalty}
\end{align}
where $\penGenNot \define \pen{\innerEmpty{}}$.
In \Cref{sec:AppendixNonSmoothAndSmoothPenaltyApproach}, we show that
for non-smooth constrained hidden convex problems the use of
the exact penalty function
$\penGen{\innerEmpty{}}  \define \plus{\innerEmpty}$
is favorable for the theoretical analysis.
Additionally, for $\xx\in\XX$ we define the quadratic-penalty
\emph{optimality value gap} as $\penOpt{2,\beta}(\xx)\define P(\xx)+\Fopt$.

\paragraph{(Stochastic) Policy Gradient.}
To derive the gradient of a general smooth, convex $F:\Lambda \to \RR$ with
respect to $\theta \in \Theta$ we can use the policy gradient theorem
\cite{suttonPolicyGradientMethods1999} and the chain rule. With
$V^{\pi_\theta}(\rew) = \IP{\lambda(\theta)}{\rew}$ we have
\begin{align}
 \begin{aligned}
  [\nabla_\theta \lambda(\theta)]^\top \rew
  = \nabla_\theta V^{\pi_\theta}(\rew)
  = (1-\gamma)\Expu{\mu_0, \pi_\theta}{\sum_{t=0}^{\infty}
   \gamma^t \rew(\ss_t, \aa_t) \sum_{k=0}^{t}
   \nabla_\theta \log(\pi_{\theta}(\aa_k \vert \ss_k))}
 \end{aligned}
 \tag{PG-Thm}
 \label{eq:PolicyGradientTheorem}
\end{align}
by denoting with $\nabla_\theta \lambda(\theta)$
the Jacobian of the (vector-valued) mapping
$\lambda(\theta)$, we can apply the chain rule and obtain
\begin{align*}
 \nabla_\theta F(\lambda(\theta))  =
 [\nabla_\theta\lambda(\theta)]^\top \cdot
 \nabla_\lambda F(\lambda(\theta))
 = \nabla_\theta V^{\pi_\theta}(\rew)\bigg\vert_{
  \rew = \nabla_\lambda F(\lambda(\theta))},
\end{align*}
i.e. we can compute the policy gradient of the objective
and the constraint estimating the state-action occupancy measure
and then applying the standard REINFORCE
\cite{zhangSampleEfficientReinforcement2021} estimator
using $H$-long trajectory truncations
\begin{align*}
 \nabla_\theta V^{\pi_\theta}(\rew)
 \approx (1-\gamma)\sum_{t=0}^{H-1} \left(\sum_{k=t}^{H-1}
 \gamma^k \rew(\ss_k, \aa_k)\right) \nabla_\theta
 \log(\pi_\theta(\aa_t\vert \ss_t))
 \defineRev g_H(\theta, \tau, \rew)
\end{align*}

The OM can be estimated using Monte-Carlo rollouts, i.e. for a batch of
$(\tau_b)_{b \in B}$, $\card{B} \geq 1$ i.i.d. trajectories, we define
\begin{align}
 \lambdaHat\left((\tau_b)_{b\in B}\right)
 \define \frac{1-\gamma}{\card{B}} \sum_{b\in B}\sum_{t=0}^{H-1} \gamma^t
 \delta_{\ss^{(b)}_t, \aa^{(b)}_t},
 \tag{$\lambda^{\mathrm{MC}}_H$-Est}
 \label{eq:StateOccupancyMeasureMonteCarloEstimate}
\end{align}
where for every $(\ss, \aa) \in \SS \times \AA$
the vector $\delta_{\ss, \aa} \in \{0,1\}^{\SS \times \AA}$
has a one at the entry $(\ss, \aa)$ and zeros elsewhere,
yielding an unbiased estimate of the \emph{truncated}
occupancy measure
\begin{align*}
 \lambda_H^{\pi_\theta}(\ss, \aa)
 \define (1-\gamma)\sum_{t=0}^{H-1} \gamma^t
 \PP_{\mu_0, \pi_{\theta}}\paren{\ss_t = \ss, \aa_t = \aa}.
\end{align*}
In the case of continuous state-action spaces, we
are using function approximation and a maximum likelihood
estimator \eqref{eq:StateOccupancyMeasureMLE}, with
implementation details in \Cref{sec:AppendixLambdaEstimateContinuous}.
In total, we define the (biased) estimate
\begin{equation}
 \gApprox_{H,B}(\nabla_{\lambda} F, \theta, (\tau_b)_{b\in\setm{B}})
 \define \frac{1}{B} \sum_{b=1}^{B} g_H\left(\theta, \tau_b,
 \rew = \nabla_\lambda F(\lambda)\big\vert_{\lambda =
  \lambdaHat((\tau_i)_{i \in \setm{B}})}\right)
 \tag{$\nabla_\theta$-Est} \label{eq:SubGradientEstimator}
\end{equation}
of the truncated policy-gradient
\(\nabla_\theta F(\lambda_H(\theta))\).
We present in \Cref{sec:AppendixGradientEstimates}, how this translates to approximation
guarantees with respect to the true, infinite horizon gradient.

\begin{wrapfigure}[22]{r}{0.54\textwidth}
 \vspace{-1.3\baselineskip}
 \begin{minipage}{0.54\textwidth}
  \begin{algorithm}[H]
   \captionof{algorithm}{\textsc{Policy Gradient Penalty Method}}
   \begin{algorithmic}[1]
    \STATE {\bfseries Input:} Objective $\HH$, constraint $\Reg$,
    initial policy $\xxZero \in \XX$, number of iterations $N \in \NN$,
    batch size $B \in \NN$, step size $\eta > 0$,
    penalty parameter $\beta>0$
    \FOR{$k = 0$ {\bfseries to} $N$}
    \STATE Sample trajectory batch $\tau^{(k)}_b \overset{\text{i.i.d.}}{\sim} \pi_{\theta^{(k)}}$, $b\in\setm{B}$
    \STATE Estimate occupancy measure $\lambda^{(k)} \define
     \lambdaHat((\tau_{b}^{(k)})_{b\in\setm{B}})$ via
    \eqref{eq:StateOccupancyMeasureMonteCarloEstimate} or \eqref{eq:StateOccupancyMeasureMLE}
    \STATE Compute pseudo-rewards via Backpropagation
    \label{AlgoLine:ShadowRewardsGradientComputation}
    \begin{align*}
     g_{O}^{(k)} & \define \nabla_{\lambda} (-\HH(\lambda))\big\vert_{\lambda = \lambda^{(k)}},                 \\
     g_{C}^{(k)} & \define \nabla_{\lambda} (\penGenNot \circ \Reg(\lambda))\big\vert_{\lambda = \lambda^{(k)}}
    \end{align*}
    \vspace{-0.6cm}
    \STATE Gradient estimate via \eqref{eq:SubGradientEstimator} for $\diamond \in \{O, C\}$:
    \begin{align*}
     \gApprox^{(k)}_{\diamond}
     \define \gApprox_{H, B}(g^{(k)}_\diamond, \thetaK,
     \{\tau^{(k)}_b\}_{b\in \setm{B}}),
    \end{align*}
    \vspace{-0.6cm}
    \STATE Stochastic Gradient Descent Step
    \begin{align*}
     \xxKnext \leftarrow \proj_{\Theta}\left[\xxK - \eta \cdot (\gApprox_{O}^{(k)}
      + \beta \gApprox_{C}^{(k)})\right]
    \end{align*}
    \vspace{-0.7cm}
    \ENDFOR
    \STATE {\bfseries Return:} Last policy $\xxNnext$.
   \end{algorithmic}
   \label{algo:Penalty}
  \end{algorithm}
 \end{minipage}
\end{wrapfigure}
\textbf{Pseudo-rewards and single-trajectory-batch penalty
 gradients.}
A key algorithmic novelty lies in the construction of \emph{pseudo-rewards}
$g^{(k)}_\diamond$, $\diamond \in \{O, C\}$,
introduced in Line~\ref{AlgoLine:ShadowRewardsGradientComputation}.
In generic stochastic
optimization, both, exact and quadratic penalty formulations require
applying the (generalized) chain rule through the constraint function,
which canonically necessitates separate stochastic estimators for the
function value and its (sub-)gradient. This mismatch induces bias and has
led prior work on stochastic penalty methods and ALM to
either restrict attention to deterministic constraints
\cite{kushnerPenaltyFunctionMethods1974,
 wangStochasticOptimisationInequality2008,
 wangRobinsMonroAugmentedLagrangian2022},
incur additional sample complexity via multiple gradient or function evaluations per
iteration
\cite{cuiExactPenaltyMethod2025, liuSingleloopSPIDERtypeStochastic2025,
 yangSingleloopAlgorithmsStochastic2025}
or rely on momentum and variance-tracking schemes to control the
resulting bias \cite{ liStochasticInexactAugmented2024,
 liRetentionCentricFrameworkContinual2025}.

Our method fundamentally departs from this paradigm by explicitly
exploiting the RL structure of the problem. Leveraging the Policy Gradient
Theorem \eqref{eq:PolicyGradientTheorem}, we reinterpret gradients of both
the objective and the penalty term as policy gradients with respect to
constructed pseudo-rewards. This yields the estimator in
\eqref{eq:SubGradientEstimator}, which allows us to compute a stochastic
gradient of the penalized objective using a single trajectory batch,
without auxiliary function-value estimation.

Importantly, the occupancy measure is estimated only once from the full
batch, and the resulting common pseudo-reward is reused on every trajectory
in that same batch. In particular, the quadratic-penalty gradient contains
$2\beta\plus{\Reg(\lambda)}\nabla\Reg(\lambda)$, whose two factors both
depend on the unknown occupancy measure. Using the same trajectory batch to
estimate these factors and form the policy-gradient estimate therefore
yields a dependent and biased estimator. We show in
\Cref{lem:BiasVarianceGradientEstimator,lem:PenaltyBiasVarianceGradientEstimator}
that our pseudo-reward construction controls both truncation and
occupancy-estimation bias, as well as the estimator variance, without
requiring a model, an exact value function, nested sampling, or auxiliary
function evaluations.

\subsection{Theoretical Analysis}
\label{sec:TheoreticalAnalysis}
The key theoretical insight is the hidden convex structure induced by the
occupancy-measure formulation. By mapping a policy $\pi$ to its
corresponding occupancy measure, the penalized problem becomes convex and
admits a one-step contraction inequality under biased stochastic gradients.
Translating this contraction back to policy space allows us to unroll the
recursion over iterations, yielding global last-iterate convergence
guarantees.
We make the following (standard) assumptions:

\begin{assumption}[Softmax-Parametrization]\label{ass:Softmax}
 The function $\psi : \SS \times \AA \times \Theta \to \RR$
 satisfies $\psi \in C^2(\SS, \AA, \Theta; \RR)$, and there exist
 constants $\gradPsi, \jacPsi > 0$ such that
 (i) $\max_{\ss\in\SS, \aa\in\AA} \sup_{\theta \in \Theta}
  \norm{\nabla_\theta \psi(\ss, \aa, \theta)} \leq \gradPsi$
 and (ii) $\max_{\ss\in\SS, \aa\in\AA} \sup_{\theta \in \Theta}
  \norm{\nabla^2_\theta \psi(\ss, \aa, \theta)} \leq \jacPsi$.
 \end{assumption}
Under these assumptions, we show the regularity of the \emph{score function}
$\nabla_\theta \log \pi_\theta(\innerEmpty{})$ in \Cref{sec:AppendixGradientEstimates}.

\begin{assumption}[Occupancy-Measure Parametrization]\label{ass:ParametrizationSOAM}
 The parameter-to-occupancy mapping
 $\lambda:\Theta\to\Lambda$, $\theta\mapsto\lambda^{\pi_\theta}$,
 satisfies:
 \begin{itemize}
  \item (Locally Invertible) For every parameter $\theta \in \Theta$
        there exists a neighborhood
        $\theta \in \mathcal{U}_\theta \subset \Theta$ such
        that $\lambda^{\pi_\theta}\vert_{\mathcal{U}_\theta}$ is a bijection
        between $\mathcal{U}_\theta$ and
        $\VV_\theta \define \lambda^{\pi_\theta}(\mathcal{U}_\theta)$. A (local) inverse is  $\lambda^{-1}_{\VV_\theta}
         : \VV_\theta \to \mathcal{U}_\theta$.
  \item %
        (Lipschitz Inverse) The inverse $\lambda^{-1}_{\VV_\theta}$ is $\lamLip$-Lipschitz continuous over $\Theta.$
  \item (Locally Hidden Convex) There exists an $\epsLam > 0$ such that
        for all $\epsilon \in (0, \epsLam]$ and all $\theta \in \Theta$
        the convex combination satisfies
        $(1-\epsilon) \lambda^{\pi_\theta} + \epsilon \lambda^{\piOpt}
         \in \VV_\theta$.
 \end{itemize}
\end{assumption}
The above assumption is common and was verified in a tabular setting, e.g.,
\citep[Prop. H.1]{zhangVariationalPolicyGradient2020}.
For the (anchored) tabular softmax parametrization above, we verify
\Cref{ass:Softmax} and \Cref{ass:ParametrizationSOAM} directly in
\Cref{sec:AppendixConcretePolicyClass}.
Note that since $\Lambda$ is a polytope ($\Lambda \subset [0,1]^{\SS
  \times \AA}$), its domain is bounded, i.e., there exists $\DDLam > 0$ such that
$\norm{\lambda(\theta) - \lambda(\theta^\prime)}
 \leq \DDLam,
$ %
for any $\theta, \theta^\prime \in \Theta$. Due to the (local)
Lipschitzness of the inverse we can also uniformly bound the domain
$\Theta$ by a diameter $\DDX$, cf. Thm. 3.2 in
\citet{zhangVariationalPolicyGradient2020}

\begin{assumption}[L-smoothness]
 \label{ass:Lsmoothness}
 Assume that for $\HH$ and $\Reg$ there exist
 $\ell_\lambda, L_\lambda, L_{\lambda, \infty} > 0$
 such that $G \in \{-\HH, \Reg\}$ satisfies for
 all $\lambda, \lambda^\prime \in \Lambda$:
 (i) $\norm{\nabla_\lambda G(\lambda)}_{\infty} \leq \ell_\lambda$,
 (ii) $\norm{\nabla_{\lambda} G(\lambda) - \nabla_\lambda G(\lambda^\prime)}_{\infty}
  \leq L_\lambda \norm{\lambda - \lambda^\prime}_2$, and
 (iii) $\norm{\nabla_{\lambda} G(\lambda) - \nabla_\lambda G(\lambda^\prime)}_{\infty}
  \leq L_{\lambda, \infty} \norm{\lambda - \lambda^\prime}_1$.
\end{assumption}
We want to emphasize that the aforementioned assumptions are
standard in convex (constrained) RL,
cf. \cite{hazanProvablyEfficientMaximum2019,
 zhangConvergenceSampleEfficiency2021,
 barakatReinforcementLearningGeneral2023a,
 yingPolicybasedPrimalDualMethods2025,
 barakatScalableGeneralUtility2025}.
Strictly spoken, the entropy objective is only locally smooth
\cite{hazanProvablyEfficientMaximum2019,
 zhangConvergenceSampleEfficiency2021}
and in \Cref{sec:AppendixSmoothedEntropy} we consider
the smoothed alternative $\HH_\sigma(\lambda)
 \define-\sum_{\ss,\aa}\lambda(\ss,\aa)
 \log\bigl(\lambda(\ss,\aa)+\sigma\bigr)$, for $\sigma > 0$, explictily
deriving the constants above as well as approximation guarantees
and strong-duality translation back to the original
entropy.

The following main theorem gives a guarantee on the optimality value gap of
the penalty function, without the assumption of strong duality.
\begin{theorem}[Biased SGD for \eqref{eq:MainPenalty}]
 \label{thm:MainResultPenalty}

 Let $\epsilon > 0$ be the target accuracy. Under Assumptions
 \ref{ass:Softmax}, \ref{ass:ParametrizationSOAM}, \ref{ass:Lsmoothness},
 running \Cref{algo:Penalty} with a step size $\eta =
  \nicefrac{1}{L_{P,\theta}}$, a batch size of $B =
  \OOTilde{\frac{1}{\epsilon^2} + \frac{\beta^2}{\epsilon^2}}$ and a
 trajectory length of $H = \OO{\log \left(\frac{1}{\epsilon} +
   \frac{\beta}{\epsilon}\right)}$, we obtain
 \begin{align*}
  \Exp{\penOpt{2, \beta}(\xxN)}
  \leq \epsilon
  \quad \text{after} \quad
  N \geq \OO{\left(\frac{1}{\epsilon} + \frac{\beta}{\epsilon}\right) \cdot
   \log\left(\frac{ \penOpt{2, \beta}(\xxZero) }{\epsilon}\right)} = \OOTilde{
   \frac{1}{\epsilon} + \frac{\beta}{\epsilon} }.
 \end{align*}
 iterations, where $\penOpt{2, \beta}(\xxZero)$ is the initial value gap and
 $L_{P,\theta}$ is defined in \Cref{cor:GradientBoundSmoothnessPenalty}.
\end{theorem}
\begin{proof}\textbf{(Sketch)}
 First, our proof establishes the regularity of the penalty objective
 $P$ in \Cref{lem:AppendixSmoothnessBoundedGradientsPenaltyLambda} using
 \Cref{ass:Lsmoothness}. Next, we analyze \Cref{algo:Penalty} by controlling
 the bias and variance of the SGD step, which yields global convergence.
 A detailed proof is available in \Cref{thm:MainResultPenalty-PROOF}.
\end{proof}

The above theorem implies that for any penalty regularization parameter
$\beta$, \Cref{algo:Penalty} solves our penalty formulation
\eqref{eq:MainPenalty} despite the bias of the stochastic gradients
occurring from the estimation of the state occupancy measure in
\eqref{eq:StateOccupancyMeasureMonteCarloEstimate} and propagating to
\eqref{eq:SubGradientEstimator}. However, the sample complexity $N \cdot B$
crucially depends on the parameter $\beta$, coming from the estimate of the
smoothness constant $L_{P,\theta}.$ In what follows, we will translate the
result to convergence to an $(\epsilon, \epsilon)$-optimal solution and
give a recommendation on the parameter $\beta$ trading the feasibility and
sample complexity. We use the following standard assumption.

\begin{assumption}[Strong Duality]
 \label{ass:StrongDuality}
 We assume \emph{strong duality} holds for \eqref{eq:MainProblemRL},
 i.e. for $\xxOpt$ there exists a Lagrangian multiplier
 $\nuOpt \geq 0$ such that $\dOpt = d(\nuOpt) = -\Fopt$, where
 the \emph{(Lagrangian) dual function} is defined as
 the infimum of the \emph{Lagrangian} with respect to the primal variable $\xx$, that is,
 $d(\nu)\define\inf_{\xx\in\XX}L(\xx,\nu)
  =\inf_{\xx\in\XX}\bigl(-F_1(\xx)+\nu F_2(\xx)\bigr)$,
 and the corresponding \emph{dual problem} as
 $\dOpt \define \sup_{\nu \geq 0} d(\nu)
  = \sup_{\nu \geq 0} \inf_{\xx \in \XX} L(\xx, \nu)$.
\end{assumption}
It is well known that,
in general for any $\nu \geq 0$, we have $d(\nu) \leq -\Fopt$, i.e. weak
duality always holds \cite{nesterovLecturesConvexOptimization2018}.
\Cref{ass:StrongDuality} is not required for convergence to the penalty
objective in \Cref{thm:MainResultPenalty}; it is used only for the subsequent
translation to objective and feasibility guarantees through a finite
multiplier. Most importantly, it is strictly weaker than Slater's condition: Slater is a
sufficient certificate for strong duality
\citep{nesterovLecturesConvexOptimization2018}, but is not necessary.
Consequently, strong duality permits boundary-only feasible sets and
equality constraints for which Slater fails
\cite{fatkhullinGlobalSolutionsNonConvex2025}. Under linear
safety or performance constraints, a feasible point is already sufficient, cf.
\citet[Prop. 5.3.1]{bertsekasConvexOptimizationTheory2009}.

\subsection{Sample Complexity Analysis}
\label{sec:ComplexityAnalysis}
Under strong duality, we can now translate the results of
\Cref{thm:MainResultPenalty} into guarantees for the optimality value gap
and the constraint violation.

\begin{corollary}
 \label{cor:FinalRateTranslation}
 Let $\epsilon > 0$ be the target accuracy.
 Under Assumptions \ref{ass:Softmax}, \ref{ass:ParametrizationSOAM},
 \ref{ass:Lsmoothness} and strong duality, \Cref{ass:StrongDuality},
 running \Cref{algo:Penalty} with
 a step size $\eta = \OO{\epsilon}$,
 a batch size of $B = \OO{{\epsilon^{-4}}}$,
 a penalty parameter $\beta = \OO{\epsilon^{-1}}$,
 and a trajectory length of
 $H = \OO{\log \frac{1}{\epsilon}}$, we obtain
 \begin{align*}
  -\epsilon \leq \Exp{F_1(\xxN) - \Fopt}
  \leq \epsilon \quad \text{and} \quad
  \Exp{\plus{F_2(\xxN)}} \leq \epsilon
 \end{align*}
 after $N = \OOTilde{\epsilon^{-2}}$ iterations,
 resulting in a total of $\OOTilde{\epsilon^{-6}}$ trajectory rollouts
 of length $\OO{\log \frac{1}{\epsilon}}$.\footnote{
  The expectation is with respect to the randomness of the MDP, the policy, and the initial
  state distribution.
 }
\end{corollary}
\vspace{-0.3cm}
\begin{proof}
 \textbf{(Sketch)} Translating the guarantees of \Cref{thm:MainResultPenalty}
 under strong duality requires
 $\beta(\epsilon) \define (\nuOpt + 1)\frac{\nuOpt
   + \sqrt{\nuOpt^2 + 2}}{\epsilon} = \OO{\epsilon^{-1}}$, using
 \Cref{cor:AppendixConstraintViolationQuadraticPenaltyTranslation}.
 See \Cref{cor:FinalRateTranslation-PROOF} for the full derivation.
\end{proof}

In Appendix~\ref{sec:AppendixSmoothedEntropy} we show that working with the
smoothed entropy with $\sigma=\frac{\epsilon}{4\card{\SS}\card{\AA}}$
requires only the original strong-duality assumption and preserves the
$\OOTilde{\epsilon^{-6}}$ rollout rate for fixed dimension. Writing
$d=\card{\SS}\card{\AA}$, tracking dimension explicitly yields the
conservative dependence $\OOTilde{(d+1)^3\epsilon^{-6}}$. To the best of
our knowledge, this is the first $\mathrm{poly}(1/\epsilon)$ sample
complexity guarantee for constrained maximum-entropy exploration under
general policy parametrization. Crucially, the guarantee is for the
\emph{last iterate} of a single-loop algorithm and applies simultaneously
to the objective and constraint in \eqref{eq:MainProblemRL}. The choice
$\beta = \OO{\epsilon^{-1}}$ follows from the strong-duality translation,
reflecting the scaling between feasibility enforcement and the dual
multiplier. A limitation of the current analysis is the large batch size $B
 = \OO{\epsilon^{-4}}$, required to control truncation bias and
occupancy-measure estimation. We expect this can be reduced, for example
via variance reduction techniques
\citep{zhangConvergenceSampleEfficiency2021,
 barakatReinforcementLearningGeneral2023a,
 fatkhullinStochasticPolicyGradient2023a}; indeed, the ablations in the next
section suggest the $\epsilon^{-4}$ dependence is likely conservative in
practice. A sharper theoretical analysis is left for future work.

\vspace{-0.3cm}
\section{Numerical Experiments}
\label{chapter:Experiments}
We validate our method in a \emph{FrozenLake} grid-world environments which is visualized
in \Cref{fig:SmallGridWorldVis}, where the experiment involves maximizing the entropy
while fulfilling a (linear) safety constraint.
The full details and additional experiments with a non-linear norm
constraint are deferred to \Cref{sec:AppendixNumericalExperiments}. All
training curves report the mean and one standard deviation across seeds;
the grid-world experiments use 10 seeds and the continuous experiments use
5 seeds each.

\subsection{Theoretical Setting: Grid World} \label{sec:NumericalExperimentsLinearPerformanceConstraints}

We compare our approach to the primal-dual (PDPG) algorithm of
\citet{yingPolicybasedPrimalDualMethods2025} for which we performed a
hyperparameter search (details in
\Cref{subsec:AppendixNumExpLinPerformanceConstraint}) and display their
best-performing model in
\Cref{fig:Figure1MaxEntropyPerformanceConstrainedVsUnconstrained}.
While our \algo{}-method achieves an objective function slightly below the
unconstrained maximum entropy policy and meets the safety constraints, the
primal-dual version exhibits weak regret in practice: the iterates
oscillate around the constraint, cf. \cite[Sec.
 6]{mullerTrulyNoRegretLearning2024}. This distinction persists throughout
late training: over the final 1000 epochs, highlighted by the pink region
\circledPurple{1} in
\Cref{fig:Figure1MaxEntropyPerformanceConstrainedVsUnconstrained},
\Cref{sub-tab:LateIterateStability-GridWorld} shows that PGP satisfies both
feasibility tolerances substantially more often than PDPG.
\begin{figure}[!h]
 \vspace{-0.3cm}
 \centering

 \begin{minipage}[t]{0.463\textwidth}
  \includegraphics[width=\columnwidth]{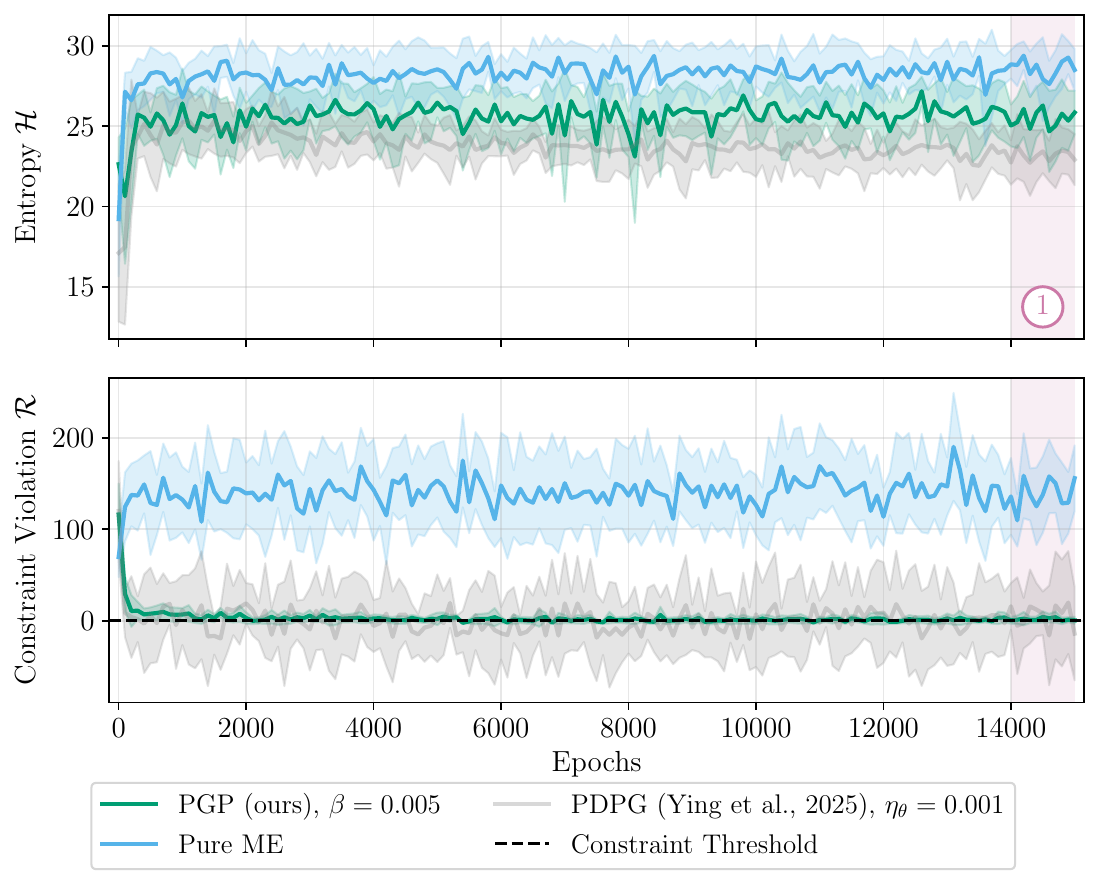}
  \caption{
   \textbf{Solving constrained maximum-entropy exploration} using
   \okBluishGreen{Policy Gradient Penalty Method (\algo{})},
   \okSkyBlue{unconstrained Maximum-Entropy (ME)}
   exploration,
   and the \gray{primal-dual (PDPG)} method
   of \cite{yingPolicybasedPrimalDualMethods2025}.
   Area \circledPurple{1} marks the
   final 1k epochs evaluated in \Cref{sub-tab:LateIterateStability-GridWorld}.
   Policies are visualized in \Cref{fig:VisPolicyMaxEntropyComparison} (Appendix).
  }
  \label{fig:Figure1MaxEntropyPerformanceConstrainedVsUnconstrained}
 \end{minipage}
 \hfill
 \begin{minipage}[t]{0.51\textwidth}
  \centering
  \includegraphics[width=\columnwidth]{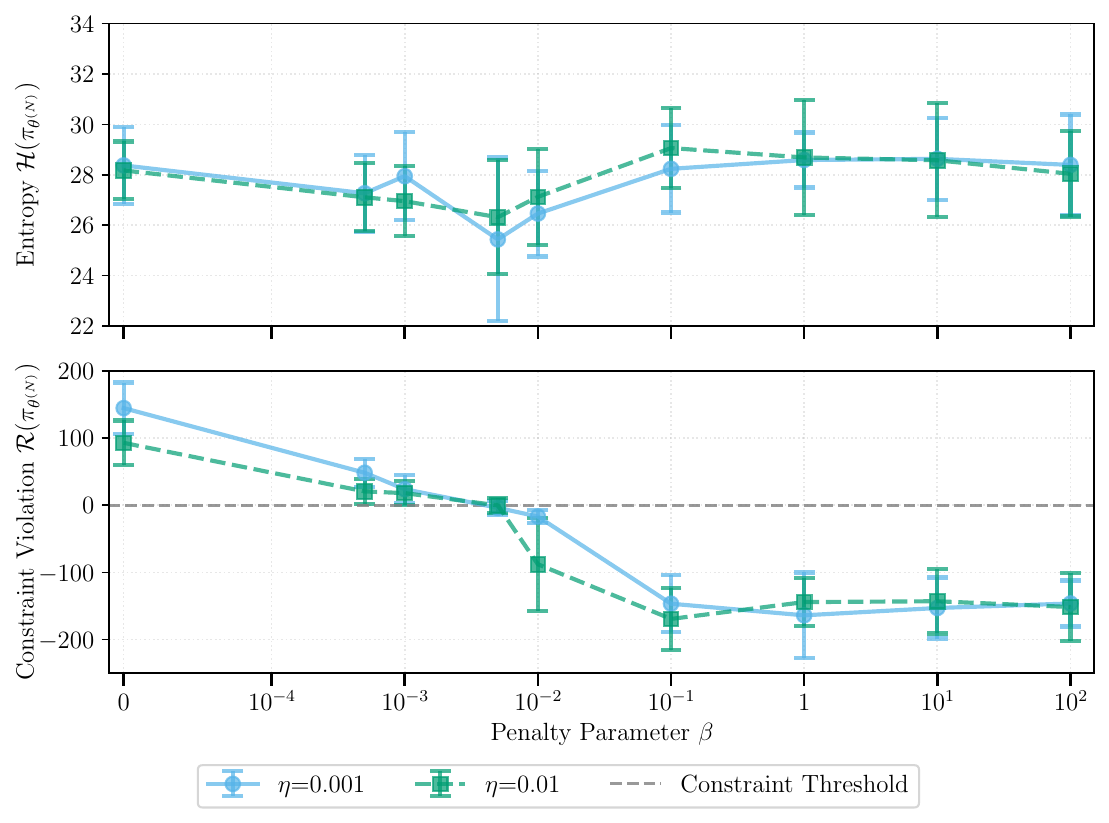}
  \caption{
   \textbf{Robustness to $\beta$.}
   The entropy of the final policy $\pi_{\theta^{(N)}}$ is
   insensitive to the choice of $\beta$ across several orders of magnitude,
   while constraint violations decrease monotonically as $\beta$ increases,
   for both \textbf{step sizes} \okSkyBlue{$\eta = 0.001$} and
   \okBluishGreen{$\eta = 0.01$}. Experimental details are available in
   \Cref{sec:AppendixNumericalExperiments}. }
  \label{fig:SensitivityPenaltyParameterBeta}
 \end{minipage}
 \vspace{-0.2cm}
\end{figure}
\Cref{fig:SensitivityPenaltyParameterBeta} shows that the proposed penalty formulation is highly robust
to the choice of
$\beta$. Across several orders of magnitude, the achieved entropy remains essentially unchanged once
feasibility is reached, while constraint violations decrease monotonically.
This indicates that effective constraint enforcement does not require fine-tuning of
the penalty parameter and does not trade off against the exploration objective.
The checkpoint analysis in
\Cref{sub-tab:BetaSensitivityConvergence-GridWorld} also
quantifies the corresponding cost: once $\beta$ is sufficiently large for
feasibility, further overestimation primarily slows convergence rather than
changing the final entropy or constraint value.

\begin{table}[!ht]
 \centering
 \begin{minipage}[t]{0.46\linewidth}
  \centering
  \scriptsize
  \setlength{\tabcolsep}{1.7pt}
  \renewcommand{\arraystretch}{0.95}
  \begin{tabular}{@{}lcccccc@{}}
   \toprule
   $\beta$         &
   $5\cdot10^{-3}$ & $10^{-2}$ & $10^{-1}$ & $1$  & $10$ & $10^2$         \\
   \midrule
   Conv. epoch     & 2200      & 2800      & 3500 & 5000 & 7000   & 10000 \\
   \bottomrule
  \end{tabular}
  \captionsetup{justification=justified,singlelinecheck=false}
  \subcaption{\textbf{Grid world.}
   Number of epochs until convergence as a function of $\beta$, cf.
   \Cref{fig:SensitivityPenaltyParameterBeta}.}
  \label{sub-tab:BetaSensitivityConvergence-GridWorld}
 \end{minipage}
 \hfill
 \begin{minipage}[t]{0.51\linewidth}
  \centering
  \scriptsize
  \setlength{\tabcolsep}{1.5pt}
  \renewcommand{\arraystretch}{0.95}
  \begin{tabular}{@{}crrr@{}}
   \toprule
   $\beta$ & Conv. epoch & Entropy           & Constraint      \\
   \midrule
   $10$    & 2500        & $798.22\pm214.13$ & $-12.78\pm8.42$ \\
   $50$    & 3500        & $825.82\pm211.10$ & $-13.94\pm8.67$ \\
   $100$   & 10000       & $787.29\pm246.31$ & $-12.50\pm7.82$ \\
   $200$   & 14000       & $817.61\pm235.52$ & $-13.72\pm8.12$ \\
   \bottomrule
  \end{tabular}
  \captionsetup{justification=justified,singlelinecheck=false}
  \subcaption{\textbf{SafeCartpole.}
   Number of epochs until convergence,
   terminal entropy and constraint values
   ($\mu\pm\sigma$) over $\beta$.}
  \label{sub-tab:BetaSensitivityConvergence-Cartpole}
 \end{minipage}
 \caption{\textbf{Effect of overestimating $\beta$.}
  Approximate convergence epochs estimated from saved checkpoints.
  Once $\beta$ is large enough to enforce feasibility, further increases
  mainly slow convergence while leaving the final entropy and constraint
  values essentially unchanged.}
 \label{tab:BetaSensitivityConvergence}
\end{table}
\vspace{-0.2cm}

\subsection{Beyond Theory: Continuous State-Action Space}
\label{sec:NumericalExperimentsCartpole}
These experiments assess practical scalability and whether the
robustness to gradient noise and the choice of $\beta$ observed in the finite
setting transfers to continuous state-action spaces. Extending our theoretical
analysis to this setting is an important direction for future work.
The main challenge in the continuous setting lies in estimating
approximating the state-action occupancy measure.
We address this using the current state-of-the-art maximum-likelihood
estimator of \citet{barakatScalableGeneralUtility2025}, with technical
details deferred to \Cref{sec:AppendixContinuousSpaceExperiments}.
\paragraph{Imitation Learning Constraint.} We train a reference policy $\pi_{\mathrm{ref}}$ using the Soft-Actor
Critic (SAC) method \citep{raffinStableBaselines3ReliableReinforcement2021}
to solve the \emph{PointMass} environment
\citep{dulac-arnoldChallengesRealworldReinforcement2021}. We then apply our
method to \eqref{eq:MainProblemRL} with an apprenticeship constraint
$\Reg(\lambda^{\pi_\theta}) \define
 \KL{\lambda^{\pi_\theta}}{\lambda^{\pi_\mathrm{ref}}} - \rMax \leq 0$.
\Cref{fig:PointMassHeatmapComparison} illustrates the central tension of
this setting: the reference policy consistently directs the agent toward
the center of the environment, inherently discouraging exploration. As the
constraint budget $\rMax$ increases, the learned policy is permitted
greater deviation from $\pi_{\mathrm{ref}}$ and progressively overcomes
this bias, achieving increasingly uniform coverage of the state space. We
visualize the controls for different starting positions and varying $\rMax$
in \Cref{fig:PointMassControlTrajectory} (Appendix).
\begin{figure*}[!h]
 \centering
 \includegraphics[width=\textwidth]{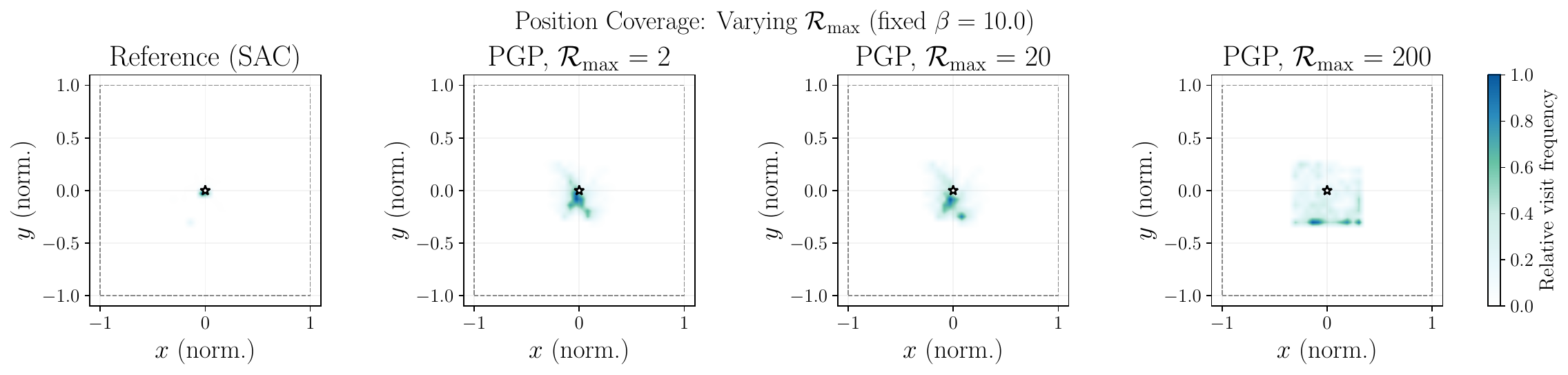}
 \caption{
  \textbf{PointMass: State-Space Coverage.}
  Relative visit frequencies of the SAC reference policy (trained to reach $(\star)$)
  and PGP-policies %
  evaluated on 30 start positions. %
  Reference policy concentrates near the
  center, while increasing $\rMax$
  yields progressively more uniform state-space coverage.}
 \label{fig:PointMassHeatmapComparison} \vspace{-0.5cm}
\end{figure*}
\paragraph{Reward Constraint.}
The \emph{SafeCartpole} task
\citep{dulac-arnoldChallengesRealworldReinforcement2021} imposes a (linear)
symmetric box constraint $\Reg(\lambda^{\pi_\theta}) \define
 \IP{c_{\mathrm{pos}}}{\lambda^{\pi_\theta}} - \cMax \leq 0$ on the cart
position, making exploration structurally challenging. Effective
exploration requires a swing-up, which must be executed within tight
positional bounds. \Cref{fig:CartpolePerformance} quantifies this tension:
increasing entropy, necessary for swing-up, systematically drives
constraint violations toward the safety threshold, with \algo{}
successfully navigating this tradeoff (\Cref{fig:Figure1Cartpole}).
Compared with the tuned PDPG baseline, PGP maintains sustained swing-up and
feasibility substantially more consistently over the final 500 epochs, cf.
pink region \okReddishPurple{\textcircled{\scriptsize 2}} in
\Cref{fig:CartpolePerformance} and quantified in
\Cref{sub-tab:LateIterateStability-Cartpole}. This comparison is practical
evidence beyond the scope of our finite-space theorem. The policy snapshots
(vertical lines in \Cref{fig:CartpolePerformance}) trace the training
dynamics in \Cref{fig:CartpoleTrainingEvolution} (Appendix): early
iterations focus on swing-up, intermediate policies exploit symmetry for
broader state-space coverage, and later stages approach the constraint
boundary to minimize swing-up time while maintaining feasibility. A
detailed quantitative analysis is provided in
\Cref{sec:AppDetailsCartpole}.
\begin{figure}[!ht]
 \centering
 \begin{subfigure}[b]{0.63\textwidth}
  \centering
  \includegraphics[width=\linewidth]{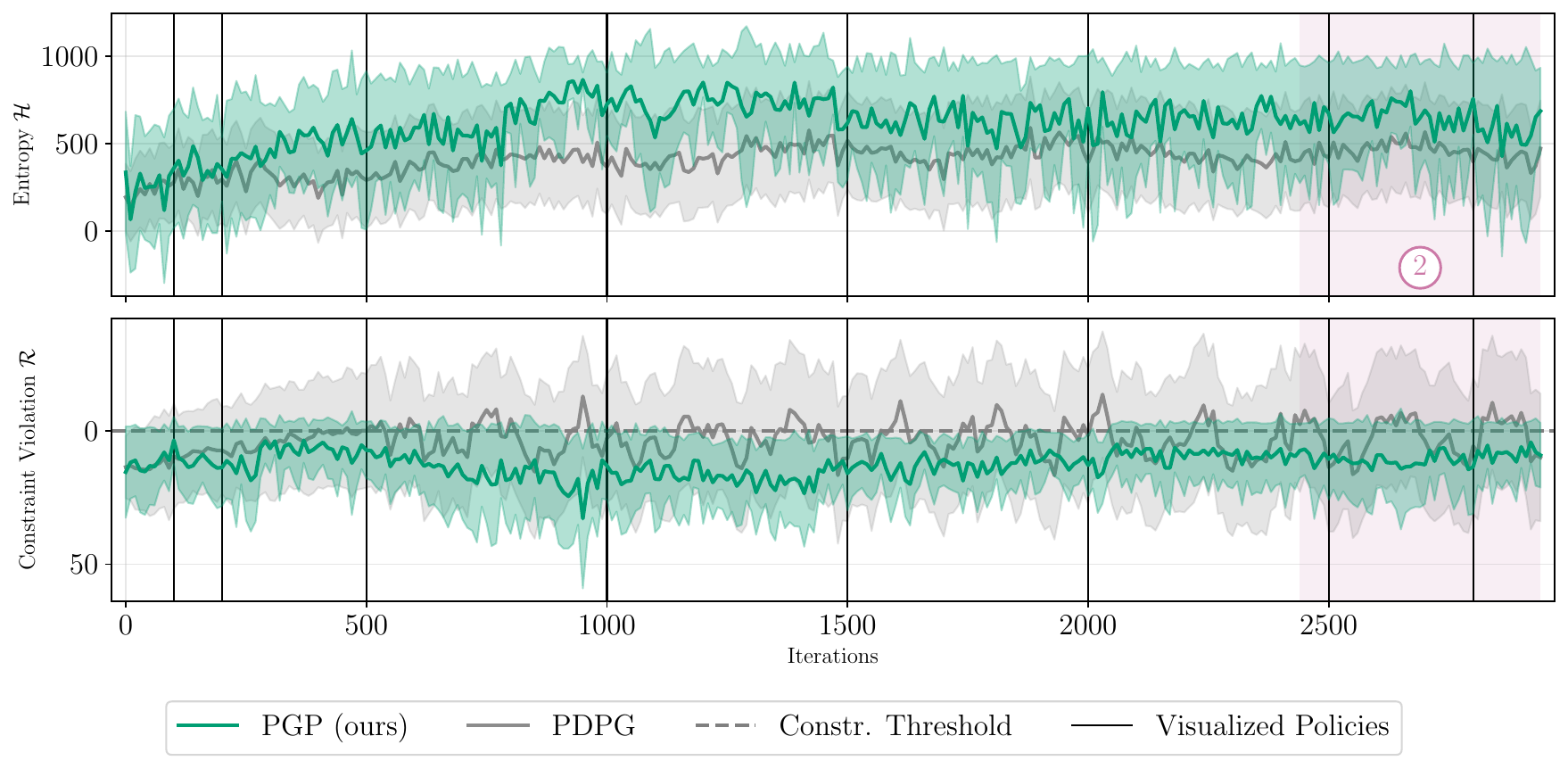}
  \caption{
   \okBluishGreen{PGP (ours)} and \gray{PDPG}: entropy $\HH$
   (top) and constraint violation $\Reg$ (bottom).
   Area \circledPurple{2} marks the final 500 epochs in
   \Cref{sub-tab:LateIterateStability-Cartpole}; vertical lines mark the policies
   visualized in \Cref{fig:CartpoleTrainingEvolution} (Appendix).
  }
  \label{fig:CartpolePerformance}
 \end{subfigure}
 \hfill
 \begin{subfigure}[b]{0.35\textwidth}
  \centering
  \imgwithtopcaption
  {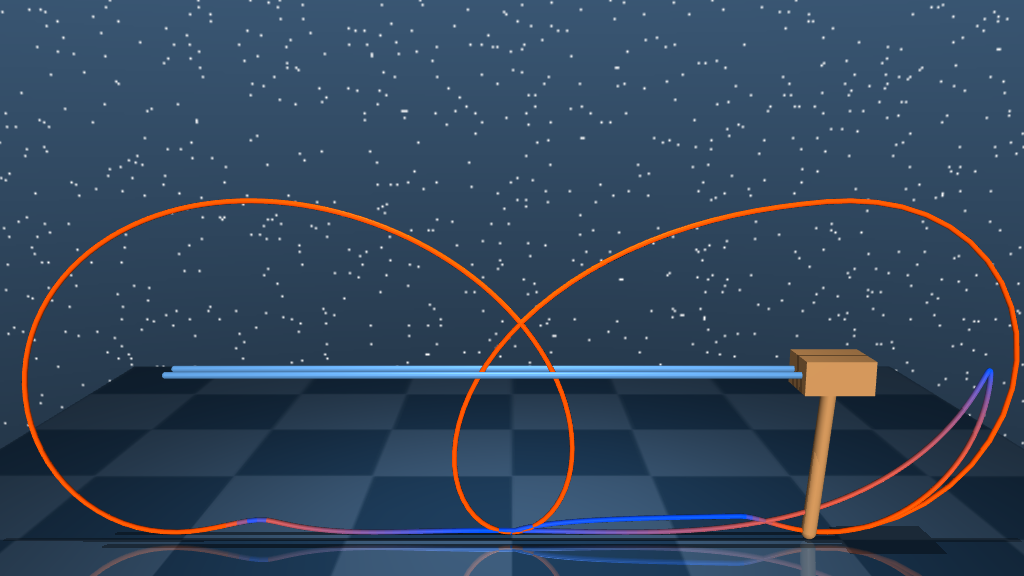}
  {$t = 9.6\,\mathrm{s}$}
  \caption{
   Final policy $\pi_{\theta^{(N)}}$
   satisfying the cart-position constraint $[-2.0,2.0]$.
   Tracer color denotes pole angular velocity
   (\markerBlue{lower}, \markerOrange{higher}), see
   \Cref{sec:AppDetailsCartpole}.
  }
  \label{fig:Figure1Cartpole}
 \end{subfigure}

 \caption{
  \textbf{\eqref{eq:MainProblemRL} on SafeCartpole.}
  \algo{} shows safe exploration of the
  full state space (pos. and vel.), while PDPG suffers under
  constraint oscillations, cf. \Cref{sub-tab:LateIterateStability-Cartpole}.
 }
 \label{fig:CartpoleCombined}
 \vspace{-0.3cm}
\end{figure}

\begin{table}[!ht]
 \centering
 \begin{minipage}[t]{0.35\linewidth}
  \centering
  \small
  \begin{tabular}{lcc}
   \toprule
   Method     & $\Reg\leq0.1$    & $\Reg\leq0.01$   \\
   \midrule
   PGP (ours) & \textbf{97.21\%} & \textbf{92.43\%} \\
   PDPG       & 65.33\%          & 51.01\%          \\
   \bottomrule
  \end{tabular}
  \subcaption{\textbf{Grid world:} final 1k epochs over 10 seeds,
   (10k checkpts), cf.
   Fig. \ref{fig:Figure1MaxEntropyPerformanceConstrainedVsUnconstrained}
   \circledPurple{1}.
  }
  \label{sub-tab:LateIterateStability-GridWorld}
 \end{minipage}
 \hfill
 \begin{minipage}[t]{0.62\linewidth}
  \centering
  \small
  \setlength{\tabcolsep}{3pt}
  \begin{tabular}{@{}lcccc@{}}
   \toprule
   Method     & SU $>1$ st.      & SU $>60$ st.
              & $\Reg\leq0.1$    & $\Reg\leq0.01$   \\
   \midrule
   PGP (ours) & \textbf{87.23\%} & \textbf{86.11\%}
              & \textbf{96.89\%} & \textbf{94.49\%} \\
   PDPG       & 59.47\%          & 41.64\%
              & 70.03\%          & 53.33\%          \\
   \bottomrule
  \end{tabular}
  \subcaption{\textbf{SafeCartpole:} final 500 epochs over 5 seeds
   (2.5k checkpts), cf. \Cref{fig:CartpoleCombined} \circledPurple{2}.
   \enquote{SU} denotes swing-up and
   \enquote{st.}\ denotes steps.}
  \label{sub-tab:LateIterateStability-Cartpole}
 \end{minipage}
 \caption{\textbf{Last-iterate stability.}
  Fractions of late-training checkpoints satisfying feasibility and
  swing-up; \algo{} is more consistently feasible and exploratory than
  PDPG.}
 \label{tab:LateIterateStability}
\end{table}

\section{Concluding Remarks \& Future Work}
\label{sec:Conclusion}
We propose a single-loop, policy-space penalty method for constrained
maximum-entropy exploration. Our analysis provides global last-iterate
convergence guarantees under strong duality, despite non-convex policy
parameterization and stochastic gradients, and we demonstrate the
scalability of our method on a continuous control task. Future work
includes extensions to non-smooth occupancy-measure constraints.
On the algorithmic side, closing the theory-practice gap in step-size and
batch-size selection remains an important and non-trivial challenge.
Finally, our current algorithm and analysis yield a relatively large
$\widetilde{\mathcal{O}}(\epsilon^{-6})$ sample complexity. We believe this
sample complexity can be reduced by using additional algorithmic
refinements, such as variance reduction or adaptive penalty schemes.
However, a rigorous validation of these directions remains open.

{\small\paragraph{Acknowledgement.} FW is supported by the Prof. Dr.-Ing. Erich Müller-Stiftung, the Almestes
 Stiftung, the Swiss National Centre of Competence in Research (NCCR)
 Automation under grant agreement No. 51NF40\_180545, and the Kortschak
 Scholars Program. IF is supported by ETH AI Center Doctoral Fellowship. The
 authors thank Yarden As for fruitful discussions.} 
 
\newpage \appendix \tableofcontents
\section{Additional Related Work}
\label{sec:AppendixDetailedComparison}
We provide additional related work and make precise why our contributions are
substantively different from the
two most closely related works.

\paragraph{Different Entropy Formulations.}
Other works deal with Rényi entropy formulation
\cite{zhangExplorationMaximizingRenyi2021, yuanRenyiStateEntropy2022},
geometric aware Shannon entropy \cite{guoGeometricEntropicExploration2021},
and f-Divergence entropy regularization
\cite{agarwal$f$PolicyGradientsGeneral2023}.
\citet{leeEfficientExplorationState2020,
 islamMarginalizedStateDistribution2019} recast exploration in terms of
marginal state distribution learning, while
\citet{jainMaximumStateEntropy2023} focus on entropy maximization of a
single trajectory, \citet{zamboniPrincipledUnsupervisedMultiAgent2025} on a
multi-agent RL setting, and \citet{zisselmanExploreGeneralizeZeroShot2023}
use ensemble agents to ensure exploration at test-time.

\paragraph{Optimization Literature.}
Our algorithmic framework is most closely related to penalty and augmented
Lagrangian methods (ALM) for constrained optimization. Such methods are
classical tools with a long history and broad adoption
\cite{courantVariationalMethodsSolution1943,
 fiaccoNonlinearProgrammingSequential1990,
 bertsekasNonlinearProgramming1999, boydConvexOptimization2004,
 nesterovLecturesConvexOptimization2018}, and have been extensively studied
from a first-order complexity perspective in both convex and non-convex
settings \cite{aybatFirstOrderAugmentedLagrangian2012,
 lanIterationcomplexityFirstorderPenalty2013,
 necoaraComplexityFirstorderInexact2019, liAugmentedLagrangianBased2021}.
However, existing analyses typically focus on deterministic optimization,
establish asymptotic convergence or convergence to $\epsilon$-KKT points,
and do not directly apply to entropy-based reinforcement learning, where
non-convexity arises from policy parameterization and gradients are
estimated stochastically.

Recent work has also highlighted naively incorporating constraints as costs
in the objective function may fail to simultaneously ensure feasibility and
optimality in deep learning settings
\cite{ramirezPositionAdoptConstraints2025}. Our results demonstrate that,
when combined with problem-specific gradient structure, quadratic penalty
methods can yield strong guarantees even in stochastic, non-convex RL
problems. A unified view connecting our approach to proximal-point and ALM
frameworks under hidden convexity is developed in
Appendix~\ref{sec:AppendixNonSmoothAndSmoothPenaltyApproach}.

\subsection{Detailed Comparison to \citet{zhangVariationalPolicyGradient2020}}

\paragraph{Problem.}
\citet{zhangVariationalPolicyGradient2020} study \emph{unconstrained}
general-utility RL, i.e., maximizing a convex functional $F(\lambda^\pi)$
without constraints; introducing constraints changes the
optimization landscape fundamentally.
Moreover, Example~2.1 therein is restricted to cumulative discounted reward
in a standard constrained MDP via a generic penalty function, leaving both
the concrete choice of penalty and the specific objective unspecified.

\paragraph{Algorithm.}
The core algorithmic novelty is not in running biased SGD per se, but in
\emph{how} we construct the gradient estimator for the penalized objective.
In generic stochastic penalty and ALM methods, the chain rule through
$\beta\,\pen{\Reg(\lambda(\theta))}$ canonically requires separate
stochastic estimators for the constraint function value
$\Reg(\lambda(\theta))$ and its sub-gradient; their product creates a
cross-term bias that does not vanish with batch size. This fundamental bias
mismatch has forced the prior stochastic penalty and ALM methods to either
(i)~restrict to deterministic constraints, (ii)~use nested sampling or
multiple gradient evaluations per iteration, or (iii)~rely on momentum or
variance-tracking. Our pseudo-reward construction (Line~5 of
\Cref{algo:Penalty}) circumvents this entirely by reinterpreting
\emph{both} gradient terms as policy gradients with respect to constructed
pseudo-rewards via the Policy Gradient Theorem. This construction is
specific to the RL structure and does not appear in, and cannot be derived
from, \citet{zhangVariationalPolicyGradient2020}, which operates in the
unconstrained setting.

\paragraph{Theory.}
\citet{zhangVariationalPolicyGradient2020} explicitly assume \emph{exact}
knowledge of the Q-function (cf.\ the Remark after their Theorem~2);
our analysis operates in the realistic stochastic setting, requiring a joint
control of truncation bias $\mathcal{O}(\gamma^H)$ and Monte-Carlo variance
$\mathcal{O}(B^{-1/2})$ (cf.\ \Cref{lem:BiasVarianceGradientEstimator}).
Beyond this stochastic extension, two further nontrivial analytical steps
are required:
\begin{itemize}
 \item \textbf{Preservation of hidden convexity under the penalty}
       (\Cref{lem:AppendixSmoothnessBoundedGradientsPenaltyLambda}):
       Showing that composing $\mathcal{R}(\lambda)$ with the quadratic penalty
       $[\innerEmpty{}]_+^2$
       preserves the smoothness and hidden convexity structure, and deriving the
       $\beta$-dependent smoothness constants $(\ell_{P, \lambda}, L_{P, \lambda}, L_{P, \lambda, \infty})$,
       requires careful analysis of the interaction between the penalty nonlinearity and the
       occupancy-measure geometry. We are not aware of any prior work establishing such a result.
 \item \textbf{Translation under strong duality}
       (\Cref{cor:FinalRateTranslation},
       \Cref{cor:AppendixConstraintViolationQuadraticPenaltyTranslation}):
       Converting the penalty gap
       $\penOpt{2,\beta}(\thetaN) \leq \epsilon$ into simultaneous objective and
       constraint guarantees requires a quantitative relationship between $\beta(\varepsilon)$ and
       the dual multiplier $\eta^*$. This translation, yielding the specific scaling
       $\beta = \mathcal{O}(\varepsilon^{-1})$, is new to hidden convex optimization.
\end{itemize}

\subsection{Detailed Comparison to \citet{fatkhullinStochasticOptimizationHidden2025}}

\citet{fatkhullinStochasticOptimizationHidden2025} analyze
\emph{unconstrained} stochastic optimization under hidden convexity with an
\emph{unbiased} gradient oracle. Thus constraints, stochastic biased gradients,
and their interaction with hidden convexity are not addressed therein.
We borrow the hidden-convexity framework but face \textbf{two difficulties} absent
in \citet{fatkhullinStochasticOptimizationHidden2025}.
\begin{itemize}
 \item First, our gradient estimator introduces bias through two distinct
       mechanisms: (i)~trajectory truncation at horizon $H$, and (ii)~the
       occupancy-measure estimation error propagated through the penalty gradient
       via the pseudo-rewards, which are evaluated at the \emph{estimated}
       occupancy measure $\lambdaHat$ rather than the true $\lambda_H$.
 \item Second, the penalty introduces a composite structure whose chain rule
       canonically induces two independently estimated quantities which we resolve
       by introducing our pseudo-reward construction.
\end{itemize}
Controlling the chain of penalty smoothness preservation, through occupancy-measure
error propagation, to the final bias-variance tradeoff, is the technical
core of our analysis beyond a direct adaptation of
\citet{fatkhullinStochasticOptimizationHidden2025} and requires carefully
balancing the horizon $H$, batch size $B$,
as well as the the step-size $\eta$ is a three-way trade-off
absent from the unbiased oracle model of
\citet{fatkhullinStochasticOptimizationHidden2025}.

Furthermore, translating the penalty guarantees to the constrained problem
requires the scaling $\beta = \OO{\epsilon^{-1}}$ from strong duality
(\Cref{ass:StrongDuality}), a condition entirely absent in the
unconstrained framework. To the best of our knowledge, no prior work
handles the combination of hidden convex constraints and stochastic biased
gradients.

\section{Useful Technical Lemmas}
\label{sec:AppendixGradientEstimates}
\subsection{Bias \& Variance of Gradient Estimator, Smooth Case, (In-)Finite Horizon}
In this section, we will derive the bias and the variance of our gradient
estimator \eqref{eq:SubGradientEstimator} for $F \in \{\HH, \Reg\}$. For a
fixed $F$ in this subsection, we abbreviate its function-specific constants
by $\ell_\lambda\define\ell_{F,\lambda}$, $L_\lambda\define L_{F,\lambda}$,
and $L_{\lambda,\infty}\define L_{F,\lambda,\infty}$, respectively.
\Cref{Appendix-sec:PenaltyApproachAnalysis} shows that the composition of
such a function with our quadratic penalty maintains the important property
of $L$-smoothness and hidden convexity.

\begin{lemma}[Properties of \eqref{eq:SubGradientEstimator}]
 \label{lem:AppendixScoreFunction}
 Under \Cref{ass:Softmax} and  \Cref{ass:Lsmoothness}, the following
 statements hold:
 \begin{itemize}
  \item $\forall \theta \in \Theta, \forall (\ss,\aa) \in \SS \times \AA$
        we have
        \begin{align*}
         \norm*{\nabla_\theta \log \pi_\theta(\aa\vert \ss)}       & \leq 2 \ell_{\psi}                            \\
         \norm*{\nabla^2_{\theta} \log \pi_{\theta}(\aa\vert \ss)} & \leq 2 (L_\psi + \ell_\psi^2)                 \\
         \norm*{\nabla_\theta F(\lambda(\theta))}                  & \leq \frac{2\ell_\psi \ell_\lambda}{1-\gamma}
         \defineRev \ell_{F, \theta}
        \end{align*}
  \item $\forall \theta_1, \theta_2 \in \Theta$ the (truncated) state-action occupancy measure
        is Lipschitz with respect to the parametrization, i.e.
        \begin{align*}
         \norm*{\lambda^{\pi_{\theta_1}} - \lambda^{\pi_{\theta_2}}}_1
          & \leq \frac{2\ell_\psi}{1-\gamma} \norm{\theta_1 - \theta_2} \\
         \norm*{\lambda_H(\theta_1) - \lambda_H(\theta_2)}_1
          & \leq \frac{2\ell_\psi}{1-\gamma} \norm{\theta_1 - \theta_2} \\
        \end{align*}
  \item Concatenating $F$ with the parametrization maintains smoothness, i.e. the
        function $\theta \mapsto F(\lambda^{\pi_\theta})$ is $L_{F, \theta}$-smooth
        with
        \begin{align*}
         L_{F, \theta} \define \frac{4 L_{\lambda, \infty} \ell_\psi^2}{(1-\gamma)^2}
         + \frac{4(1+\gamma)\ell_\lambda\ell_\psi^2}{(1-\gamma)^2}
         + \frac{2\ell_\lambda(L_\psi + \ell_\psi^2)}{1-\gamma}.
        \end{align*}
 \end{itemize}
\end{lemma}
\begin{proof}
 See Lem. H.1 in \citet{barakatReinforcementLearningGeneral2023a} and Lem. 5.3
 in \citet{zhangConvergenceSampleEfficiency2021}.
 In particular, using
 \begin{align*}
  (1-\gamma)\sum_{t=0}^{\infty}(t+1)\gamma^t
  =\frac{1}{1-\gamma},\quad
  (1-\gamma)\sum_{t=0}^{\infty}(t+1)^2\gamma^t
  =\frac{1+\gamma}{(1-\gamma)^2},
 \end{align*}
 gives the stated gradient and occupancy Lipschitz bounds. Applying
 the chain rule to $F\circ\lambda$, the Lipschitz variation of
 $\nabla_\lambda F$ gives the first term in $L_{F,\theta}$, while the
 squared-score and score-Hessian terms give the second and third terms,
 respectively.
\end{proof}

The following lemma allows us to run (biased) SGD as a solver for the
unconstrained penalty problem.
\begin{lemma}[Gradient Estimator: Bias and Variance]
 \label{lem:BiasVarianceGradientEstimator}
 Under \Cref{ass:Softmax} and \Cref{ass:Lsmoothness}, the full-batch
 gradient estimator
 $\gApprox_{H,B}(\nabla_\lambda F, \theta,
  (\tau_i)_{i\in\setm{B}})$ in \eqref{eq:SubGradientEstimator} satisfies,
 for every $\theta\in\Theta$,
 \begin{align*}
  \norm*{\Exp{\gApprox_{H,B}(\nabla_\lambda F,\theta,
   (\tau_i)_{i\in\setm{B}})}-\nabla_\theta F(\lambda(\theta))}
   & \leq D_H\gamma^H+\frac{D_{\gApprox}}{\sqrt B}, \\
  \Exp{\norm*{\Exp{\gApprox_{H,B}(\nabla_\lambda F,\theta,
   (\tau_i)_{i\in\setm{B}})}
   -\gApprox_{H,B}(\nabla_\lambda F,\theta,
   (\tau_i)_{i\in\setm{B}})}^2}
   & \leq\frac{\sigma_H^2}{B},
 \end{align*}
 where
 \begin{align*}
  D_H          & \define 2\ell_\psi\left[
   \frac{L_{\lambda,\infty}}{1-\gamma}
  +\ell_\lambda\left(H+1+\frac{\gamma}{1-\gamma}\right)\right], \\
  \kappa_H     & \define 2\ell_\psi(1-\gamma)
  \sum_{t=0}^{H-1}(t+1)\gamma^t
  \leq\frac{2\ell_\psi}{1-\gamma},                              \\
  D_{\gApprox} & \define\kappa_H L_\lambda,                     \\
  \sigma_H^2   & \define
  2\kappa_H^2\left(\ell_\lambda^2+L_\lambda^2\right).
 \end{align*}
 Importantly, the \emph{same trajectory batch} is used in the occupancy
 estimate and in the REINFORCE term, i.e.
 no independence between these two quantities is assumed.
\end{lemma}
\begin{proof}
 For any fixed pseudo-reward $\rew$ with
 $\norm{\rew}_\infty\leq\ell_\lambda$, the normalized policy-gradient
 expansion satisfies
 \begin{align*}
  \norm*{\Exp{g_H(\theta,\tau,\rew)}}
  \leq\frac{2\ell_\psi}{1-\gamma}\norm{\rew}_\infty.
 \end{align*}
 Set $q_H\define\nabla_\lambda F(\lambda_H(\theta))$ and
 $q\define\nabla_\lambda F(\lambda(\theta))$. Adding and subtracting
 $[\nabla_\theta\lambda_H(\theta)]^\top q$ gives
 \begin{align*}
  \nabla_\theta F(\lambda_H(\theta))-\nabla_\theta F(\lambda(\theta))
  =[\nabla_\theta\lambda_H(\theta)]^\top(q_H-q)
  +\left([\nabla_\theta\lambda_H(\theta)]^\top
  -[\nabla_\theta\lambda(\theta)]^\top\right)q.
 \end{align*}
 The first term is the pseudo-reward approximation error, while the second
 is the truncation tail for the fixed pseudo-reward $q$.
 Using $\norm{\lambda-\lambda_H}_1=\gamma^H$ for the pseudo-reward
 error and bounding the omitted fixed-reward tail gives
 \begin{align*}
  \norm*{\nabla_\theta F(\lambda_H(\theta))
   -\nabla_\theta F(\lambda(\theta))}
  \leq\frac{2\ell_\psi L_{\lambda,\infty}}{1-\gamma}\gamma^H
  +2\ell_\psi\ell_\lambda(1-\gamma)
  \sum_{t=H}^{\infty}(t+1)\gamma^t
  =D_H\gamma^H.
 \end{align*}
 For each trajectory $\tau_i$, define its normalized empirical occupancy
 and its trajectory-wise linear policy-gradient operator by
 \begin{align*}
  X_i\define(1-\gamma)\sum_{t=0}^{H-1}\gamma^t
  \delta_{\ss_t^{(i)},\aa_t^{(i)}},\quad
  \mathsf M_i\rew\define g_H(\theta,\tau_i,\rew).
 \end{align*}
 Let
 \begin{align*}
  \bar X_B & \define\frac1B\sum_{i=1}^B X_i,
           & \bar{\mathsf M}_B               & \define\frac1B\sum_{i=1}^B\mathsf M_i,
           & \mathsf M                       & \define\Exp{\mathsf M_i},
 \end{align*}
 and with $\bar X_B=\lambdaHat((\tau_i)_{i\in\setm{B}})$,
 $\Exp{X_i}=\lambda_H$, and for every fixed $\rew$, it holds that
 \begin{align*}
  \mathsf M\rew=[\nabla_\theta\lambda_H(\theta)]^\top\rew.
 \end{align*}
 Moreover, reordering the two sums in $g_H$ and using
 $\norm{\nabla_\theta\log\pi_\theta(\aa\mid\ss)}\leq2\ell_\psi$
 yields the pathwise operator bound
 \begin{align*}
  \norm{\mathsf M_i\rew}
   & \leq2\ell_\psi(1-\gamma)
  \sum_{t=0}^{H-1}(t+1)\gamma^t\norm{\rew}_\infty
  =\kappa_H\norm{\rew}_\infty.
 \end{align*}
 Since $\norm{X_i}_1=1-\gamma^H\leq1$, independence across trajectories
 gives
 \begin{align}
  \Exp{\norm{\bar X_B-\lambda_H}_2^2}
   & =\frac1B\Exp{\norm{X_1-\lambda_H}_2^2}
  =\frac1B\left(\Exp{\norm{X_1}_2^2}-\norm{\lambda_H}_2^2\right)
  \leq\frac1B\Exp{\norm{X_1}_1^2}\notag     \\
   & =\frac{(1-\gamma^H)^2}{B}\leq\frac1B.
  \label{eq:FullBatchOccupancyMSE}
 \end{align}

 \textbf{Bounding the Bias.} Set $q(\lambda)\define\nabla_\lambda F(\lambda)$. The common
 pseudo-reward in \eqref{eq:SubGradientEstimator} gives the exact
 factorization
 \begin{align*}
  \gApprox_{H,B}(\nabla_\lambda F,\theta,
  (\tau_i)_{i\in\setm{B}})
  =\bar{\mathsf M}_B q(\bar X_B).
 \end{align*}
 Consequently,
 \begin{align*}
  \Exp{\bar{\mathsf M}_Bq(\bar X_B)}
  -\nabla_\theta F(\lambda_H(\theta))
   & =\Exp{(\bar{\mathsf M}_B-\mathsf M)q(\lambda_H)}
  +\Exp{\bar{\mathsf M}_B(q(\bar X_B)-q(\lambda_H))}     \\
   & =\Exp{\bar{\mathsf M}_B(q(\bar X_B)-q(\lambda_H))},
 \end{align*}
 because $\Exp{\bar{\mathsf M}_B}=\mathsf M$. Thus, without requiring
 independence between $\bar{\mathsf M}_B$ and $\bar X_B$,
 we obtain the upper bound
 \begin{align*}
  \norm*{\Exp{\gApprox_{H,B}}-\nabla_\theta F(\lambda_H(\theta))}
  \leq\kappa_H L_\lambda\Exp{\norm{\bar X_B-\lambda_H}_2}
  \leq\frac{\kappa_HL_\lambda}{\sqrt B}
  =\frac{D_{\gApprox}}{\sqrt B}.
 \end{align*}
 Combining this estimate with the truncation bound proves the bias claim.

 \textbf{Bounding the Variance.}
 The expected squared distance from the mean is no larger
 than that from the deterministic vector
 $\mathsf Mq(\lambda_H)=\nabla_\theta F(\lambda_H(\theta))$. Therefore,
 \begin{align*}
  \Exp{\norm*{\gApprox_{H,B}-\Exp{\gApprox_{H,B}}}^2}
   & \leq
  \Exp{\norm*{\bar{\mathsf M}_Bq(\bar X_B)-\mathsf Mq(\lambda_H)}^2} \\
   & \quad\leq
  2\Exp{\norm*{(\bar{\mathsf M}_B-\mathsf M)q(\lambda_H)}^2}
  +2\Exp{\norm*{\bar{\mathsf M}_B(q(\bar X_B)-q(\lambda_H))}^2},
 \end{align*}
 where the first term is an average of independent, centered, fixed-pseudo-reward
 policy-gradient operators, so
 \begin{align*}
  \Exp{\norm*{(\bar{\mathsf M}_B-\mathsf M)q(\lambda_H)}^2}
  =\frac1B\Exp{\norm*{(\mathsf M_1-\mathsf M)q(\lambda_H)}^2}
  \leq\frac{\kappa_H^2\ell_\lambda^2}{B},
 \end{align*}
 and the second term can be controlled pathwise
 using \eqref{eq:FullBatchOccupancyMSE}, yielding
 \begin{align*}
  \Exp{\norm*{\bar{\mathsf M}_B(q(\bar X_B)-q(\lambda_H))}^2}
  \leq\kappa_H^2L_\lambda^2
  \Exp{\norm{\bar X_B-\lambda_H}_2^2}
  \leq\frac{\kappa_H^2L_\lambda^2}{B}.
 \end{align*}
 In total, combining the last three bounds, yields
 \begin{align*}
  \Exp{\norm*{\gApprox_{H,B}-\Exp{\gApprox_{H,B}}}^2}
  \leq\frac{2\kappa_H^2(\ell_\lambda^2+L_\lambda^2)}{B}
  =\frac{\sigma_H^2}{B},
 \end{align*}
 which proves the variance claim and concludes the proof.
\end{proof}

\subsection{Convergence of Biased SGD for \eqref{eq:MainPenalty}}
We start with the following auxilary result, yielding the
Lipschitz constant of the penalty function's gradient.

\begin{lemma}[Smoothness and Bounded Gradients of Penalty
  function in $\lambda$]
 \label{lem:AppendixSmoothnessBoundedGradientsPenaltyLambda}
 With $\penGenNot = \pen{\innerEmpty{}}$, under
 \Cref{ass:Softmax} and
 \Cref{ass:Lsmoothness}, the
 function $\lambda \mapsto P(\lambda)
  =-\HH(\lambda)+\beta\pen{\Reg(\lambda)}$
 fulfills the inequalities of \Cref{ass:Lsmoothness} with
 \begin{align*}
  \ell_{P,\lambda}
   & \define \ell_{\HH,\lambda}
  +2\beta\RegMax\ell_{\Reg,\lambda},                          \\
  L_{P,\lambda}
   & \define L_{\HH,\lambda}
  +2\beta\left(\RegMax L_{\Reg,\lambda}
  +\sqrt{\card{\SS}\card{\AA}}\,\ell_{\Reg,\lambda}^2\right), \\
  L_{P,\lambda,\infty}
   & \define L_{\HH,\lambda,\infty}
  +2\beta\left(\RegMax L_{\Reg,\lambda,\infty}
  +\ell_{\Reg,\lambda}^2\right),
 \end{align*}
 with the upper bound $\RegMax \define \max_{\lambda \in \Lambda} \abs{\Reg(\lambda)}$.
 In particular, the smoothness constants of the entropy function
 enter additively and are not multiplied by $\beta$.
\end{lemma}
\begin{proof}
 Writing $r(\lambda)=\plus{\Reg(\lambda)}$, the chain rule gives
 \begin{align*}
  \nabla_\lambda P(\lambda)
  =-\nabla_\lambda\HH(\lambda)
  +2\beta r(\lambda)\nabla_\lambda\Reg(\lambda).
 \end{align*}
 The gradient bound follows from $0\leq r\leq\RegMax$.
 The penalty-gradient difference is bounded by
 \begin{align*}
  2\beta r(\lambda)
  \norm{\nabla\Reg(\lambda)-\nabla\Reg(\lambda^\prime)}_\infty
  +2\beta\abs{r(\lambda)-r(\lambda^\prime)}
  \norm{\nabla\Reg(\lambda^\prime)}_\infty.
 \end{align*}
 The positive-part map is non-expansive, and hence
 $\abs{r(\lambda)-r(\lambda^\prime)}
  \leq\ell_{\Reg,\lambda}\norm{\lambda-\lambda^\prime}_1$.
 For the Euclidean bound, the assumed infinity-norm gradient bound gives
 $\norm{\nabla\Reg}_2\leq
  \sqrt{\card{\SS}\card{\AA}}\,\ell_{\Reg,\lambda}$.
 Combining these inequalities with the function-specific smoothness
 constants proves all three claims.
\end{proof}

\begin{corollary}[Gradient Bound and Smoothness of Penalty Function]
 \label{cor:GradientBoundSmoothnessPenalty}
 Let
 \Cref{ass:Softmax} and
 \Cref{ass:Lsmoothness} hold.
 For the penalty function \eqref{eq:MainPenalty}, we get the
 following constants for the gradients with respect to
 $\theta \in \Theta$ in the parameter space:
 \begin{align*}
  \ell_{P,\theta}
   & =\frac{2\ell_\psi\ell_{P,\lambda}}{1-\gamma},          \\
  L_{P,\theta}
   & =\frac{4L_{P,\lambda,\infty}\ell_\psi^2}{(1-\gamma)^2}
  +\frac{4(1+\gamma)\ell_{P,\lambda}\ell_\psi^2}{(1-\gamma)^2}
  +\frac{2\ell_{P,\lambda}(L_\psi+\ell_\psi^2)}{1-\gamma},
 \end{align*}
 as the upper bound of the norm of the gradients
 and the Lipschitz constant of the gradient, respectively.
\end{corollary}
\begin{proof}
 Apply \Cref{lem:AppendixScoreFunction} to the penalty objective
 with $\ell_\lambda$, $L_{\lambda,\infty}$ replaced by
 $\ell_{P,\lambda}$, $L_{P,\lambda,\infty}$ from
 \Cref{lem:AppendixSmoothnessBoundedGradientsPenaltyLambda}.
\end{proof}

\begin{corollary}[Bias \& Variance of Penalty Function Gradient Estimator]
 \label{lem:PenaltyBiasVarianceGradientEstimator}
 Under \Cref{ass:Softmax} and \Cref{ass:Lsmoothness}, the full-batch
 gradient estimator
 $\gApprox_{H,B}(\nabla_\lambda P,\theta,
  (\tau_i)_{i\in\setm{B}})$ in \eqref{eq:SubGradientEstimator} satisfies,
 for every $\theta\in\Theta$,
 \begin{align*}
  \norm*{\Exp{\gApprox_{H,B}(\nabla_\lambda P,\theta,
   (\tau_i)_{i\in\setm{B}})}-\nabla_\theta P(\lambda(\theta))}
   & \leq D_{P,H}\gamma^H+\frac{D_{P,\gApprox}}{\sqrt B}, \\
  \Exp{\norm*{\Exp{\gApprox_{H,B}(\nabla_\lambda P,\theta,
   (\tau_i)_{i\in\setm{B}})}
   -\gApprox_{H,B}(\nabla_\lambda P,\theta,
   (\tau_i)_{i\in\setm{B}})}^2}
   & \leq\frac{\sigma_{P,H}^2}{B},
 \end{align*}
 where
 \begin{align*}
  D_{P,H}        & \define 2\ell_\psi\left[
   \frac{L_{P,\lambda,\infty}}{1-\gamma}
  +\ell_{P,\lambda}\left(H+1+\frac{\gamma}{1-\gamma}\right)\right], \\
  D_{P,\gApprox} & \define\kappa_H L_{P,\lambda},                   \\
  \sigma_{P,H}^2 & \define
  2\kappa_H^2\left(\ell_{P,\lambda}^2+L_{P,\lambda}^2\right),
 \end{align*}
 and $\kappa_H$ is defined in \Cref{lem:BiasVarianceGradientEstimator}.
\end{corollary}
\begin{proof}
 Apply \Cref{lem:BiasVarianceGradientEstimator} to the penalty objective
 with the bounds $\ell_{P,\lambda}$, $L_{P,\lambda}$, and
 $L_{P,\lambda,\infty}$ from
 \Cref{lem:AppendixSmoothnessBoundedGradientsPenaltyLambda}.
\end{proof}

Next, we are ready to prove the main convergence result, namely applying
biased SGD on our penalty problem \eqref{eq:MainPenalty}.
\subsubsection{Proof of \Cref{thm:MainResultPenalty}}
\label{thm:MainResultPenalty-PROOF}
\begin{proof}
 The key modification of the proof of Thm. 7 in
 \cite{fatkhullinStochasticOptimizationHidden2025} is to construct an element
 of the form $\xx_{\alpha} \define \lambda^{-1}_{\VV(\xx)}\left((1-\alpha)
  \lambda(\xx) + \alpha \lambda(\xxOpt)\right)$, $\xx \in \XX$,
 with $\xxOpt$ being the
 global solution to our constrained problem \eqref{eq:MainProblemRL}.
 The constant $\alpha$ is a minimum of the value in
 \Cref{thm:BiasedSGDHiddenConvex} and $\epsLam$ of
 \Cref{ass:ParametrizationSOAM}. Invoking \Cref{thm:BiasedSGDHiddenConvex} yields
 $\Exp{\penOpt{2,\beta}(\xxNnext)}
  \leq\epsilon$
 after
 \begin{align*}
  N
   & \geq 3\DDUsq\lamLipSq
  \frac{L_{P,\theta}}{\epsilon}
  \log\left(\frac{3\penOpt{2,\beta}(\xxZero)}{\epsilon}\right)
  =\OOTilde{\frac{L_{P,\theta}}{\epsilon}}.
 \end{align*}
 Our \Cref{thm:BiasedSGDHiddenConvex} requires both, the squared
 gradient bias and its variance to be at most
 $\epsilon^2/(9\DDUsq\lamLipSq)$ which we achieve by setting
 \begin{align*}
  B\geq\max\left\{
  \frac{36\DDUsq\lamLipSq D_{P,\gApprox}^2}{\epsilon^2},
  \frac{9\DDUsq\lamLipSq\sigma_{P,H}^2}{\epsilon^2}
  \right\}=\OO{\epsilon^{-2}},
 \end{align*}
 and
 \begin{align*}
  D_{P,H}\gamma^H
  \leq\frac{\epsilon}{6\DDU\lamLip},
 \end{align*}
 In the notation of
 \Cref{lem:PenaltyBiasVarianceGradientEstimator} and by absorbing the fixed
 hidden-convex geometry constants, we obtain
 \begin{align*}
  B=\OO{\frac{L_{P,\lambda}^2+\ell_{P,\lambda}^2}
   {(1-\gamma)^2\epsilon^2}}.
 \end{align*}
 Finally, the exact expression for $D_{P,H}$ is affine in $H$ and
 the truncation condition is achieved with
 \begin{align*}
  H=\OOTilde{\frac{1}{-\log\gamma}
   \log\left(\frac{\DDU\lamLip
    (L_{P,\lambda,\infty}+\ell_{P,\lambda})}
   {\epsilon}\right)},
 \end{align*}
 which concludes the proof.
\end{proof}

\subsubsection{Proof of \Cref{cor:FinalRateTranslation}}
\label{cor:FinalRateTranslation-PROOF}
\begin{proof}
 The proof is similar to the proof of
 \Cref{thm:MainResultPPPMQuadraticPenaltyConvergenceGuarantee}.
 Define $\beta \define \beta(\epsilon) \define (\nuOpt + 1)\frac{\nuOpt
   + \sqrt{\nuOpt^2 + 2}}{\epsilon} = \OO{\epsilon^{-1}}$. With
 a target accuracy of $\epsTilde \define \frac{1}{\beta(\epsilon)} \leq \epsilon$,
 \Cref{thm:MainResultPenalty} guarantees that after
 $N = \OO{\frac{1}{\epsTilde} + \frac{\beta}{\epsTilde}} = \OO{\frac{1}{\epsilon^2}}$ iterations
 \begin{align*}
  \Exp{\penOpt{2, \beta}(\xxN)} \leq \epsTilde \leq \epsilon
 \end{align*}
 holds. Under strong duality, \Assref{ass:StrongDuality}, with
 an optimal Lagrangian multiplier $\nuOpt \geq 0$,
 \Cref{cor:AppendixConstraintViolationQuadraticPenaltyTranslation} shows
 that this implies
 \begin{align*}
  \Exp{\plus{F_2(\xxN)}} \leq \frac{\nuOpt + \sqrt{\nuOpt^2
    + 2\beta(\epsilon) \epsTilde}}{\beta(\epsilon)}
  = \frac{\nuOpt + \sqrt{\nuOpt^2 + 2}}{\beta(\epsilon)}
  = \frac{\epsilon}{\nuOpt + 1} \leq \epsilon.
 \end{align*}
 Bounding the corresponding optimality gap for $F_1$ follows by the fact that
 \begin{align*}
  \Exp{-F_1(\xxN) + \Fopt} \leq \Exp{\penOpt{2, \beta}(\xxN)} \leq \epsilon
 \end{align*}
 as well as
 \begin{align*}
  -\Exp{-F_1(\xxN) + \Fopt} \leq \nuOpt \Exp{\plus{F_2(\xxN)}}\leq
  \nuOpt \frac{\epsilon}{\nuOpt + 1}
  \leq \epsilon,
 \end{align*}
 by \Cref{lemma:AppendixTranslatePenaltyGuaranteeToFuncGapConstrViol}. In total
 we obtain the two-sided bound on the expected optimality gap:
 \begin{align*}
  -\epsilon \leq \Exp{F_1(\xxN) - \Fopt} \leq \epsilon.
 \end{align*}
 Substituting the batch size and trajectory length required by
 \Cref{thm:MainResultPenalty} gives a total sample complexity of
 $\OOTilde{\epsilon^{-6}}$.
\end{proof}

\newpage
\section{Softmax Policy Density and Smoothed Entropy}
\label{sec:AppendixConcretePolicyClass}

This section contains two independent results. First, we provide a
self-contained policy-class example satisfying the parametrization
assumptions of \Cref{thm:MainResultPenalty}. Second, we analyze the
smoothed entropy $\HH_\sigma$ and its sample complexity. The smoothing
analysis does not rely on the softmax density result, and vice versa.

\subsection{Anchored softmax density}
We work with the anchored direct-softmax parametrization in a finite
tabular MDP and compute its occupancy-map constants explicitly.

\begin{proof}(of \Cref{prop:SoftmaxApproximationStationaryPolicies})
 Without loss of generality it
 suffices to consider $\epsilon\in(0,1]$. Fix the same reference
 action $\aa_0\in\AA$ for every state and, for an arbitrary stationary policy
 $\pi^*\in\Pi$, define the full-support approximation
 \begin{align*}
  \pi^*_{\epsilon}(\aa\vert\ss)
  \define \left(1-\frac{\epsilon}{2}\right)\pi^*(\aa\vert\ss)
  + \frac{\epsilon}{2\card{\AA}}.
 \end{align*}
 For every $\ss\in\SS$, this construction satisfies
 $\norm{\pi^*_{\epsilon}(\innerEmpty{}\vert\ss)
   -\pi^*(\innerEmpty{}\vert\ss)}_1\leq\epsilon$ and
 $\pi^*_{\epsilon}(\aa\vert\ss)\geq
  \epsilon/(2\card{\AA})$ for all $\aa\in\AA$. Set the reference logit to
 zero and, for $\aa\neq\aa_0$, and choose
 \begin{align*}
  \theta_{\ss,\aa}
  \define \log\frac{\pi^*_{\epsilon}(\aa\vert\ss)}
  {\pi^*_{\epsilon}(\aa_0\vert\ss)}.
 \end{align*}
 Substitution into the anchored softmax shows that it represents
 $\pi^*_{\epsilon}$ exactly. Importantly, this anchoring trick
 does not remove the
 reference action: its probability is
 $\pi_\theta(\aa_0\vert\ss)=
  [1+\sum_{\aa\neq\aa_0}\exp(\theta_{\ss,\aa})]^{-1}>0$; only its redundant
 logit coordinate is fixed. We use the rescaling to ensure the mapping
 between the paramter space and the occupancy measure fulfills
 \Cref{ass:ParametrizationSOAM}, cf.
 \Cref{lem:AnchoredSoftmaxFulfillsAssumptions}. Moreover,
 $\abs{\theta_{\ss,\aa}}\leq
  \log(2\card{\AA}/\epsilon)$, so one may take
 $R=\log(2\card{\AA}/\epsilon)=
  \OO{\log(\card{\AA}/\epsilon)}$, concluding the proof.
\end{proof}

\begin{lemma}[Softmax Parametrization fulfills
  \Cref{ass:Softmax} and \ref{ass:ParametrizationSOAM}]
 \label{lem:AnchoredSoftmaxFulfillsAssumptions}
 Let $\min_{\ss\in\SS}\mu_0(\ss)>0$, fix a reference action
 $\aa_0\in\AA$, and let
 $\Theta_R=[-R,R]^{\SS\times(\AA\setminus\{\aa_0\})}$ for $R>0$. Define
 \begin{align*}
  \psi(\ss,\aa;\theta)\define
  \begin{cases}
   \theta_{\ss,\aa}, & \aa\neq\aa_0, \\
   0,                & \aa=\aa_0,
  \end{cases}
  \qquad \textrm{and}\qquad
  \pi_\theta(\aa\vert\ss)
  \define \frac{\exp(\psi(\ss,\aa;\theta))}
  {1+\sum_{\aa^\prime\neq\aa_0}\exp(\theta_{\ss,\aa^\prime})}.
 \end{align*}
 Then the parameter-to-occupancy mapping
 $\theta\mapsto\lambda^{\pi_\theta}$ satisfies
 \Cref{ass:Softmax} and \Cref{ass:ParametrizationSOAM}.
\end{lemma}
\begin{proof}
 The anchored score parametrization is linear in $\theta$, with
 $\norm{\nabla_\theta\psi(\ss,\aa;\theta)}\leq 1$ and
 $\nabla^2_\theta\psi(\ss,\aa;\theta)=0$. Hence,
 \Cref{ass:Softmax} holds with $\ell_\psi=1$ and $L_\psi=0$.
 Let $\Lambda_R\define
  \{\lambda^{\pi_\theta}:\theta\in\Theta_R\}$. Since the discounted state
 occupancy satisfies
 \begin{align*}
  d^{\pi_\theta}(\ss)
  \define\sum_{\aa\in\AA}\lambda^{\pi_\theta}(\ss,\aa)
  \geq (1-\gamma)\mu_0(\ss)>0,
 \end{align*}
 the occupancy measure uniquely determines the policy and its anchored
 logits through
 \begin{align*}
  \pi_\theta(\aa\vert\ss)
  =\frac{\lambda^{\pi_\theta}(\ss,\aa)}
  {\sum_{\aa^\prime}\lambda^{\pi_\theta}(\ss,\aa^\prime)},
  \qquad
  \theta_{\ss,\aa}
  =\log\frac{\lambda^{\pi_\theta}(\ss,\aa)}
  {\lambda^{\pi_\theta}(\ss,\aa_0)}
  \quad(\aa\neq\aa_0).
 \end{align*}
 This proves injectivity. To obtain a uniform Lipschitz inverse, note
 that every component of every $\lambda\in\Lambda_R$ is bounded below by
 \begin{align*}
  m_R\define
  (1-\gamma)\min_{\ss\in\SS}\mu_0(\ss)
  \frac{\exp(-R)}{1+(\card{\AA}-1)\exp(R)}>0.
 \end{align*}
 The logarithm is $(1/m_R)$-Lipschitz on $[m_R,1]$; therefore the
 displayed inverse is Lipschitz on $\Lambda_R$ with a constant depending only
 on $R$, $\gamma$, $\mu_0$, $\card{\SS}$, and $\card{\AA}$.
 Finally, $\Lambda_R$ is exactly the intersection occupancy
 measure polytope with the linear inequalities
 \begin{align*}
  \exp(-R)\lambda(\ss,\aa_0)
  \leq \lambda(\ss,\aa)
  \leq \exp(R)\lambda(\ss,\aa_0),
  \qquad \ss\in\SS,\ \aa\neq\aa_0.
 \end{align*}
 Consequently, $\Lambda_R$ is convex. Taking
 $\mathcal U_\theta=\Theta_R$ and $\mathcal V_\theta=\Lambda_R$ gives a
 bijection with a Lipschitz inverse, and every convex combination of
 $\lambda^{\pi_\theta}$ with an optimal occupancy in $\Lambda_R$ remains in
 $\Lambda_R$. This verifies all parts of
 \Cref{ass:ParametrizationSOAM}, concluding the proof.
\end{proof}
\begin{theorem}[Final Complexity Guarantee under Anchored Direct
  Softmax Parametrization]
 \label{thm:FinalComplexityAnchoredSoftmax}
 Let $\lambda^*$ be the occupancy measure of an optimal stationary,
 potentially deterministic policy $\pi^*\in\Pi$ of \eqref{eq:MainProblemRL}. Let $\epsilon>0$
 be the target accuracy. Under
 \Cref{ass:Lsmoothness} for the constraint and strong duality, and using the
 anchored direct-softmax parametrization on
 $\Theta_R=[-R,R]^{\SS\times(\AA\setminus\{\aa_0\})}$, run
 \Cref{algo:Penalty} with a step size $\eta=\OO{\epsilon}$, a batch size
 $B=\OO{\epsilon^{-4}}$, a penalty parameter
 $\beta=\OO{\epsilon^{-1}}$, and a trajectory length
 $H=\OO{\log(1/\epsilon)}$. Then
 \begin{align*}
  -\epsilon
  \leq\Exp{\HH(\lambda^{\pi_{\xxN}})-\HH(\lambda^*)}
  \leq\epsilon,\qquad
  \Exp{\plus{\Reg(\lambda^{\pi_{\xxN}})}}\leq\epsilon
 \end{align*}
 after $N=\OOTilde{\epsilon^{-2}}$ iterations, resulting in a total
 of $\OOTilde{\epsilon^{-6}}$ trajectory rollouts of length
 $\OO{\log(1/\epsilon)}$.
\end{theorem}
\begin{proof}
 For any two stationary policies, their normalized discounted
 occupancy measures satisfy
 \begin{align*}
  {\norm{\lambda^\pi-\lambda^{\pi^\prime}}_1
   \leq\frac{1}{1-\gamma}
   \max_{\ss\in\SS}
   \norm{\pi(\innerEmpty{}\vert\ss)
    -\pi^\prime(\innerEmpty{}\vert\ss)}_1.}
 \end{align*}
 Indeed, the changes in the discounted state occupancy and in the
 conditional policy are bounded by
 $\frac{\gamma\rho}{1-\gamma}$ and $\rho$, respectively, whenever the maximum
 policy distance on the right-hand side is at most $\rho$.

 \textbf{(1) Density argument.}
 Choose an approximation accuracy
 $\epsilon_{\rm app}=\OOTilde{\epsilon}$ sufficiently small. By
 \Cref{prop:SoftmaxApproximationStationaryPolicies}, there are
 $R=\OO{\log(\card{\AA}/\epsilon_{\rm app})}$ and
 $\bar\pi\in\{\pi_\theta:\theta\in\Theta_R\}$ whose maximum statewise
 $\ell_1$-distance to $\pi^*$ is at most $\epsilon_{\rm app}$. Hence,
 writing $\bar\lambda=\lambda^{\bar\pi}$ and
 $u=\epsilon_{\rm app}/(1-\gamma)$,
 \begin{align*}
  {\norm{\bar\lambda-\lambda^*}_1\leq u.}
 \end{align*}
 By \Cref{lem:AnchoredSoftmaxFulfillsAssumptions}, this anchored
 policy class satisfies \Cref{ass:Softmax} and
 \Cref{ass:ParametrizationSOAM}.

 \textbf{(2) Running \algo{}.}
 The entropy objective is uniformly continuous on the finite-dimensional
 occupancy simplex. In particular, for
 $d\define\card{\SS}\card{\AA}$, we have
 \begin{align*}
  \abs{\HH(\bar\lambda)-\HH(\lambda^*)}
  =\OO{u\log(d/u)}.
 \end{align*}
 Moreover, Lipschitz continuity of the constraint and
 $\Reg(\lambda^*)\leq0$ give
 $\plus{\Reg(\bar\lambda)}\leq\ell_{\Reg,\lambda}u$. Therefore, for the
 penalty objective $P=-\HH+\beta\pen{\Reg}$ we have
 \begin{align*}
  P(\bar\lambda)-P(\lambda^*)
  \leq\OO{u\log(d/u)}
  +\beta\ell_{\Reg,\lambda}^2u^2
  =\OO{\epsilon},
 \end{align*}
 where the last equality follows by taking
 $\epsilon_{\rm app}=\OOTilde{\epsilon}$ and
 $\beta=\OO{\epsilon^{-1}}$. Thus the policy-approximation and optimization
 errors can be split at the target scale. Since the minimum of $P$ over
 $\Theta_R$ is no larger than $P(\bar\lambda)$, the penalty-gap argument of
 \Cref{thm:MainResultPenalty}, followed by the strong-duality translation in
 \Cref{cor:FinalRateTranslation}, gives the claimed objective and constraint
 bounds.

 \textbf{(3) Complexity.}
 Every occupancy component is bounded away from zero on the compact class
 $\Theta_R$, so the entropy constants required by
 \Cref{ass:Lsmoothness} are finite. Using the constants and the convention
 of \Cref{cor:FinalRateTranslation}, the algorithm requires
 $\eta=\OO{\epsilon}$, $B=\OO{\epsilon^{-4}}$,
 $\beta=\OO{\epsilon^{-1}}$, $H=\OO{\log(1/\epsilon)}$, and
 $N=\OOTilde{\epsilon^{-2}}$. Multiplying $N$ and $B$ yields
 a total of $\OOTilde{\epsilon^{-6}}$ trajectory rollouts.
\end{proof}

\subsection{Smoothed entropy}
\label{sec:AppendixSmoothedEntropy}

\begin{lemma}[Smoothed entropy, penalty constants, and dual translation]
 \label{lem:SmoothedEntropyProperties}
 Let $d\define\card{\SS}\card{\AA}$ and $\sigma\in(0,1]$.
 For every $\lambda\in\RR_+^d$ with $\norm{\lambda}_1\leq1$, the surrogate
 the smoothed surrogate $\HH_\sigma(\lambda)
  \define-\sum_{\ss,\aa}\lambda(\ss,\aa)
  \log\bigl(\lambda(\ss,\aa)+\sigma\bigr)$
 is concave and satisfies
 \begin{align}
  \nabla_i(-\HH_\sigma)(\lambda)
   & =\log(\lambda_i+\sigma)
  +\frac{\lambda_i}{\lambda_i+\sigma},
  \label{eq:SmoothedEntropyGradient}                  \\
  \nabla_{ii}^2(-\HH_\sigma)(\lambda)
   & =\frac{\lambda_i+2\sigma}{(\lambda_i+\sigma)^2}.
  \label{eq:SmoothedEntropyHessian}
 \end{align}
 Consequently, $-\HH_\sigma$ satisfies the three regularity
 inequalities used in \Cref{ass:Lsmoothness} with
 \begin{align}
  \ell_{\HH_\sigma,\lambda}
   & \define 1+\log\frac{2}{\sigma},
   & L_{\HH_\sigma,\lambda}
   & =L_{\HH_\sigma,\lambda,\infty}\define\frac{2}{\sigma}.
  \label{eq:SmoothedEntropyConstants}
 \end{align}
 In particular, \eqref{eq:SmoothedEntropyGradient} equals
 $\log\sigma$ at an empirical zero. Moreover, with
 \begin{align}
  a_\sigma\define\log(1+d\sigma)\leq d\sigma,
 \end{align}
 we have the uniform approximation bound
 \begin{align*}
  0\leq\HH(\lambda)-\HH_\sigma(\lambda)\leq a_\sigma.
 \end{align*}
 For $P_\sigma=-\HH_\sigma+\beta\pen{\Reg}$, its constants can
 therefore be chosen as
 \begin{align}
  \ell_{P_\sigma,\lambda}
   & \leq1+\log\frac{2}{\sigma}
  +2\beta\RegMax\ell_{\Reg,\lambda},
  \label{eq:SmoothedPenaltyGradientConstant} \\
  L_{P_\sigma,\lambda}
   & \leq\frac{2}{\sigma}
  +2\beta\left(\RegMax L_{\Reg,\lambda}
  +\sqrt d\,\ell_{\Reg,\lambda}^2\right),
  \label{eq:SmoothedPenaltyL2Constant}       \\
  L_{P_\sigma,\lambda,\infty}
   & \leq\frac{2}{\sigma}
  +2\beta\left(\RegMax L_{\Reg,\lambda,\infty}
  +\ell_{\Reg,\lambda}^2\right).
  \label{eq:SmoothedPenaltyL1Constant}
 \end{align}
 Finally, let $\lambda^*$ be an optimizer of the original (unsmoothed)
 problem \eqref{eq:MainProblemRL}, let $\nuOpt$ be the multiplier from
 \Cref{ass:StrongDuality}, and set $r(\lambda)=\plus{\Reg(\lambda)}$.
 Then the same original strong-duality assumption implies, that for every
 occupancy $\lambda=\lambda^{\pi_\theta}$ generated by some
 $\theta\in\Theta$, it holds
 \begin{align}
  -\HH_\sigma(\lambda)+\HH_\sigma(\lambda^*)
  +\beta r(\lambda)^2
  \geq\beta r(\lambda)^2-\nuOpt r(\lambda)-a_\sigma.
  \label{eq:SmoothedPenaltyDualLowerBound}
 \end{align}
 Thus, if $\lambda$ satisfies
 $\Exp{-\HH_\sigma(\lambda)+\HH_\sigma(\lambda^*)
   +\beta r(\lambda)^2}\leq\delta$ and
 $s\define\Exp{r(\lambda)}$, then
 \begin{align}
  s
   & \leq\frac{\nuOpt+
   \sqrt{\nuOpt^2+4\beta(\delta+a_\sigma)}}{2\beta},
  \label{eq:SmoothedConstraintTranslation}       \\
  \HH(\lambda^*)-\Exp{\HH(\lambda)}
   & \leq\delta+a_\sigma, \quad \text{and} \quad
  \Exp{\HH(\lambda)}-\HH(\lambda^*)
  \leq\nuOpt s.
  \label{eq:SmoothedObjectiveTranslation}
 \end{align}
\end{lemma}
\begin{proof}
 Differentiating each coordinate of
 $-\HH_\sigma(\lambda)=\sum_i\lambda_i\log(\lambda_i+\sigma)$ gives
 \eqref{eq:SmoothedEntropyGradient}--\eqref{eq:SmoothedEntropyHessian}.
 The Hessian is diagonal and positive semidefinite, and
 \begin{align*}
  0\leq\frac{x+2\sigma}{(x+\sigma)^2}\leq\frac{2}{\sigma}
  \qquad(x\geq0).
 \end{align*}
 The mean-value theorem then gives both smoothness inequalities.
 Since $x+\sigma\in[\sigma,2]$ for $x\in[0,1]$, the gradient formula gives
 $\norm{\nabla\HH_\sigma}_\infty\leq1+\log(2/\sigma)$.

 For the approximation bound, define $f_\sigma(x)=x\log(1+\sigma/x)$ for
 $x>0$ and $f_\sigma(0)=0$. Direct differentiation gives
 $f_\sigma^{\prime\prime}(x)=-\sigma^2/[x(x+\sigma)^2]\leq0$. For
 $g(m)=m\log(1+d\sigma/m)$, $g^{\prime}(m)=\log(1+t)-t/(1+t)\geq0$ with
 $t=d\sigma/m$. Thus, if $m=\norm{\lambda}_1>0$, concavity of $f_\sigma$ and
 $m\leq1$ give
 \begin{align*}
  \HH(\lambda)-\HH_\sigma(\lambda)
  =\sum_i f_\sigma(\lambda_i)
  \leq m\log\left(1+\frac{d\sigma}{m}\right)
  \leq\log(1+d\sigma)\leq d\sigma.
 \end{align*}
 The case $m=0$ follows by continuity. Substitution of
 \eqref{eq:SmoothedEntropyConstants} into
 \Cref{lem:AppendixSmoothnessBoundedGradientsPenaltyLambda} proves
 \mbox{\eqref{eq:SmoothedPenaltyGradientConstant}-
  \eqref{eq:SmoothedPenaltyL1Constant}}.

 Strong duality for the original problem gives, for every occupancy
 $\lambda=\lambda^{\pi_\theta}$ with $\theta\in\Theta$, the inequality
 \begin{align*}
  -\HH(\lambda)+\nuOpt\Reg(\lambda)
   & \geq d(\nuOpt)=-\HH(\lambda^*).
 \end{align*}
 Since $\Reg(\lambda)\leq r(\lambda)$, this implies the pointwise
 bound $\HH(\lambda)-\HH(\lambda^*)\leq\nuOpt r(\lambda)$.
 Writing $q_\sigma=\HH-\HH_\sigma\in[0,a_\sigma]$ therefore yields
 \begin{align*}
  \HH_\sigma(\lambda)-\HH_\sigma(\lambda^*)
  \leq\nuOpt r(\lambda)+a_\sigma,
 \end{align*}
 which is \eqref{eq:SmoothedPenaltyDualLowerBound}. Taking
 expectations and applying Jensen's inequality yields
 $\beta s^2-\nuOpt s\leq\delta+a_\sigma$; solving this quadratic proves
 \eqref{eq:SmoothedConstraintTranslation}. Finally, using
 $q_\sigma\in[0,a_\sigma]$ and the nonnegativity of the penalty gives the
 first inequality in \eqref{eq:SmoothedObjectiveTranslation}; the second
 is the original dual bound, concluding the proof.
\end{proof}

\begin{corollary}[{Sample complexity with $\HH_\sigma$}]
 \label{cor:SmoothedEntropyFinalRate}
 {Let $\epsilon\in(0,1]$, $d\define\card{\SS}\card{\AA}$, and set}
 \begin{align*}
  {\sigma\define\frac{\epsilon}{4d},\qquad
   \delta\define\frac{\epsilon}{4},\qquad
   \beta\define(\nuOpt+1)
   \frac{\nuOpt+\sqrt{\nuOpt^2+2}}{\epsilon}.}
 \end{align*}
 {Under Assumptions \ref{ass:Softmax},
 \ref{ass:ParametrizationSOAM}, \ref{ass:Lsmoothness}, and the original
 strong-duality \Cref{ass:StrongDuality}, run \Cref{algo:Penalty} with the
 surrogate $\HH_\sigma$ and target penalty accuracy $\delta$. Then}
 \begin{align*}
  {-\epsilon\leq
   \Exp{\HH(\lambda^{\pi_{\xxN}})-\HH(\lambda^*)}\leq\epsilon,
   \qquad
   \Exp{\plus{\Reg(\lambda^{\pi_{\xxN}})}}\leq\epsilon.}
 \end{align*}
 {For fixed discount factor, constraint-regularity constants, and
 policy-parametrization constants, this requires
 $N=\OOTilde{(d+1)\epsilon^{-2}}$ iterations,
 $B=\OOTilde{(d+1)^2\epsilon^{-4}}$ trajectories per iteration, and
 $H=\OO{\log(d/\epsilon)}$. Hence, for fixed $d$, the total number of
 trajectory rollouts remains $\OOTilde{\epsilon^{-6}}$. Tracking dimension
 explicitly gives the conservative dependence $\OOTilde{(d+1)^3\epsilon^{-6}}$.}
\end{corollary}
\begin{proof}
 {The choice of $\sigma$ gives $a_\sigma\leq d\sigma=\epsilon/4$,
  and \Cref{thm:MainResultPenalty} gives the smoothed penalty gap with
  $\delta=\epsilon/4$. Let $q=\epsilon/(\nuOpt+1)$. Since
  $\sqrt{\nuOpt^2+2}\geq(\nuOpt+1)/2$ for every $\nuOpt\geq0$, the
  stated choice of $\beta$ satisfies}
 \begin{align*}
  \beta q^2-\nuOpt q
  =\epsilon\frac{\sqrt{\nuOpt^2+2}}{\nuOpt+1}
  \geq\frac{\epsilon}{2}\geq\delta+a_\sigma.
 \end{align*}
 {Therefore \eqref{eq:SmoothedConstraintTranslation} gives
 $s\leq q\leq\epsilon$, while
 \eqref{eq:SmoothedObjectiveTranslation} gives the claimed two-sided
 entropy bound.}
 {For the convergence rate, \Cref{lem:SmoothedEntropyProperties} gives}
 \begin{align*}
  {\ell_{\HH_\sigma,\lambda}}
  {\leq1+\log\frac{8d}{\epsilon},}
  \qquad {L_{\HH_\sigma,\lambda,\infty}
  =L_{\HH_\sigma,\lambda}}
  {=\frac{8d}{\epsilon}.}
 \end{align*}
 {Since $\beta=\OO{\epsilon^{-1}}$, the separated penalty constants
 satisfy
 $\ell_{P_\sigma,\lambda}=\OOTilde{\epsilon^{-1}}$ and
 $L_{P_\sigma,\lambda},L_{P_\sigma,\lambda,\infty}
  =\OO{(d+1)\epsilon^{-1}}$ up to fixed constraint-regularity constants.
 Consequently, $L_{P_\sigma,\theta}=\OO{(d+1)\epsilon^{-1}}$,
 $D_{P,\gApprox}=\OO{(d+1)\epsilon^{-1}}$, and
 $\sigma_{P,H}^2=\OOTilde{(d+1)^2\epsilon^{-2}}$.
 Substitution into \Cref{thm:MainResultPenalty} yields the stated values of
 $N$, $B$, and $H$, and multiplying $N$ and $B$ proves the rollout bound.}
\end{proof}

\section{Unified Penalty Framework for Deterministic Smooth and
  Non-Smooth Constrained Optimization under Hidden Convexity}
\label{sec:AppendixNonSmoothAndSmoothPenaltyApproach}

This section places the proposed Policy Gradient Penalty (PGP) method into
a broader optimization perspective. We show that the constrained
maximum-entropy exploration problem studied in the main paper is a special
instance of a wider class of constrained optimization problems exhibiting
\emph{hidden convexity}, and that our analysis extends naturally to both
smooth and non-smooth settings.

In particular, we connect our approach to penalty and augmented Lagrangian
methods (ALM), proximal-point methods, and recent developments in
(non-)convex optimization under hidden convexity. While these connections
are not required to understand or apply the RL algorithm proposed in the
main paper, they clarify the generality of our framework and highlight how
problem-specific structure enables guarantees that are unavailable to
generic penalty or ALM schemes.

\paragraph{Proximal-point Methods and Penalty Methods.}
Proximal-point methods (PPM), originally introduced by
\citet{moreauProximiteDualiteDans1965} and
\citet{martinetBreveCommunicationRegularisation1970}, were later
systematically developed and popularized by
\citet{rockafellarConvexAnalysis1970,
 rockafellarMonotoneOperatorsProximal1976,
 rockafellarVariationalAnalysis1998}, with
\citet{boydProximalAlgorithms2014} providing a recent overview. In the
context of non-smooth and non-convex optimization, PPM appears either
implicitly in convergence analyses or explicitly as an algorithmic design
principle for unconstrained problems \cite{
 davisSubgradientMethodsSharp2018, davisProximallyGuidedStochastic2019,
 davisStochasticModelBasedMinimization2019,
 davisStochasticSubgradientMethod2020, davisNonsmoothLandscapePhase2020,
 fatkhullinStochasticOptimizationHidden2025}, and more recently in
primal-only methods for solving non-smooth, non-convex (deterministic)
constrained optimization problems \cite{ chenPenaltyMethodsClass2016,
 maQuadraticallyRegularizedSubgradient2020,
 huangOracleComplexitySingleLoop2023, jiaFirstOrderMethodsNonsmooth2025,
 fatkhullinGlobalSolutionsNonConvex2025}. Related ideas have also been
explored in multi-agent and game-theoretic settings, e.g.,
\cite{jordanIndependentLearningConstrained2024} for constrained Markov
games.

Penalty and augmented Lagrangian methods are classical tools for
constrained optimization, dating back to early variational approaches
\cite{courantVariationalMethodsSolution1943,
 fiaccoNonlinearProgrammingSequential1990}, and have since been widely
adopted and studied \cite{bertsekasNonlinearProgramming1999,
 yangPartiallyStrictlyMonotone2003, boydConvexOptimization2004,
 nesterovLecturesConvexOptimization2018, xiaoExactPenaltyApproach2025}. A
substantial body of work analyzes first-order penalty and ALM schemes with
explicit complexity guarantees in convex optimization
\cite{aybatFirstOrderAugmentedLagrangian2012,
 lanIterationcomplexityFirstorderPenalty2013,
 patrascuAdaptiveInexactFast2017,
 lanIterationcomplexityFirstorderAugmented2016,
 necoaraComplexityFirstorderInexact2019, liuNonergodicConvergenceRate2019,
 xuIterationComplexityInexact2021, zhuOptimalLowerUpper2023}, as well as
more recently in non-convex settings \cite{liAugmentedLagrangianBased2021,
 liRateimprovedInexactAugmented2021}.

\paragraph{ALM-PPM Hybrids and Hidden Convexity.}
The combination of augmented Lagrangian schemes with proximal-point ideas,
originally introduced by
\citet{auslenderPenaltyproximalMethodsConvex1987a}, has proven particularly
effective. Recent work establishes strong complexity guarantees for smooth
non-convex optimization \cite{kongComplexityQuadraticPenalty2019,
 zhangProximalAlternatingDirection2020, xieComplexityProximalAugmented2021,
 linComplexityInexactProximalpoint2022}, as well as for non-smooth but
convex problems \cite{dhingraProximalAugmentedLagrangian2019,
 tran-dinhProximalAlternatingPenalty2019}.

As already mentioned in the main part of this work,
\eqref{eq:MainProblemRL} is a special instance of a broader class of
(deterministic) optimization problems, called \emph{hidden convex}
problems, of the form
\begin{align}
 \min_{\xxApp \in \XXApp} \;F_1(\xxApp) \define H_1(c(\xxApp)), \quad
 \text{s.t. } F_2(\xxApp) \define H_2(c(\xxApp)) \leq 0,
 \tag{HC}
 \label{eq:MainProblem}
\end{align}
where we use the function $c(\innerEmpty{})$ to indicate a more general
transformation, beyond the state-action-occupancy measure we
studied in the main part of this work. The following central assumption
on $c(\innerEmpty{})$ is a stronger version of
\Cref{ass:ParametrizationSOAM}, but the analysis can also be carried
out under \Cref{ass:ParametrizationSOAM}, cf.
\citet{fatkhullinGlobalSolutionsNonConvex2025} after Definition 1.

\begin{definition}[Hidden Convexity]
 \label{def:HC}
 The above problem \eqref{eq:MainProblem} is called \emph{hidden convex}
 with modulus $\mu_c > 0$, if its components satisfy the following underlying conditions.
 \begin{enumerate}
  \item The domain $\UUApp = c(\XXApp)$ is convex, the functions $H_1, H_2 : \UUApp
         \to \RR$ are convex, i.e. satisfy for $i = 1, 2$ and for all $u, v \in
         \UUApp$ and any $\alpha \in [0,1]$
        \begin{align}
         H_{i}((1-\alpha)\uu + \alpha \vv) \leq (1-\alpha)H_i(\uu)
         + \alpha H_i(\vv)
         - \frac{(1-\alpha)\alpha \mu_H}{2} \norm{\uu - \vv}^2.
         \tag{HC-1}
         \label{eq:HCHCDefinitionConvexityInequality}
        \end{align}
        Additionally, we assume \eqref{eq:MainProblem}
        admits a solution $\uu^* \in \UUApp$ with its corresponding objective function
        value $\Fopt \eqdef H_1(\uu^*) = F_1(c^{-1}(\uu^*))$. Throughout this appendix,
        $\Fopt$ denotes the minimum which corresponds to $-\Fopt$ from
        \eqref{eq:MainProblemRL}.
  \item The map $c : \XXApp \to \UUApp$ is invertible and there exists a $\mu_c >
         0$ such that for all $\xxApp, \yy \in \XXApp$ it holds
        \begin{align}
         \norm{c(\xxApp) - c(\yy)} \geq \mu_c \norm{\xxApp - \yy} \tag{HC-2}.
         \label{eq:HCDefinitionNormCfunctionInequality}
        \end{align}
 \end{enumerate}
 If $\mu_H>0$, we refer to \eqref{eq:MainProblem} as $(\mu_c, \mu_H)$-strongly
 convex. Note that in this more generic setting, we have
 $\mu_c = \nicefrac{1}{\mu_\lambda}$
 of \Cref{ass:ParametrizationSOAM}. We can map \eqref{eq:MainProblemRL} to
 \eqref{eq:MainProblem} via $H_1 = -\HH$ and $H_2 = \Reg$.
\end{definition}
Note that the condition \eqref{eq:HCDefinitionNormCfunctionInequality}
along with \Cref{ass:WeaklyConvex} imply that the domain $\XXApp$ has
bounded diameter $\DDXApp \leq \frac{1}{\mu_c} \DDUApp $. For necessary and
sufficient conditions for \eqref{eq:HCDefinitionNormCfunctionInequality},
we refer the reader to  \citet{fatkhullinStochasticOptimizationHidden2025}.

We start this section with two key inequalities for the analysis of
gradient-based methods under hidden convexity, then we analyze biased SGD
under hidden convexity and afterwards we generalize our penalty method
efforts of the main part to non-smooth problems.

\subsection{Key Inequalities for Analysis under Hidden Convexity}
We state the following proposition which is key for deriving global
convergence guarantees under hidden convexity.
\begin{proposition}[\citet{fatkhullinStochasticOptimizationHidden2025} Prop. 3]
 \label{lemma:HCContractionInequality}
 Let $F_i(\innerEmpty{})$, $i=1,2$, be hidden convex with $\mu_c > 0$. For any $\alpha \in [0,1]$
 and $\xxApp, \yy \in \XXApp$, define $\xxApp_\alpha \define c^{-1}((1-\alpha) c(\xxApp) + \alpha c(\yy))$,
 then, for $i =1, 2$, the following functional inequality
 \begin{align}
  F_i(\xxApp_\alpha) \leq (1-\alpha)F_i(\xxApp) + \alpha F_i(\yy)
  - \frac{(1-\alpha)\alpha \mu_H}{2} \norm{c(\xxApp) - c(\yy)}^2,
  \tag{HC-FI}
  \label{eq:HCContractionInequalityObjective}
 \end{align}
 and norm inequality
 \begin{align}
  \norm{\xxApp_\alpha - \xxApp} \leq \frac{\alpha}{\mu_c} \norm{c(\xxApp) - c(\yy)}.
  \tag{HC-NI}
  \label{eq:HCContractionInequalityNorm}
 \end{align}
 hold.
\end{proposition}
\begin{proof}
 The proof is originally presented in \citet{fatkhullinStochasticOptimizationHidden2025} but
 we recall it here for educational purposes, since it illustrates the core idea of convergence
 analyses under hidden convexity.

 By (strong) convexity of $H_i(\innerEmpty{})$ and the convexity of the
 domain $\UUApp$, we have
 \begin{align*}
  F_i(\xxApp_\alpha) & = F_i(c^{-1}((1-\alpha)c(\xxApp) + \alpha c(\yy)))  \\
                     & = H_i((1-\alpha)c(\xxApp) + \alpha c(\yy))          \\
                     & \leq (1-\alpha) H_i(c(\xxApp)) + \alpha H_i(c(\yy))
  - \frac{(1-\alpha) \alpha \mu_H}{2}\norm{c(\xxApp) - c(\yy)}^2           \\
                     & = (1-\alpha) F_i(\xxApp) + \alpha F_i(\yy)
  - \frac{(1-\alpha) \alpha \mu_H}{2}\norm{c(\xxApp) - c(\yy)}^2,
 \end{align*}
 using that $(1-\alpha)c(\xxApp) + \alpha c(\yy) \in \UUApp$ for any $\xxApp \in \XXApp$,
 with $i=1,2$.

 Using the property of $c$ we obtain
 \begin{align*}
  \norm{\xxApp_\alpha - \xxApp}
  = \norm{c^{-1}((1-\alpha) c(\xxApp) + \alpha c(\yy)) - c^{-1}(c(\xxApp))}
  \leq \frac{1}{\mu_c} \norm{\alpha (c(\xxApp)- c(\yy))},
 \end{align*}
 which concludes the proof.
\end{proof}
\subsection{Biased Stochastic Gradient Descent under Hidden Convexity}
To solve our penalty problem \eqref{eq:InexactPenaltyContractionResults},
we use (projected) Stochastic Gradient Descent (SGD) with a biased gradient
estimate. We analyze biased SGD in our more general, hidden convex setting.

Consider the following general update of projected stochastic gradient
descent:
\begin{align}
 \xxAppKnext \define \proj_{\XXApp}\left(\xxAppK - \eta \gK\right),
 \quad \gKnext \define \gHatKnext(\xxAppK, \xiK),
 \tag{SGD}
 \label{eq:SGDwithMomemtumUpdate}
\end{align}
where the stochastic gradient oracle satisfies
\begin{align}
 \Expu{\xi}{\gHatK(\xxApp, \xi)}                                             &
 = \nabla_{\xxApp} F(\xxApp) + b(\xxApp)
 \label{eq:M-SGD-Bias}
 \tag{SGD-Bias}                                                                \\
 \Expu{\xi}{\norm*{\gHatK(\xxApp, \xi) - \Expu{\xi}{\gHatK(\xxApp, \xi)}}^2} &
 \leq \sigma^2
 \label{eq:M-SGD-Var}
 \tag{SGD-Var}
\end{align}
with variance $\sigma^2 \geq 0$ and a uniformly bounded
bias $\norm*{b(\xxApp)}^2 \leq \delta^2$, $\delta \geq 0$, for all $\xxApp \in \XXApp$.

\begin{theorem}[Biased SGD, Hidden Convexity]\label{thm:BiasedSGDHiddenConvex}
 Let $F$ be $L$-smooth and hidden convex with $\Fopt \define F(\xxAppOpt)$
 for an $\xxAppOpt \in \XXApp$, i.e. $F_2 \equiv 0$. Then running (projected) biased
 \eqref{eq:SGDwithMomemtumUpdate}, with a step size
 $\eta \le \frac{1}{L}$, yields for every $K\in \NN$
 and any $\alpha \in [0,1]$ the following inequality
 \begin{align*}
  \Exp{F(\xxAppNnext) - \Fopt} \leq (1-\alpha)^{N}\Delta_0
  + \frac{\DDUAppsq}{\muLipSq \eta} \alpha + \frac{\eta}{2}
  \frac{1}{\alpha} \left(\sigma^2 + \delta^2\right)
 \end{align*}
 with $\Delta_0 \define \Exp{F(\xxAppZero) - \Fopt}$. In particular, for
 an arbitrary target accuracy of $\epsilon>0$, without loss of generality
 $\epsilon \leq \frac{3\DDUAppsq L}{\muLipSq}$,
 by picking
 \begin{align*}
  \alpha \leq  \frac{\muLipSq \eta}{\DDUAppsq} \frac{\epsilon}{3} \le
  \frac{\muLipSq}{\DDUAppsq L} \frac{\epsilon}{3}
 \end{align*}
 after
 \begin{align*}
  N \geq \frac{1}{\alpha} \log\left(\frac{3\Delta_0}{\epsilon}\right)
  \geq
  \frac{3\DDUAppsq \cdot L}{\muLipSq} \frac{1}{\epsilon}
  \log\left(\frac{3\Delta_0}{\epsilon}\right)
 \end{align*}
 steps, the last iterate satisfies
 \begin{align*}
  \Exp{F(\xxAppNnext) - \Fopt} \leq \frac{2\epsilon}{3} + \frac{\eta}{2}
  \frac{1}{\alpha} \left(\sigma^2 + \delta^2\right).
 \end{align*}
 Consequently, we need the bias and the variance to satisfy
 \begin{align*}
  \sigma^2, \delta^2 %
  \leq
  \frac{\epsilon}{3} \frac{\alpha}{\eta} \leq \frac{\muLipSq}{9\DDUAppsq} \epsilon^2
 \end{align*}
 in order to get $\epsilon$-close to the global minima.
\end{theorem}
\begin{proof}
 We use Thm. 7 in \citet{fatkhullinStochasticOptimizationHidden2025}
 and that the gradient estimation error can be uniformly bounded by
 \begin{align*}
  \Exp{\norm*{\gHatK - \nabla F(\xxApp)}^2} =
  \Exp{\norm*{\gHatK - \Exp{\gHatK}}^2} +
  \Exp{\norm*{\Exp{\gHatK} - \nabla F(\xxApp)}^2}
  \leq \sigma^2 + \delta^2.
 \end{align*}
 Plugging this bound into the proof of Thm. 7 in
 \citet{fatkhullinStochasticOptimizationHidden2025}, unrolling the recursion
 and choosing $\alpha$ as proposed concludes the proof.
\end{proof}

\subsection{Inexact Proximal Point Penalty Method (IPPPM)}

\begin{table*}[t]
 \centering
 \resizebox{\textwidth}{!}{
  \begin{tabular}{lcccc}
   \toprule
   Setting                                                         & Method                        & \multicolumn{2}{c}{\parbox{2cm}{\centering Sub-gradient Complexity}} & \parbox{2cm}{\centering Constraint Qualification}   \\ \midrule
                                                                   &                               & Smooth                                                               & Non-smooth                                        & \\
   \begin{tabular}[c]{@{}l@{}}\end{tabular}                        &                               &                                                                      &                                                   & \\
   \begin{tabular}[c]{@{}l@{}} $F_2(\cdot) \equiv 0$ \end{tabular} & GD / IPPM+SM                  & \parbox{4cm}{\centering $\epsilon^{-1}$                                                                                    \\ \cite{zhangVariationalPolicyGradient2020}}                                                      & \parbox{4cm}{\centering $\epsilon^{-3}$                                                                                       \\ \cite{fatkhullinGlobalSolutionsNonConvex2025}}                                   & -- \\\midrule
   \begin{tabular}[c]{@{}l@{}}\end{tabular}                        &                               &                                                                      &                                                   & \\
   $F_2(\cdot) \not\equiv 0$                                       & \parbox{4cm}{\centering IPPPM                                                                                                                              \\ \Cref{algo:PenaltyPPM}  with $\penGenNot = \pen{\innerEmpty{}}$ }    & \parbox{3cm}{\centering $\epsilon^{-2}$ \\ \Cref{cor:FinalRateSmoothPenaltyIPPM}} & \parbox{3cm}{\centering $\epsilon^{-6}$ \\ \Cref{cor:FinalRateNonSmoothQuadraticPenaltyIPPM}} & Strong Duality       \\
   $F_2(\cdot) \not\equiv 0$                                       & \parbox{4cm}{\centering IPPPM                                                                                                                              \\ \Cref{algo:PenaltyPPM} with $\penGenNot = \plus{\innerEmpty{}}$} &  --  & \parbox{3cm}{\centering $\epsilon^{-3}$ \\ \Cref{cor:FinalRateNonSmoothExactPenaltyIPPM}} & Strong Duality     \\ \bottomrule
  \end{tabular}
 }
 \caption[Summary of Contribution]{Summary of total \textbf{(sub-)gradient}
  and \textbf{function evaluation}
  complexities for deterministic gradient methods under {hidden convexity}.
  The second column shows the number of (sub)-gradient
  (and function in case $F_2(\innerEmpty{}) \not\equiv 0$) evaluations in $\OOTilde{\innerEmpty{}}$ notation to find a
  point $\xxN$ with
  $\abs{F_1(\xxN) - \Fopt} \leq \epsilon$, $\plus{F_2(\xxN)}\leq \epsilon$.
  \enquote{IPPPM} refers to the inexact proximal point penalty framework,
  \enquote{GD} to gradient descent established by
  \citet{zhangVariationalPolicyGradient2020},
  and \enquote{IPPM+SM} to the inexact proximal point framework combined with the
  sub-gradient method analyzed by \citet{fatkhullinGlobalSolutionsNonConvex2025}.
 }
 \label{tab:OverviewRatesTable1}
\end{table*}

\begin{algorithm}[tb]
 \caption{\textsc{Inexact Proximal Point Penalty Method} ($\BothPPPM$)}
 \begin{algorithmic}[1]
  \STATE {\bfseries Input:} objective $F_1$, constraint $F_2$, target tolerance $\epsilon > 0$,
  inner tolerance $\epsin > 0$,
  initial point $\xxApp_0 \in \mathcal{X}$, outer loops $N$, inner loops $\Tinner$,
  inner algorithm $\Oracle$, regularization $\hat\rho > \rho_\phi$,
  penalty parameter $\beta>0$, penalty function $\penGenNot \in \{\pen{\innerEmpty{}},
   \plus{\innerEmpty{}}\}$
  \FOR{$k = 0$ {\bfseries to} $N-1$}
  \STATE Define the regularized penalty function for $\xxApp \in \XXApp$ as
  \begin{align*}
   \phiKpen(\xxApp) \define F_1(\xxApp) + \frac{\rhoHat}{2}\norm{\xxApp
    - \xxAppK}^2 + \beta\penGen{F_2(\xxApp)},
  \end{align*}
  \STATE Compute approximate solution of \eqref{eq:MainPPPMLoop} via
  \begin{align*}
   \xxAppKnext \leftarrow \Oracle(\phiKpen, \xxAppK, \Tinner,
   \epsin)
  \end{align*}
  \ENDFOR
  \STATE {\bfseries Return:} $x^{(N)}$
 \end{algorithmic}
 \label{algo:PenaltyPPM}
\end{algorithm}

Inspired by the success of the penalty method for solving
\eqref{eq:MainProblemRL}, the question arises if penalty methods can be
used to solve the broader class of constrained hidden convex optimization
problems.

\begin{assumption}\label{ass:WeaklyConvex}
 We conduct our analysis under the following (standard) assumptions:
 \begin{itemize}
  \item (Weak Convexity) $F_1, F_2$ to be $\rho$-weakly convex, for an $\rho > 0$.
        It is known that weak convexity is a relatively mild assumption,
        compared to e.g. smoothness
        \cite{davisStochasticModelBasedMinimization2019,
         drusvyatskiyEfficiencyMinimizingCompositions2019}.
        $F_i$ is weakly convex, if e.g. $F_i$ is twice
        continuously differentiable with
        the Hessian satisfying $\nabla^2 F_i(\xxApp) \succcurlyeq -\rho$, for
        all $\xxApp \in \XXApp$, $i=1,2$, cf. Prop 4.12 in
        \citet{vialStrongWeakConvexity1983}.
  \item (Bounded [Sub-]Gradients) The functions
        $F_i$, $i=1,2$, are continuous and satisfy for all $\xxApp \in \XXApp$ that
        $\partial F_i(\xxApp) \neq \emptyset$ and with $\norm{g} \leq G$ we
        have a uniform bound for all sub-gradients $g\in \partial F_i(\xxApp)$,
        $\xxApp \in \XXApp$.
  \item (Bounded Domain) The Domain $\UUApp$ has a bounded diameter $\DDUApp > 0$.
 \end{itemize}
\end{assumption}

In particular in the context of non-convex, non-smooth optimization, the
proximal point method (PPM) framework \cite{boydProximalAlgorithms2014} has
proven to be very successful in the unconstrained setting
\cite{davisSubgradientMethodsSharp2018,
 davisStochasticSubgradientMethod2018,
 davisStochasticModelBasedMinimization2019,
 davisProximallyGuidedStochastic2019, davisStochasticSubgradientMethod2020,
 fatkhullinStochasticOptimizationHidden2025} as well as in the constrained
setting \cite{maQuadraticallyRegularizedSubgradient2020,
 boobStochasticFirstorderMethods2023, huangOracleComplexitySingleLoop2023,
 boobLevelConstrainedFirst2025, jiaFirstOrderMethodsNonsmooth2025,
 fatkhullinGlobalSolutionsNonConvex2025}, either implicitly in the analysis
or explicitly for algorithm design. Motivated by the non-smooth literature,
as well as recent advances in proximal point method penalty
\cite{linComplexityInexactProximalpoint2022} and augmented Lagrangian
methods \cite{tran-dinhProximalAlternatingPenalty2019,
 xieComplexityProximalAugmented2021}, we propose the following,
\textbf{inexact proximal point penalty method}:
\begin{align}
 \begin{aligned}
   & \xxAppK \approx \xxAppHatK \define
  \argmin_{\xxApp \in \XXApp}\phiK{}(\xxApp)
  \define F_1(\xxApp)
  + \frac{\rhoHat}{2} \norm{\xxApp - \xxAppK}^2
  + \beta \penGen{F_2(\xxApp)},
 \end{aligned}
 \tag{IPPPM}
 \label{eq:MainPPPMLoop}
\end{align}
where the penalty function $\penGenNot$ depends on the smoothness of the constraint.
In the case of non-smooth
$F_1, F_2$, the exact penalty function $\plus{\innerEmpty{}}$ is a natural
choice. Under the additional assumption of $L$-smooth $F_1, F_2$ one needs
to \enquote{smoothen} the penalty, in order to propagate the smoothness and
thus a quadratic penalty $\pen{\innerEmpty{}}$ is applied. For the
convergence analysis of IPPM, we use the notation $\penGen{\innerEmpty{}}$
for both penalty functions and perform a unified analysis by using that
both penalty functions are convex and monotone, before we compute the
corresponding sample complexities for $\plus{\innerEmpty{}}$ and
$\pen{\innerEmpty{}}$ individually.

Under \Cref{ass:WeaklyConvex}, if $\rhoHat > \rho$, the inner penalty
problem is strongly convex, see Prop.~2.1 in
\citet{davisProximallyGuidedStochastic2019}.

The \textbf{roadmap} of the PPPM analysis is the following:
\begin{itemize}
 \item \emph{Preservation of Weak Convexity and Bounded (Sub-)Gradients.}
       Under \Assref{ass:WeaklyConvex}, we show in
       \Cref{Appendix-sec:PenaltyApproachAnalysis} that the penalty term
       $\penGenNot \circ F_2$ has uniformly bounded (sub-)gradients and is weakly
       convex. Thus, \eqref{eq:MainPPPMLoop} is $\rho_\phi$-weakly convex
       and can be lifted to be strongly convex. The constant $\rho_\phi$ depends
       on the penalty function and the smoothness assumption.
 \item \emph{Solve inner problem efficiently.} We apply an efficient algorithm
       for the inner (non-)smooth strongly convex unconstrained
       optimization problems to obtain convergence guarantees for the
       optimality gap of the penalty function $\penOpt{\penGenNot, \beta}$.
 \item \emph{Translate guarantees.} Under the assumption of
       \emph{strong duality}, we translate the results
       into guarantees on the optimality gap and constraint violation.
       An overview of all of the resulting rates can be found in \Cref{tab:OverviewRatesTable1}.
\end{itemize}

\begin{assumption}[Unconstrained Solution Oracle]
 \label{ass:PPPMinexactOracle}
 We assume to have access to an inner approximate
 solution oracle
 $\Oracle_{\epsin}$ to solve \eqref{eq:MainPPPMLoop} to
 an $\epsin$-accuracy in expectation for
 any $\epsin > 0$, i.e.
 \begin{align*}
  \phiKpen(\xxAppK) - \phiKpen(\xxAppHatK) \leq \epsin
 \end{align*}
 for all iterations $k \in \setm{N}$
 after $\Tinner = \Tinner(\epsin)$ steps.
\end{assumption}
Exemplarily, for the non-smooth case,
we will formulate \Cref{cor:FinalRateNonSmoothExactPenaltyIPPM} and
\Cref{cor:FinalRateNonSmoothQuadraticPenaltyIPPM} in terms of the classical
(stochastic) sub-gradient descent and in the smooth case
\Cref{cor:FinalRateSmoothPenaltyIPPM} using SGD, respectively.

\subsubsection{Analysis of PPPM, Penalty Function Optimality Value Gap}
We start our analysis with the following key lemma, yielding a guarantee on
the optimal value gap of the penalty function, not requiring any constraint
qualification.

\begin{theorem}[PPPM, Penalty Function]
 \label{thm:MainResultIPPPMHiddenConvexPenaltyOptGap}
 Let \eqref{eq:MainProblem} be hidden convex and \Cref{ass:WeaklyConvex} hold.
 Given a target precision $\epsilon > 0$, a lifting parameter $\rhoHat >
  \rho_\phi$, assume $\epsilon \leq \frac{3\rhoHat\DDLamSq}{2\lamAppLipSq}$
 holds. Then after $N\in\NN$ steps, the last iterate of
 \Cref{algo:PenaltyPPM} satisfies
 \begin{align}
  \begin{aligned}
   F_1(\xxAppN) - \Fopt +
   \beta\penGen{F_2(\xxAppN)} \leq \epsilon
  \end{aligned}
  \tag{OptGap}
  \label{eq:InexactPenaltyContractionResults}
 \end{align}
 after
 \begin{align*}
  N \geq
  \frac{3 \rhoHat \DDLamSq}{2\lamAppLipSq} \frac{1}{\epsilon} \cdot
  \log\paren*{\frac{3\paren{\Delta_0 + \beta \ConstrViolZero}}{\epsilon}}
 \end{align*}
 iterations, where $\Delta_0 \define F_1(\xxAppZero) -
  \Fopt$, $\ConstrViolZero \define
  \penGen{F_2(\xxAppZero)}$, an $\alpha \leq
  \frac{2\lamAppLipSq}{3 \rhoHat \DDLamSq}\epsilon
  \in \brac{0,1}$ used in the analysis and a penalty parameter $\beta > 1$.
 With an inner precision of $\epsin \leq \frac{\alpha}{2} \epsilon$, the
 total sample complexity is given by
 \begin{align*}
   & \Ttotal \geq N \cdot
  \Tinner\paren*{\epsin}                                                 \\
   & = \frac{3 \rhoHat \DDLamSq}{2\lamAppLipSq} \frac{1}{\epsilon} \cdot
  \log\paren*{\frac{3\paren{\Delta_0 + \beta \ConstrViolZero}}{\epsilon}}
  \cdot \Tinner\paren*{\epsin}.
 \end{align*}
\end{theorem}

\begin{proof}
 We define for an iteration $k \in \setN$ the element
 $\xxAppK_\alpha\define c^{-1}\paren{(1-\alpha)c(\xxAppK) + \alpha c(\xxAppOpt)}$
 and we obtain the following inequalities:
 \begingroup
 \allowdisplaybreaks
 \begin{align*}
  F_1(\xxAppKnext)    + \beta \penGen{F_2(\xxAppKnext)}
  \overset{(i)}  & {\leq} \phiKpen(\xxAppKnext)
  \overset{(ii)}   {\leq} \phiKpen(\xxAppHatKnext) + \epsin
  \overset{(iii)}  {\leq} \phiKpen(\xxAppK_\alpha) + \epsin                                    \\
                 & = F_1(\xxAppK_\alpha) + \frac{\rhoHat}{2} \norm{\xxAppK_\alpha - \xxAppK}^2
  + \beta \penGen{F_2(\xxAppK_\alpha)} + \epsin                                                \\
  \overset{(iv)} & {\leq} \paren{1-\alpha} F_1(\xxAppK) + \alpha \Fopt
  + \frac{\rhoHat}{2} \frac{\alpha^2}{\mu_c^2} \norm{c(\xxAppK) - c(\xxAppOpt)}^2              \\
                 & \qquad + \beta \penGen{\paren{1-\alpha}F_2(\xxAppK)
   + \alpha F_2(\xxAppOpt)}
  + \epsin                                                                                     \\
  \overset{(v)}  & {\leq} \paren{1-\alpha} F_1(\xxAppK) + \alpha \Fopt
  + \alpha^2\frac{\rhoHat \DDUAppsq}{2\mu_c^2}
  + \beta \paren{1-\alpha}\penGen{F_2(\xxAppK)}
  + \epsin                                                                                     \\
                 & =\paren{1-\alpha} \paren*{F_1(\xxAppK) +
   \beta\penGen{F_2(\xxAppK)}}
  + \alpha \Fopt
  + \alpha^2\frac{\rhoHat \DDUAppsq}{2\mu_c^2} + \epsin,
 \end{align*}
 \endgroup
 where we used in $(i)$ that the regularization is
 always non-zero, in $(ii)$ the $\epsin$-approximation property of the oracle
 $\Oracle_\epsin$, in $(iii)$ the optimality of $\xxAppHatKnext$,
 in $(iv)$ hidden convexity and monotonicity of $\penGenNot$, and in $(v)$
 that the domain is bounded by our assumption,
 $F_2(\xxAppOpt) \leq 0$ as well as the fact that $\penGen{\innerEmpty{}}$ is convex.

 Subtracting $\Fopt$ from both sides yields
 \begin{align}
  \begin{aligned}
    & F_1(\xxAppKnext) - \Fopt + \beta \penGen{F_2(\xxAppKnext)} \\
    & \qquad \leq \paren{1-\alpha} \cdot
   F_1(\xxAppK) - \Fopt + \beta\penGen{F_2(\xxAppK)}
   + \alpha^2\frac{\rhoHat \DDUAppsq}{2\mu_c^2} + \epsin.
  \end{aligned}
  \tag{IPPPM-Rec}
  \label{eq:RecursionEquationIPPPMHiddenConvex}
 \end{align}
 We unroll for $k=0, \ldots, N$
 and obtain after $N\in \NN$ iterations the following inequalities
 for the last iterate
 \begin{align*}
   & F_1(\xxAppN) - \Fopt + \beta \penGen{F_2(\xxAppN)} \\
   & \qquad\leq \paren{1-\alpha}^N \paren*{
   \Delta_0 + \beta\ConstrViolZero
  } + \alpha \frac{\rhoHat \DDUAppsq}{2\mu_c^2} + \frac{\epsin}{\alpha},
 \end{align*}
 with $\Delta_0 \define F_1(\xxAppZero) - \Fopt$,
 and $\ConstrViolZero \define \penGen{F_2(\xxAppZero)}$,
 since the partial sum of the geometric series is bounded by $\nfr{1}{\alpha}$.
 The claim follows by the choice of $\alpha, \epsin$ and $N$.
\end{proof}

\begin{remark}[Role of the Penalty Parameter]
 It is important to emphasize that, although the parameter $\beta$ appears only
 logarithmically in the number of outer iterations $N$, it plays a
 crucial role in shaping the properties of the penalty function.
 In particular, $\beta$ affects the norm of the (sub-)gradients,
 the associated smoothness constant, and consequently the weak convexity parameter.
 As a result, it directly influences the computational complexity of the inner
 solver employed for the subproblem.
\end{remark}

\begin{theorem}[PPPM, Exact Penalty]
 \label{thm:MainResultPPPMExactPenaltyConvergenceGuarantee}
 We assume the assumptions of
 \Cref{thm:MainResultIPPPMHiddenConvexPenaltyOptGap} and
 strong duality hold for \eqref{eq:MainProblem} with a Lagrangian
 multiplier $\nuOpt \geq 0$. Then, running \Cref{algo:PenaltyPPM}
 with the exact penalty $\penGen{\innerEmpty{}} = \plus{\innerEmpty{}}$
 and any $\beta \geq 2\cdot\nuOpt + 2$ yields following last iterate
 guarantee on the optimality gap and constraint violation
 \begin{align}
  \abs*{F_1(\xxAppN) - \Fopt} \leq \epsilon \quad
  \text{and } \; \plus{F_2(\xxAppN)} \leq \epsilon
  \tag{IPPPM-Ex}
 \end{align}
 after
 \begin{align*}
  N \geq
  \frac{3 \rhoHat \DDUAppsq}{2\mu_c^2} \frac{\beta}{\epsilon} \cdot
  \log\paren*{\frac{3\beta\brac{\Delta_0 + \beta \ConstrViolZero}}{\epsilon}}
 \end{align*}
 iterations, where $\Delta_0 \define F_1(\xxAppZero) - \Fopt$,
 and $\ConstrViolZero \define \plus{F_2(\xxAppZero)}$.
 With an inner precision $\epsin \leq \frac{\alpha}{2\beta} \epsilon$,
 where $\alpha \leq \frac{2\mu_c^2}{3 \rhoHat \DDUAppsq} \epsilon \in \brac{0,1}$
 as in \Cref{thm:MainResultIPPPMHiddenConvexPenaltyOptGap},
 the total rate is given by
 \begin{align*}
  \Ttotal
   & \geq N \cdot
  \Tinner\paren*{\epsin}
  = \frac{3 \rhoHat \DDUAppsq}{2\mu_c^2} \frac{\beta}{\epsilon} \cdot
  \log\paren*{\frac{3\beta\brac{\Delta_0 + \ConstrViolZero}}{\epsilon}}
  \cdot \Tinner\paren*{\epsin} \\
   & \geq
  \frac{3 \rhoHat \DDUAppsq}{2\mu_c^2} \frac{\beta}{\epsilon} \cdot
  \log\paren*{\frac{3\beta\brac{\Delta_0 + \ConstrViolZero}}{\epsilon}}
  \cdot \Tinner\paren*{\frac{\alpha}{2\beta} \epsilon}
 \end{align*}
\end{theorem}
\begin{proof}
 We define $\epsTilde \define \frac{\epsilon}{\beta} < \epsilon$.
 By invoking \Cref{thm:MainResultIPPPMHiddenConvexPenaltyOptGap}
 with $\epsTilde$, we obtain $\penOpt{1, \beta}(\xxAppN) \leq \epsTilde$ after
 \begin{align*}
  N \geq
  \frac{3 \rhoHat \DDUAppsq}{2\mu_c^2} \frac{\beta}{\epsilon} \cdot
  \log\paren*{\frac{3\beta\brac{\Delta_0 + \beta \ConstrViolZero}}{\epsilon}}.
 \end{align*}
 Thus, the inequality
 \begin{align*}
  F_1(\xxAppN) - \Fopt \leq \penOpt{1, \beta}(\xxAppN) \leq \epsTilde \leq \epsilon
 \end{align*}
 follows trivially and applying
 \Cref{cor:AppendixConstraintViolationExactPenaltyTranslation} yields
 \begin{align*}
  \plus{F_2(\xxAppN)} \leq \frac{\penOpt{1, \beta}(\xxAppN)}{\beta - \nuOpt}
  \leq \frac{\epsTilde}{\beta - \nuOpt} \leq \epsTilde \leq \epsilon.
 \end{align*}
 To obtain the guarantee on the function value gap in absolute value, we derive
 \begin{align*}
   & - \beta \plus{F_2(\xxAppN)} \leq F_1(\xxAppN) - \Fopt                           \\
  \Rightarrow
   & -\left(F_1(\xxAppN) - \Fopt\right) \leq \beta \plus{F_2(\xxAppN)} \leq \epsilon
 \end{align*}
 which concludes the proof.
\end{proof}

\begin{theorem}[PPPM, Quadratic Penalty]
 \label{thm:MainResultPPPMQuadraticPenaltyConvergenceGuarantee}
 We assume the assumptions of
 \Cref{thm:MainResultIPPPMHiddenConvexPenaltyOptGap} and
 strong duality hold for \eqref{eq:MainProblem} with a Lagrangian
 multiplier $\nuOpt \geq 0$. Then, running \Cref{algo:PenaltyPPM}
 with the quadratic penalty $\penGen{\innerEmpty{}} = \pen{\innerEmpty{}}$
 and $\beta = \beta(\epsilon) \geq \frac{(\nuOpt + 1)(\nuOpt + \sqrt{\nuOpt^2 + 2})}{\epsilon}$
 yields following guarantee on the optimality gap and constraint violation
 \begin{align}
  \abs{F_1(\xxAppN) - \Fopt} \leq \epsilon \quad \text{and }
  \; \plus{F_2(\xxAppN)} \leq \epsilon
  \tag{IPPPM-Qu}
  \label{eq:HCHCInexactPenaltyPPMcontractionResult}
 \end{align}
 after
 \begin{align*}
  N \geq
  \frac{3 \rhoHat \DDUAppsq}{2\mu_c^2} \beta(\epsilon) \cdot
  \log\paren*{3\beta(\epsilon)\brac{\Delta_0 + \beta(\epsilon) \ConstrViolZero}}
 \end{align*}
 iterations, where $\Delta_0 \define F_1(\xxAppZero) - \Fopt$,
 and $\ConstrViolZero \define \pen{F_2(\xxAppZero)}$.
 With an inner precision $\epsin \leq \frac{\alpha}{2\beta(\epsilon)}$,
 where $\alpha \leq \frac{2\mu_c^2}{3 \rhoHat \DDUAppsq} \epsilon \in \brac{0,1}$
 as in \Cref{thm:MainResultIPPPMHiddenConvexPenaltyOptGap},
 the total rate is in general given by
 \begin{align*}
   & \Ttotal \geq N \cdot
  \Tinner\paren*{\epsin}                                          \\
   & = \frac{3 \rhoHat \DDUAppsq}{2\mu_c^2} \beta(\epsilon) \cdot
  \log\paren*{3\beta(\epsilon)\brac{\Delta_0 + \beta(\epsilon) \ConstrViolZero}}
  \cdot \Tinner\paren*{\epsin}.
 \end{align*}
\end{theorem}
\begin{proof}
 We define
 $\epsTilde \define \frac{1}{\beta(\epsilon)} \leq \epsilon$.
 By invoking
 \Cref{thm:MainResultIPPPMHiddenConvexPenaltyOptGap} with
 $\epsTilde$ we obtain $\penOpt{2, \beta}(\xxAppN) \leq \epsTilde \leq \epsilon$,
 after
 \begin{align*}
  N \geq \frac{3 \rhoHat \DDUAppsq}{2\mu_c^2} \beta(\epsilon) \cdot
  \log\paren*{3\beta(\epsilon)\brac{\Delta_0
    + \frac{\beta(\epsilon)}{2} \ConstrViolZero}}
 \end{align*}
 iterations. Trivially, $F_1(\xxAppN) - \Fopt \leq \epsilon$ as before.
 Next, we apply
 \Cref{cor:AppendixConstraintViolationQuadraticPenaltyTranslation} to obtain
 \begin{align*}
  \plus{F_2(\xxAppN)} \leq \frac{\nuOpt + \sqrt{\nuOpt^2
    + 2\beta(\epsilon) \epsTilde}}{\beta(\epsilon)}
  = \frac{\nuOpt + \sqrt{\nuOpt^2 + 2}}{\beta(\epsilon)}
  = \frac{\epsilon}{\nuOpt + 1} \leq \epsilon,
 \end{align*}
 by the choice of $\epsTilde$ and $\beta(\epsilon)$.
 To obtain a guarantee on the absolute value
 of the function optimality value gap, we apply
 \Cref{lemma:AppendixTranslatePenaltyGuaranteeToFuncGapConstrViol}
 to obtain
 \begin{align*}
  -\left(F_1(\xxAppN) - \Fopt\right) \leq \nuOpt \plus{F_2(\xxAppN)}
  \leq \nuOpt \frac{\epsilon}{\nuOpt + 1} \leq \epsilon,
 \end{align*}
 concluding the claim.
\end{proof}

\begin{remark}[Dependency of $\beta$ on $\nuOpt$]
 To make the comparison between the constants in
 \Cref{cor:AppendixConstraintViolationExactPenaltyTranslation} and
 \Cref{cor:AppendixConstraintViolationQuadraticPenaltyTranslation} easier, we note that
 \begin{align*}
  2\nuOpt + 2 \geq \nuOpt + \sqrt{\nuOpt + 2},
 \end{align*}
 simplifying the term in $\beta(\epsilon)$ in
 \Cref{cor:AppendixConstraintViolationQuadraticPenaltyTranslation}
 to
 \begin{align*}
  (2\nuOpt + 2)(\nuOpt + 1) = 2(\nuOpt + 1)^2
  > (\nuOpt + 1) \left(\nuOpt + \sqrt{\nuOpt^2 + 2}\right),
 \end{align*}
 corresponding to the term in
 \Cref{cor:AppendixConstraintViolationExactPenaltyTranslation} with an additional
 factor of $\nuOpt + 1$.
\end{remark}

Under the assumption of strong duality, we translate the results of to
convergence guarantees in terms of the function optimality gap and the
constraint violation.

\subsubsection{Sample Complexity Analysis}

\begin{corollary}[IPPPM+SM, Non-Smooth, Exact Penalty]
 \label{cor:FinalRateNonSmoothExactPenaltyIPPM}
 Under the assumptions of
 \Cref{thm:MainResultPPPMExactPenaltyConvergenceGuarantee},
 running \Cref{algo:PenaltyPPM} with the \textbf{exact penalty}
 $\penGen{\innerEmpty} \define \plus{\innerEmpty}$,
 $\rhoHat \define 2 \rho_\phi$ and the
 \textbf{sub-gradient method (SM)} as an inner solver
 for
 \begin{align*}
  N  \geq \frac{3 \beta(\nuOpt) \rho \DDUAppsq}{\mu_c^2} \frac{\beta(\nuOpt)}{\epsilon}
  \cdot \log \paren*{\frac{3 \beta(\nuOpt)
    \brac{\Delta_0 + \beta(\nuOpt)\ConstrViolZero}}{\epsilon}}
 \end{align*}
 outer iterations yields the guarantee
 \begin{align*}
  \abs*{F_1(\xxAppN) - \Fopt} \leq \epsilon\quad \text{and} \quad
  \plus{F_2(\xxAppN)} \leq \epsilon
 \end{align*}
 as a special case of \eqref{eq:HCHCInexactPenaltyPPMcontractionResult}.
 SM has a complexity of
 \begin{align*}
  \TinnerUp{SM}(\epsin) \geq \frac{12 G_\phi^2 \beta(\nuOpt) \DDUAppsq}{\mu_c^2} \epsilon^{-2}
  = \OO{\epsilon^{-2}}
 \end{align*}
 to achieve the desired inner accuracy, where
 $\Delta_0 \define F_1(\xxAppZero) - \Fopt$, $\ConstrViolZero \define \plus{F_2(\xxAppZero)}$,
 $\beta(\nuOpt) \define 2\nuOpt + 2$,
 $\rho_\phi(\beta) \define \beta(\nuOpt) \cdot \rho$,
 and $G_\phi(\beta) \define G + \beta(\nuOpt) \rho \DDXApp + \beta(\nuOpt) G$.

 The total number of sub-gradient and function evaluations is Consequently
 given by:
 \begin{align*}
  \TtotalUp{\nonSmoothLabel} = N \cdot \TinnerUp{SM}(\epsin) = \OOTilde{\epsilon^{-3}}.
 \end{align*}
\end{corollary}
\begin{proof}
 We invoke \Cref{lemma:ExactPenaltyPropertiesNonSmoothCaseHelperPPPM}
 and obtain that each problem \eqref{eq:MainPPPMLoop} is
 $\rho_\phi$-weakly convex, with $\rho_\phi = \beta \rho$ and has
 uniformly bounded gradients by $G_\phi \define G + \rho_\phi \DDXApp + \beta G$.
 Thus, we can lift with $\rhoHat(\beta) \define 2 \rho_\phi = \beta(\nuOpt) \rho$ and obtain
 $(\beta(\nuOpt)\cdot\rho)$-strongly convex inner problems.
 Plugging in the constants into \Cref{thm:MainResultPPPMExactPenaltyConvergenceGuarantee}
 yields
 \begin{align*}
  N & \geq \frac{3 \rhoHat(\beta) \cdot \DDUAppsq}{2\mu_c^2} \frac{\beta(\nuOpt)}{\epsilon} \cdot
  \log\paren*{\frac{3\beta(\nuOpt) \brac{\Delta_0
  + \beta(\nuOpt) \ConstrViolZero}}{\epsilon}}                                                    \\
    & = \frac{3 \beta(\nuOpt) \rho \DDUAppsq}{\mu_c^2} \frac{\beta(\nuOpt)}{\epsilon}
  \cdot \log \paren*{\frac{3 \beta(\nuOpt)
    \brac{\Delta_0 + \beta(\nuOpt) \ConstrViolZero}}{\epsilon}}
 \end{align*}
 for the number of outer iterations. For the inner solver, we use the sub-gradient method (SM)
 for simplicity with a desired accuracy of
 \begin{align*}
  \epsin & \leq \frac{\alpha}{2\beta(\nuOpt)} \epsilon
  \leq \frac{2\mu_c^2}{3\rhoHat(\beta) \cdot \DDUAppsq}\epsilon
  \cdot \frac{1}{2\beta(\nuOpt)} \epsilon
  = \frac{\mu_c^2}{3 \beta(\nuOpt)^2 \rho \DDUAppsq} \epsilon^2.
 \end{align*}
 Invoking Corollary 3.3 in \cite{lanFirstorderStochasticOptimization2020} yields
 \begin{align*}
  \TinnerUp{SM}(\epsin) & \geq \frac{4 G_\phi(\beta)^2}{\beta(\nuOpt)\rho \cdot \epsin}
  = \frac{4 G_\phi(\beta)^2}{\beta(\nuOpt)\rho}
  \cdot \frac{3 \beta(\nuOpt)^2 \rho \DDUAppsq}{\mu_c^2}\epsilon^{-2}                        \\
                        & = \frac{12 \beta(\nuOpt) G_\phi(\beta)^2 \cdot \DDUAppsq}{\mu_c^2}
  \epsilon^{-2}.
 \end{align*}
 Multiplying both rates concludes the claim.
\end{proof}

\begin{corollary}[IPPPM+GD, Smooth, Quadratic Penalty]
 \label{cor:FinalRateSmoothPenaltyIPPM}
 Under the assumptions of \Cref{thm:MainResultIPPPMHiddenConvexPenaltyOptGap}, and assuming
 that $F_1, F_2$ are $L$-smooth, running \Cref{algo:PenaltyPPM}
 with the \textbf{quadratic penalty}
 $\penGen{\innerEmpty} \define \pen{\innerEmpty}$,
 $\rhoHat(\epsilon) \define 2 \rho_\phi(\epsilon)$,
 and $\beta = \beta(\epsilon) \define \frac{2(\nuOpt + 1)^2}{\epsilon}$
 yields the guarantee
 \begin{align*}
  \abs*{F_1(\xxAppN) - \Fopt} \leq \epsilon\quad \text{and} \quad
  \plus{F_2(\xxAppN)} \leq \epsilon.
 \end{align*}
 The total number of gradient evaluations when using the \textbf{gradient
  descent} (GD) method as the inner solver, is given by:
 \begin{align*}
  \TtotalUp{\smoothLabel}
   & =\frac{\DDUAppsq}{\mu_c^2} \beta(\epsilon)
  \cdot \log \paren*{3 \beta(\epsilon) \brac{\Delta_0
  + \beta(\epsilon) \ConstrViolZero}} \cdot                     \\
   & \qquad\cdot L_\phi(\epsilon)
  \cdot \log \paren*{\frac{L_\phi(\epsilon) \DDXAppsq}{\epsin}} \\
   & = \OOTilde{\epsilon^{-2}},
 \end{align*}
 where
 $\Delta_0 \define F_1(\xxAppZero) - \Fopt$,
 $\ConstrViolZero \define \pen{F_2(\xxAppZero)}$,
 $L_\phi(\epsilon)= L + \rhoHat
  + \beta(\epsilon)\cdot G\paren{\DDXApp L + G} = \OO{\epsilon^{-1}}$
 and $\rho_\phi(\epsilon)
  = \beta(\epsilon) G\DDXApp\cdot \min\cbrac{\rho, L} = \OO{\epsilon^{-1}}$.
\end{corollary}
\begin{proof}
 We invoke \Cref{lemma:QuadraticPenaltyPropertiesSmoothCaseHelperPPPM}
 to obtain that each subproblem \eqref{eq:MainPPPMLoop} is
 $\rho_\phi(\epsilon)$-weakly convex with
 \begin{align*}
  \rho_\phi(\epsilon)
  = \beta(\epsilon)\cdot G\DDXApp\cdot \min\cbrac{\rho, L} = \OO{\epsilon^{-1}},
 \end{align*}
 and $L_\phi$-smooth, where
 \begin{align*}
  L_\phi(\epsilon) = L_\phi = L + \rhoHat(\epsilon)
  + \beta(\epsilon)\cdot G\paren{\DDXApp L + G} = \OO{\epsilon^{-1}}.
 \end{align*}
 We define $\rhoHat(\epsilon) = \rhoHat \define 2\cdot \rho_\phi(\epsilon)$
 and by invoking \Cref{thm:MainResultPPPMQuadraticPenaltyConvergenceGuarantee}
 we get
 \begin{align*}
  N \geq \frac{\rhoHat(\epsilon) \DDUAppsq}{2 \mu_c^2} \beta(\epsilon)
  \cdot \log \paren*{3 \beta(\epsilon) \brac{\Delta_0
    + \beta(\epsilon)\ConstrViolZero}}
 \end{align*}
 as the number of outer iterations. To solve the inner
 $\rho_\phi(\epsilon)-$strongly convex,
 $L_\phi(\epsilon)$-smooth problems with an accuracy of
 \begin{align*}
  \epsin \leq \frac{\alpha}{2\beta(\epsilon)}
  \leq \frac{2 \mu_c^2}{3 \rhoHat(\epsilon) \DDUAppsq} \epsilon
  \frac{1}{2 \beta(\epsilon)} = \frac{\mu_c^2}{6 \DDUAppsq}
  \cdot \frac{\epsilon}{\beta(\epsilon) \cdot \rho_\phi(\epsilon)} = \OO{\epsilon^{3}}
 \end{align*}
 we use the gradient descent method, resulting in an inner complexity
 of
 \begin{align*}
  \TinnerUp{GD}(\epsin) \geq \frac{L_\phi(\epsilon)}{\rho_\phi(\epsilon)}
  \cdot \log \paren*{\frac{L_\phi(\epsilon) \DDXAppsq}{\epsin}},
 \end{align*}
 cf. \cite{nesterovLecturesConvexOptimization2018}. Multiplying the two rates gives
 \begin{align*}
  \TtotalUp{\smoothLabel} & \geq N \cdot \TinnerUp{GD}(\epsin)                                                   \\
                          & = \frac{2\rho_\phi(\epsilon) \DDUAppsq}{2 \mu_c^2} \beta(\epsilon)
  \cdot \log \paren*{3 \beta(\epsilon) \brac{\Delta_0
    + \beta(\epsilon)\ConstrViolZero}}
  \cdot \frac{L_\phi(\epsilon)}{\rho_\phi(\epsilon)}
  \cdot \log \paren*{\frac{L_\phi(\epsilon) \DDXAppsq}{\epsin}}                                                  \\
                          & = \frac{\DDUAppsq}{\mu_c^2} \beta(\epsilon)
  \cdot \log \paren*{3 \beta(\epsilon) \brac{\Delta_0
    + \beta(\epsilon) \ConstrViolZero}}
  \cdot L_\phi(\epsilon)
  \cdot \log \paren*{\frac{L_\phi(\epsilon) \DDXAppsq}{\epsin}}                                                  \\
                          & = \OO{\epsilon^{-2} \cdot \log\paren{\epsilon^{-2}} \cdot \log\paren{\epsilon^{-4}}}
  = \OOTilde{\epsilon^{-2}}
 \end{align*}
 and consequently concludes the claim.
\end{proof}

\begin{corollary}[IPPPM+SM, Non-Smooth, Quadratic Penalty]
 \label{cor:FinalRateNonSmoothQuadraticPenaltyIPPM}
 Under the assumptions of
 \Cref{thm:MainResultPPPMExactPenaltyConvergenceGuarantee}
 running \Cref{algo:PenaltyPPM} with the
 \textbf{quadratic penalty}
 $\penGen{\innerEmpty} \define \pen{\innerEmpty}$,
 $\beta(\epsilon) \define (\nuOpt + 1)\frac{\nuOpt
   + \sqrt{\nuOpt^2 + 2}}{\epsilon}$,
 and $\rhoHat(\epsilon) \define 2 \cdot \rho_\phi(\epsilon)$
 yields the guarantee
 \begin{align*}
  \abs*{F_1(\xxAppN) - \Fopt} \leq \epsilon\quad \text{and} \quad
  \plus{F_2(\xxAppN)} \leq \epsilon.
 \end{align*}
 The total number of sub-gradient evaluations when
 using the \textbf{sub-gradient method}
 (SM) as an inner solver, is given by:
 \begin{align*}
  \TtotalUp{\nonSmoothLabel}  \geq \OOTilde{\epsilon^{-2}} \cdot \OO{\epsilon^{-4}}
  = \OOTilde{\epsilon^{-6}},
 \end{align*}
 where $G_\phi(\epsilon) \define G + \rhoHat(\epsilon) \cdot \DDXApp
  + 2 \beta(\epsilon) \cdot \DDXApp G^2 = \OO{\epsilon^{-1}}$
 and $\rho_\phi(\epsilon)
  = \beta(\epsilon)\cdot G\DDXApp\cdot \min\cbrac{\rho, L} = \OO{\epsilon^{-1}}$.
\end{corollary}

\begin{proof}
 We invoke \Cref{lemma:QuadraticPenaltyPropertiesSmoothCaseHelperPPPM}
 to obtain that each subproblem \eqref{eq:MainPPPMLoop} is
 $\rho_\phi(\epsilon)$-weakly convex with
 \begin{align*}
  \rho_\phi(\epsilon)
  = \beta(\epsilon)\cdot G\DDXApp\cdot \min\cbrac{\rho, L} = \OO{\epsilon^{-1}},
 \end{align*}
 and has $G_\phi(\epsilon)$ bounded sub-gradients, where
 \begin{align*}
  G_\phi(\epsilon) \define G + \rhoHat(\epsilon) \cdot \DDXApp
  + 2 \beta(\epsilon) \cdot \DDXApp G^2 = \OO{\epsilon^{-1}}.
 \end{align*}
 We define $\rhoHat(\epsilon) \define 2\cdot \rho_\phi(\epsilon)$
 and by invoking \Cref{thm:MainResultPPPMQuadraticPenaltyConvergenceGuarantee}
 we get
 \begin{align*}
  N \geq \frac{\rhoHat(\epsilon) \DDUAppsq}{2 \mu_c^2} \beta(\epsilon)
  \cdot \log \paren*{3 \beta(\epsilon) \brac{\Delta_0
    + \beta(\epsilon) \ConstrViolZero}} = \OOTilde{\epsilon^{-2}}
 \end{align*}
 as the number of outer iterations. To solve the inner
 $\rho_\phi(\epsilon)-$strongly convex problems with an accuracy of
 \begin{align*}
  \epsin \leq \frac{\alpha}{2\beta(\epsilon)}
  \leq \frac{2 \mu_c^2}{3 \rhoHat(\epsilon) \DDUAppsq} \epsilon
  \frac{1}{2 \beta(\epsilon)} = \frac{\mu_c^2}{6 \DDUAppsq}
  \cdot \frac{\epsilon}{\beta(\epsilon) \cdot \rho_\phi(\epsilon)} = \OO{\epsilon^{3}}
 \end{align*}
 we use the sub-gradient method (SM), resulting in an inner complexity of
 \begin{align*}
  \TinnerUp{SM}(\epsin) & \geq \frac{4 G_\phi(\epsilon)^2}{\rho_\phi(\epsilon) \cdot \epsin}
  \geq \frac{4 G_\phi(\epsilon)^2}{\rho_\phi(\epsilon) }
  \frac{6 \DDUAppsq}{\mu_c^2}
  \cdot \frac{\beta(\epsilon) \cdot \rho_\phi(\epsilon)}{\epsilon}                           \\
                        & = \frac{24 \DDUAppsq}{\mu_c^2} \cdot
  \frac{G_\phi(\epsilon)^2 \beta(\epsilon)}{\epsilon} = \OO{\epsilon^{-4}},
 \end{align*}
 cf. Corollary 3.3 in
 \citet{lanAlgorithmsStochasticOptimization2020}.
 Multiplying the two rates concludes the claim.
\end{proof}

A direct comparison between
\Cref{cor:FinalRateNonSmoothQuadraticPenaltyIPPM} and
\Cref{cor:FinalRateNonSmoothExactPenaltyIPPM} reveals that, in the
non-smooth case, the exact penalty function is the natural choice, leading
to improved numerical constants since $\plus{\innerEmpty{}}$ preserves the
weak convexity and the upper bound on the sub-gradients without introducing
additional constants.

In contrast, in the smooth case, the quadratic penalty function
$\pen{\innerEmpty{}}$ has to be used to smooth the penalty term. While this
approach ensures differentiability, it slightly worsens the convergence
guarantee compared the results in
\citet{fatkhullinGlobalSolutionsNonConvex2025}.

\subsection{PPPM under Hidden Strong Convexity}
\label{appendix:HiddenStrongConvexityPPPM}
In the case of hidden \emph{strong} convexity, we can significantly
strengthen the result of
\Cref{thm:MainResultIPPPMHiddenConvexPenaltyOptGap} to a logarithmic
sample complexity in $\epsilon^{-1}$ for the number of outer loops.

\begin{theorem}[Hidden Strongly Convex, PPPM, Penalty Function]
 \label{thm:MainResultIPPPMHiddenStronglyConvex}
 Assume \eqref{eq:MainProblem} is hidden \emph{strongly} convex
 with modulus
 $\mu_H > 0$ and we have access to an oracle $\Oracle$ according to
 \Assref{ass:PPPMinexactOracle}. For a target precision $\epsilon > 0$,
 lifting parameter $\rhoHat > \rho$,
 and a penalty parameter $\beta \geq 1$.
 With the choice of $\alpha \define \frac{\mu_c^2 \mu_H}{\rhoHat + \mu_c^2 \mu_H}
  \in \brac{0,1}$ and for $N\in \NN$, the last iterate of \Cref{algo:PenaltyPPM} satisfies
 \begin{align*}
  \penOpt{\penGenNot, \beta}(\xxAppN) =
  F_1(\xxAppN) - \Fopt + \beta\penGen{F_2(\xxAppN)} \leq \epsilon
 \end{align*}
 after
 \begin{align*}
  N \geq \frac{\rhoHat + \mu_c^2 \mu_H}{\mu_c^2 \mu_H} \cdot
  \log \paren*{\frac{2 \paren{\Delta_0 + \beta \ConstrViolZero}}{\epsilon}}
 \end{align*}
 iterations, where $\Delta_0 \define F_1(\xxAppHatK) - \Fopt$
 and $\ConstrViolZero \define \penGen{F_2(\xxAppZero)}$.
 With an inner precision of $\epsin \leq \frac{\alpha \epsilon}{2}$,
 the total rate is in general given by
 \begin{align*}
  \Ttotal \geq N \cdot \Tinner(\epsin)
  = \frac{\rhoHat + \mu_c^2 \mu_H}{\mu_c^2 \mu_H} \cdot
  \log \paren*{\frac{2 \paren{\Delta_0 + \beta \ConstrViolZero}}{\epsilon}}
  \cdot \Tinner\paren*{\frac{1}{2}\alpha \epsilon}.
 \end{align*}
\end{theorem}
\begin{proof}
 Similar to the proof of Theorem 5 in
 \citet{fatkhullinGlobalSolutionsNonConvex2025}, we can pick $\alpha$
 uniformly small to obtain a recursion of the form
 \begin{align*}
  F_1(\xxAppKnext) - \Fopt + \beta \penGen{F_2(\xxAppKnext)}
  \leq (1-\alpha)
  \paren{F_1(\xxAppK) - \Fopt + \beta \penGen{F_2(\xxAppK)}} + \epsin.
 \end{align*}
 By unrolling, we obtain
 \begin{align*}
  \penOpt{\penGenNot, \beta}(\xxAppN)
  = F_1(\xxAppN) - \Fopt + \beta \penGen{F_2(\xxAppN)}
  \leq (1-\alpha)^N
  \paren{\Delta_0 + \beta} + \frac{\epsin}{\alpha}.
 \end{align*}
 The claim follows by the choice of $\alpha, \epsin$ and $N$.
\end{proof}

\subsection{Technical Lemmas}
\label{Appendix-sec:PenaltyApproachAnalysis}

For our penalty approach, we are generally interested in weak convexity of
the penalty term, a (generalized) chain rule as well as bounded
sub-gradients and $L$-smoothness respectively. For the chain rule in the
smooth case, we use the result of
\citet{drusvyatskiyEfficiencyMinimizingCompositions2019}, and in the
non-smooth case we generalize their results to both, the special case of
the quadratic penalty $\pen{\innerEmpty{}}$ and the exact penalty
$\plus{\innerEmpty{}}$.

Before we start analyzing the two cases separately, we state a few utility
Lemmas.

\begin{lemma}[Lipschitz Continuity, Quadratic Penalty]
 \label{lem:AppendixLipschitzContinuityPenaltyFunction}
 Let $h: \YY \to \RR, \yy \mapsto \pen{\yy}$, $\YY\subset \RR^m$, $m \in \NN$
 be the quadratic penalty function over a $\DDY$-bounded domain. Then
 $h$ is $2\cdot\DDY$-Lipschitz continuous.

 Furthermore, if composed with a $G_f$-Lipschitz continuous function $f :
  \XXApp \to \RR^m$, over a $\DDXApp$-bounded, compact domain $\XXApp \subset
  \RR^d$, the resulting function $h \circ f$ is $2\cdot \DDXApp \cdot
  G_f^2$-Lipschitz continuous.
\end{lemma}
\begin{proof}
 We first consider only the quadratic norm.
 Since we assume to have a feasible point for our problem \eqref{eq:MainProblem},
 we can assume that $\YY$ contains the origin. More generally, one would obtain
 $2\sup_{\vv \in \YY} \norm{\plus{\vv}}$ as a Lipschitz constant.

 The penalty function is globally $2\cdot \DDY$-Lipschitz continuous. I.e.
 for $\vv, \ww \in \Image(f)$ we have
 \begin{align}
  \abs*{h(\vv) - h(\ww)} = \abs*{\norm{\vv}^2 - \norm{\ww}^2} = \abs*{ \IP{\vv+\ww}{\vv-\ww}  }
  \leq \paren*{\norm{\vv} + \norm{\ww}} \cdot \norm{\vv - \ww},
  \label{eq:InequalityLipschitzNorm}
 \end{align}
 which implies $2\cdot \DDXApp$-Lipschitz continuity on a bounded
 domain. The same also holds
 true if we substitute $\norm{\innerEmpty{}}^2$ by $\pen{\innerEmpty{}}$.

 Since $f$ is Lipschitz continuous and $\XXApp$ is compact and bounded, also
 the image $\YY \define \Image(f) \define f(\XXApp)$ is compact and bounded
 by $G_f\cdot \DDXApp$. Invoking \eqref{eq:InequalityLipschitzNorm} with
 $\vv = f(\xxApp), \ww = f(\yy)$, $\xxApp, \yy \in \XXApp$, yields that $h
  \circ f$ is $(2\cdot G_f \cdot \DDXApp \cdot G_f)$-Lipschitz continuous,
 concluding the claim.
\end{proof}

\begin{lemma}[Weak Convexity of Composition]
 \label{lemma:HelperWeakConvexCompositionNonSmooth}
 Let $f: \XXApp \to \RR$, with $\XXApp \subset \RR^d$ bounded and convex,
 be $\rho$-weakly convex
 with a bounded range, and
 let $h: \RR \to \RR$ be non-decreasing, convex and $G_h$-Lipschitz continuous
 on $\range(f)$. Then the composition $g \define h \circ f: \XXApp \to \RR$
 is $(G_h \cdot \rho)$-weakly convex.
\end{lemma}
\begin{proof}
 To show $\mu$-weak convexity, we need to prove that for all $\xxApp, \yy \in \XXApp$
 and all $t \in [0,1]$ the inequality
 \begin{align*}
  g(t\xxApp + (1-t)\yy) \leq tg(\xxApp) + (1-t) g(\yy) + \frac{\mu}{2}t(1-t)\norm{\xxApp - \yy}^2
 \end{align*}
 holds. Since by assumption $f$ is $\rho$-WC, we know that the function
 $\fTilde(\xxApp) \define f(\xxApp) + \frac{\rho}{2} \norm{\xxApp}^2$ is convex.
 Let $\xxApp, \yy \in \XXApp$ and $t \in [0,1]$ be arbitrary but fixed. By convexity of
 $\fTilde$, we obtain
 \begin{align*}
  f(t\xxApp + (1-t)\yy) + \frac{\rho}{2} \norm{t \xxApp + (1-t)\yy}^2
  \leq t\paren*{f(\xxApp) + \frac{\rho}{2}\norm{\xxApp}^2}
  + (1-t)\paren*{f(\yy) + \frac{\rho}{2}\norm{\yy}^2}.
 \end{align*}
 Since the Euclidean norm is $2$-strongly convex, we have
 \begin{align*}
  \norm{t\xxApp + (1-t) \yy}^2 = t \norm{\xxApp}^2 +
  (1-t) \norm{\yy}^2 - t(1-t)\norm{\xxApp - \yy}^2.
 \end{align*}
 and thus the left-hand side expands to
 \begin{align*}
   & f(t\xxApp + (1-t)\yy) + \frac{\rho}{2} \norm{t \xxApp + (1-t)\yy}^2                        \\
   & \quad= f(t\xxApp + (1-t)\yy) + \frac{\rho}{2} \brac*{t \norm{\xxApp}^2 + (1-t)\norm{\yy}^2
   - t(1-t)\norm{\xxApp - \yy}^2}.
 \end{align*}
 Rearranging yields
 \begin{align*}
  f(t\xxApp + (1-t)\yy) \leq tf(\xxApp) + (1-t) f(\yy)
  + \underbrace{\frac{\rho}{2} (1-t)t\norm{ \xxApp - \yy}^2}_{\defineRev \delta \geq 0}.
 \end{align*}
 By inserting the above into $h$ we obtain
 \begin{align*}
  g\paren*{t\xxApp + (1-t)\yy}
                  & = h\paren*{f(t\xxApp + (1-t)\yy)}                                      \\
  \overset{(i)}   & {\leq} h\paren*{tf(\xxApp) + (1-t) f(\yy) + \delta}                    \\
  \overset{(ii)}  & {\leq} h\paren*{tf(\xxApp) + (1-t) f(\yy)} + G_h \cdot \delta          \\
  \overset{(iii)} & {\leq} th\paren*{f(\xxApp)} + (1-t)h\paren*{f(\yy)} + G_h \cdot \delta \\
  \overset{}      & {=} tg\paren{\xxApp} + (1-t)g\paren{\yy}
  + \frac{G_h\cdot \rho}{2} (1-t)t\norm{ \xxApp - \yy}^2
 \end{align*}
 by applying in $(i)$ that $h$ is non-decreasing, in $(ii)$
 the Lipschitz continuity of $h$, and in $(iii)$
 the convexity of $h$. This concludes the claim.
\end{proof}

\subsubsection{Smooth Case: Quadratic Penalty}
We show that the penalty term is weakly convex, satisfies the chain rule
and we give a uniform upper bound for the gradients.

\begin{lemma}[Smooth, Quadratic Penalty]
 \label{lemma:QuadraticPenaltyPropertiesSmoothCaseHelperPPPM}
 Let $\XXApp \subset \RR^d$ be a $\DDXApp$-bounded, compact domain,
 and  $h: \RR \to \RR, \xxApp \mapsto \pen{\xxApp}$
 be the \textbf{quadratic penalty} function.
 We assume that $f: \XXApp \to \RR$ is $L$-smooth, $\rho$-weakly
 convex and has $G_f$ uniformly bounded gradients.
 Then the composition $g \define h \circ f: \XXApp \to \RR$ satisfies:
 \begin{itemize}
  \item (Weak Convexity) The penalty term $g$ is
        $(2G_f \DDXApp \cdot \min\cbrac{\rho, L})$-weakly convex.
  \item (Chain Rule) For all $\xxApp \in \XXApp$ we have
        \begin{align*}
         \nabla g(\xxApp) = h^\prime(f(\xxApp)) \cdot \nabla f(\xxApp)
         = 2\plus{f(\xxApp)} \cdot \nabla f(\xxApp).
        \end{align*}
  \item (Bounded Gradients) The gradients of $g$ are uniformly bounded by
        $G_g \define 2\DDXApp G_f^2$.
  \item ($L_g$-smoothness) The penalty term $g$ is $L_g$-smooth,
        with $L_g \define 2\paren{\Fmax L + G_f^2}$,
 \end{itemize}
 where $\max_{\xxApp \in \XXApp} \abs{f(\xxApp)} \defineRev \Fmax < \infty$
 is an upper bound.
\end{lemma}
\begin{proof}
 The weak convexity is a consequence of
 \Cref{lemma:HelperWeakConvexCompositionNonSmooth} and
 Lemma 4.2 in
 \citet{drusvyatskiyEfficiencyMinimizingCompositions2019}.
 We apply \Cref{lem:AppendixLipschitzContinuityPenaltyFunction} to obtain that
 $g$ is Lipschitz continuous. The third claim follows by
 Theorem 3.1 in
 \citet{drusvyatskiyEfficiencyMinimizingCompositions2019}.
 To show the $L_g$-smoothness, let $\xxApp, \yy \in \XXApp$. We have
 \begin{align*}
  \norm*{\nabla g(\xxApp) - \nabla g(\yy)}
   & \leq 2\cdot\norm*{\plus{f(\xxApp)} \nabla f(\xxApp) - \plus{f(\xxApp)}\nabla f(\yy)
  + \plus{f(\xxApp)} \nabla f(\yy)- \plus{f(\yy)} \nabla f(\yy)}                         \\
   & \leq 2 \cdot \norm{\plus{f(\xxApp)} \paren{\nabla f(\xxApp) - \nabla f(\yy)}}
  + 2\cdot \norm{\paren{\plus{f(\xxApp)} - \plus{f(\yy)}} \nabla f(\yy)}                 \\
   & \leq 2\cdot \Fmax L \norm{\xxApp - \yy} + 2\cdot G_f G_f \norm{\xxApp - \yy}        \\
   & = 2\paren{\Fmax L + G_f^2} \cdot \norm{\xxApp - \yy},
 \end{align*}
 since for all $\xxApp \in \XXApp$ we have
 $\plus{f(\xxApp)} \leq \abs{f(\xxApp)} \leq \max_{\xxApp \in \XXApp} \abs{f(\xxApp)}
  \defineRev \Fmax < \infty$.
 This concludes the claim.
\end{proof}

\subsubsection{Non-smooth Case: Exact and Quadratic Penalty}
We can generalize the result above to the case of non-smooth $f$ and $h$
being either the quadratic $\pen{\innerEmpty{}}$ or exact penalty
$\plus{\innerEmpty{}}$.

\begin{lemma}[Non-Smooth, Quadratic Penalty]
 \label{lemma:QuadraticPenaltyPropertiesNonSmoothCaseHelperPPPM}
 Let $f: \XXApp \to \RR$ be a $\rho$-weakly convex, (potentially) non-smooth,
 $G_f$-Lipschitz
 continuous function over a $\DDXApp$-bounded domain $\XXApp \subset \RR^d$.
 Furthermore, let
 $h: \RR \to \RR, \xxApp \mapsto \pen{\xxApp}$ be the \textbf{quadratic penalty} function.
 Then the composition $g \define h \circ f$ satisfies:
 \begin{itemize}
  \item (Weak Convexity) The penalty term $g$ is $(2G_f \DDXApp \cdot \rho)$-weakly convex.
  \item (Generalized Chain Rule) For all $\xxApp \in \XXApp$ we have
        \begin{align}
         \partial g(\xxApp) = 2 \plus{f(\xxApp)} \cdot \partial f(\xxApp).
        \end{align}
  \item (Bounded Sub-gradients) The sub-gradients of $g$ are uniformly bounded
        by $G_g \define 2\DDXApp G_f^2$.
 \end{itemize}
\end{lemma}
\begin{proof}
 By \Cref{lem:AppendixLipschitzContinuityPenaltyFunction} we know that $h$ is
 $(2 G_f \DDXApp)$—Lipschitz continuous and the third claim is proved.
 Consequently, the first claim follows by
 applying \Cref{lemma:HelperWeakConvexCompositionNonSmooth}.

 Since $g$ is weakly convex, it is \emph{regular} by Proposition 4.4 in
 \citet{vialStrongWeakConvexity1983}, according to Definition 3 in
 \citet{liUnderstandingNotionsStationarity2020} of regularity, which is by
 Corollary 8.11 in \citet{rockafellarVariationalAnalysis1998} equivalent to
 Definition 6.24 and Corollary 6.29 in
 \citet{rockafellarVariationalAnalysis1998} of regularity. Applying the
 generalized chain rule for sub-gradients, Theorem 10.49 of
 \citet{rockafellarVariationalAnalysis1998} and using the fact that $h$ is
 continuously differentiable concludes the claim.
\end{proof}

\begin{lemma}[Non-Smooth, Exact Penalty]
 \label{lemma:ExactPenaltyPropertiesNonSmoothCaseHelperPPPM}
 Let $f: \XXApp \to \RR$ be a $\rho$-weakly convex, (potentially) non-smooth,
 $G_f$-Lipschitz
 continuous function over a $\DDXApp$-bounded domain $\XXApp \subset \RR^d$.
 Furthermore, let
 $h: \RR \to \RR, \xxApp \mapsto \plus{\xxApp}$ be the \textbf{exact penalty} function.
 Then the composition $g \define h \circ f$ satisfies:
 \begin{itemize}
  \item (Weak Convexity) The penalty term $g$ is $\rho$-weakly convex.
  \item (Generalized Chain Rule) For all $\xxApp \in \XXApp$ we have
        \begin{align}
         \partial g(\xxApp) =
         \begin{cases}
          \partial f(\xxApp),             & f(\xxApp) > 0, \\
          [0,1] \cdot \partial f(\xxApp), & f(\xxApp) = 0, \\
          \cbrac{0},                      & f(\xxApp) < 0,
         \end{cases}
         \label{eq:appendixGeneralizedChainRuleNonSmooth}
        \end{align}
        with a Minkowski product in the second case.
  \item (Bounded Sub-gradients) The sub-gradients of $g$ are uniformly
        bounded by $G_g \define G_f$.
 \end{itemize}
\end{lemma}
\begin{proof}
 If $f$ is $G_f$-Lipschitz, then trivially also $h\circ f$ is $G_f$-Lipschitz.
 Additionally, $\plus{\innerEmpty{}}$ is non-decreasing, convex,
 and $1$-Lipschitz. Thus, the first claim follows by applying
 \Cref{lemma:HelperWeakConvexCompositionNonSmooth}.
 With the same argument as above, we can apply the generalized chain rule
 for sub-gradients, Thm. 10.49 in
 \citet{rockafellarVariationalAnalysis1998}, and obtain
 \begin{align*}
  \partial_\xxApp g(\xxApp)
  = \bigcup_{y \in \partial h(f(\xxApp))} \partial(y f)(\xxApp),
 \end{align*}
 with $\xxApp \in \XXApp$. To obtain \eqref{eq:appendixGeneralizedChainRuleNonSmooth},
 we use the known sub-differential of $\plus{\xxApp}$:
 \begin{align*}
  \partial_\xxApp \plus{\xxApp} =
  \begin{cases}
   \cbrac{1},  & \xxApp > 0, \\
   \brac{0,1}, & \xxApp = 0, \\
   \cbrac{0},  & \xxApp < 0.
  \end{cases}
 \end{align*}
 Thus, at $f(\xxApp) = 0$, we obtain a convex combination of sub-gradients
 $g \in \partial f(\xxApp)$ and weights $\alpha \in [0,1]$. The result follows.
\end{proof}

\subsection{Strong Duality: Translation of Convergence Guarantees}
The challenge in \Cref{sec:TheoreticalAnalysis} is the translation of the
guarantee on $\Exp{\penOpt{\penGenNot, \beta}(\innerEmpty{})}$ to a
guarantee on the constraint violation, namely
$\Exp{\plus{F_2(\innerEmpty{})}}$. The following results show how to
perform this translation under strong duality for both, the exact and the
quadratic penalty function. For simplicity, we formulate the following
corollaries in the deterministic case. Taking expectation yields the
desired results.

\begin{lemma}[Upper Bound Constraint Violation]
 \label{lemma:AppendixTranslatePenaltyGuaranteeToFuncGapConstrViol}
 Let $\xxAppOpt$ be the optimal solution of \eqref{eq:MainProblemRL}
 and assume that strong duality holds, i.e. there exists $\nuOpt \geq 0$
 such that $\Fopt = d(\nuOpt)$. Then we have for all $\xxApp \in \XXApp$ the
 inequality
 \begin{align*}
  - \nuOpt \plus{F_2(\xxApp)} \leq F_1(\xxApp) - \Fopt.
 \end{align*}
 In particular, for all $\xxApp\in \XXApp$ and any $\beta \geq 2\cdot\nuOpt$,
 we have
 \begin{align*}
  0 & \leq F_1(\xxApp) - \Fopt + \beta \plus{F_2(\xxApp)} = \penOpt{1, \beta}(\xxApp).
 \end{align*}
\end{lemma}
\begin{proof}
 Lemma 3.1.21 in \citet{nesterovIntroductoryLecturesConvex2004}.
\end{proof}

\begin{corollary}[Constraint Violation, Exact Penalty]
 \label{cor:AppendixConstraintViolationExactPenaltyTranslation}
 Assume that for \eqref{eq:MainProblemRL} strong duality holds
 with a Lagrangian multiplier $\nuOpt \geq 0$, and let
 $\xxApp \in \XXApp$ be a point such that
 $\Exp{\penOpt{1, \beta}(\xxApp)} \leq \epsilon$ with the exact penalty
 and $\beta > \nuOpt$.
 Then we have the following translation for the constraint violation:
 \begin{align*}
  0 & \leq \Exp{\plus{F_2(\xxApp)}} \leq
  \frac{\Exp{\penOpt{1, \beta}(\xxApp)}}{\beta - \nuOpt}
  \leq \frac{\epsilon}{\beta - \nuOpt}.
 \end{align*}
\end{corollary}
\begin{proof}
 We invoke \Cref{lemma:AppendixTranslatePenaltyGuaranteeToFuncGapConstrViol}
 and obtain as a trivial consequence:
 \begin{align*}
  -\nuOpt \plus{F_2(\xxApp)} \leq F_1(\xxApp) - \Fopt + \beta \plus{F_2(\xxApp)}
  -\beta \plus{F_2(\xxApp)}.
 \end{align*}
 Taking expectation concludes the claim.
\end{proof}

\begin{corollary}[Constraint Violation, Quadratic Penalty]
 \label{cor:AppendixConstraintViolationQuadraticPenaltyTranslation}
 Assume that for \eqref{eq:MainProblemRL} strong duality holds
 with a Lagrangian multiplier $\nuOpt \geq 0$, and let
 $\xxApp \in \XXApp$ be a point such that
 $\Exp{\penOpt{2, \beta}(\xxApp)} \leq \epsilon$ with the quadratic penalty.
 Then we have the following translation for the constraint violation:
 \begin{align*}
  \Exp{\plus{F_2(\xxApp)}} \leq \frac{\nuOpt
   + \sqrt{\nuOpt^2 + 2 \beta \epsilon}}{\beta}.
 \end{align*}
\end{corollary}
\begin{proof}
 We invoke \Cref{lemma:AppendixTranslatePenaltyGuaranteeToFuncGapConstrViol}
 and obtain
 \begin{align*}
  \epsilon \geq \Exp{\penOpt{2, \beta}(\xxApp)}
  = \Exp{F_1(\xxApp) - \Fopt + \beta
   \pen{F_2(\xxApp)}}
  \overset{\text{Jensen}}{\geq}
  - \nuOpt \Exp{\plus{F_2(\xxApp)}}
  + \beta \left(\Exp{\plus{F_2(\xxApp)}}\right)^2.
 \end{align*}
 Solving
 \begin{align*}
  s\paren*{\beta \cdot s - \nuOpt} - \epsilon \leq 0, \;\text{s.t.} \;s \geq 0
 \end{align*}
 concludes the claim.
\end{proof}

\section{Experimental Details \& Ablation Studies}
\label{sec:AppendixNumericalExperiments}

\begin{figure}[H]
 \begin{subfigure}{0.48\textwidth}
  \centering
  \includegraphics[scale=0.3]{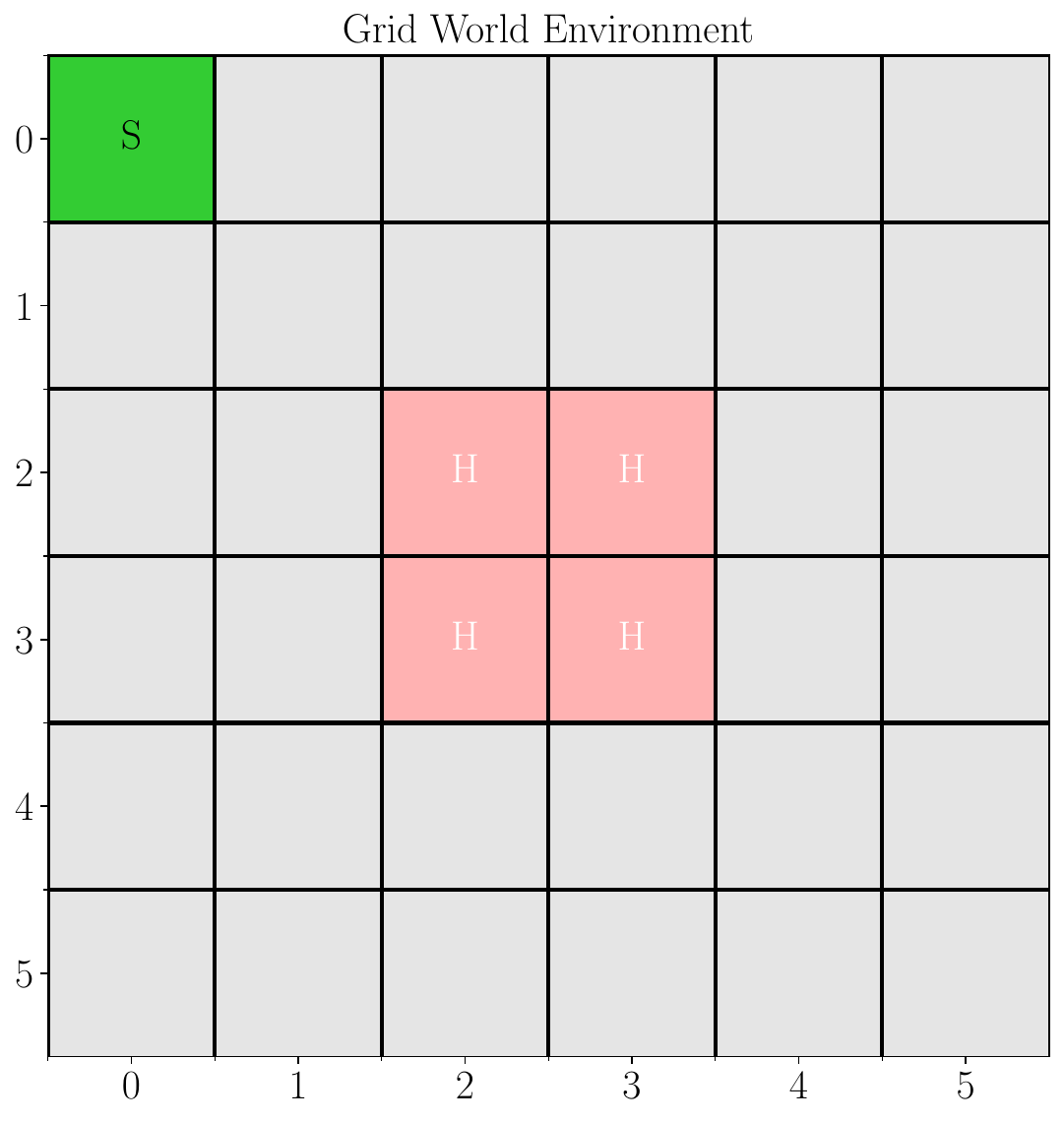}
  \caption{Visualization of the \emph{Frozen Lake} grid-world environment used in
   \Cref{sec:NumericalExperimentsLinearPerformanceConstraints} to test and ablate our
   \Cref{algo:Penalty}. "\texttt{S}" refers to the starting position of the agent,
   "\texttt{H}" to a (penalized) unsafe hole and " " to a safe tile.}
  \label{fig:SmallGridWorldVis}
 \end{subfigure}
 \hfill
 \begin{subfigure}{0.48\textwidth}
  \includegraphics[width=\textwidth]{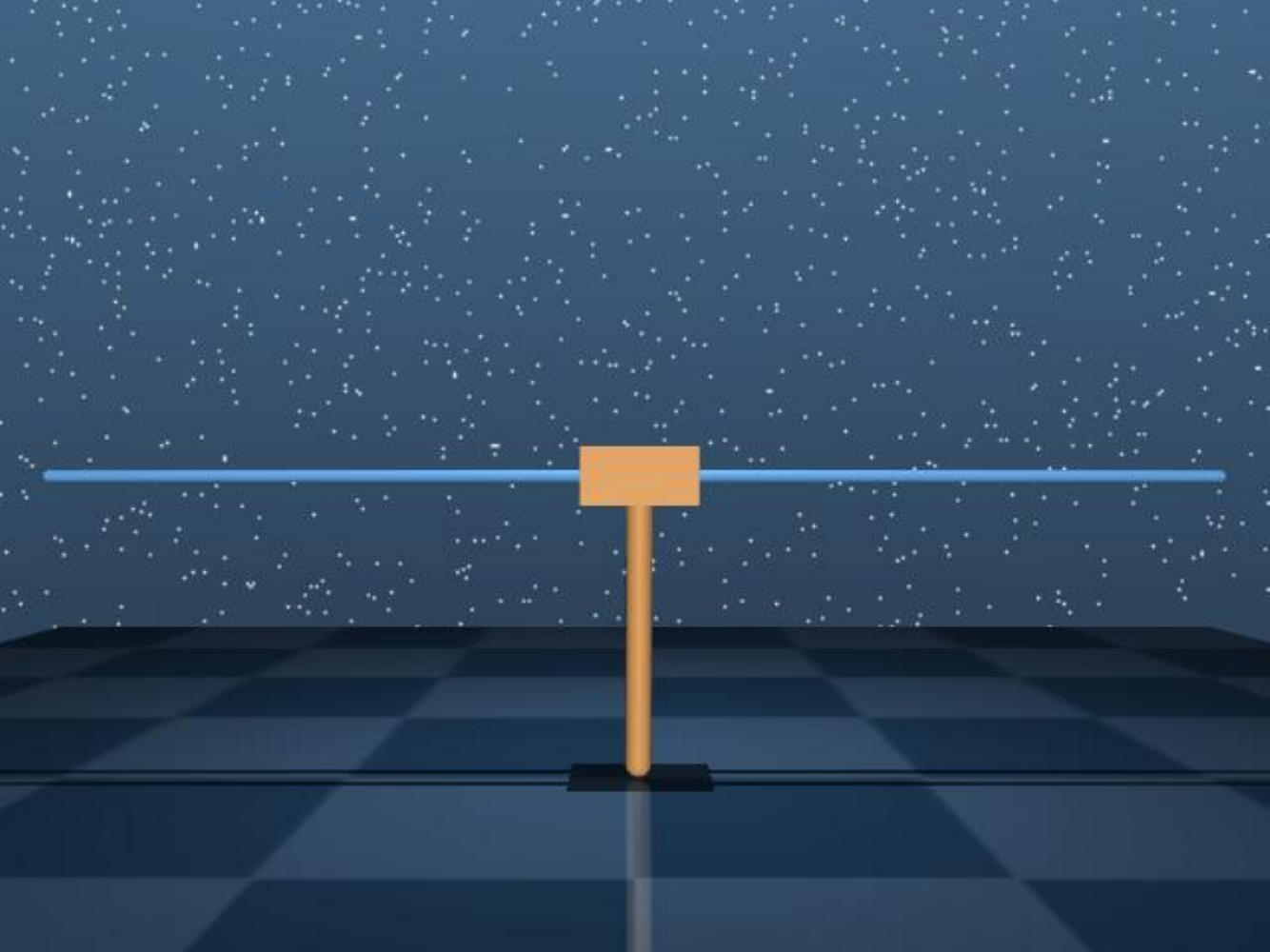}
  \caption{Visualization of the \emph{SafeCartpole} \cite{dulac-arnoldChallengesRealworldReinforcement2021}
   which is an extension of the Google DeepMind Control suite \cite{todorovMuJoCoPhysicsEngine2012,
    tassaDeepMindControlSuite2018}.
   The pendulum starts in the resting downwards position and the goal is to safely explore the
   dynamics of the system without exceeding an imposed slider safety limits on the left and right.}
  \label{fig:CartpoleVis}
 \end{subfigure}
 \caption{Overview of the grid-world with discrete state-action space (left) and
  the safe cartpole exploration with continuous state-action spaces (right).}
 \label{fig:Environment}
\end{figure}

\subsection{Grid World: Linear Performance Constraint}
\label{subsec:AppendixNumExpLinPerformanceConstraint}

In this ablation study, we perform a Maximum Entropy exploration under a
linear performance constraint, i.e. \eqref{eq:MainProblemRL} of the form
\begin{align*}
 \max_{\theta\in\Theta} & \; \HH(\lambda^{\pi_\theta})           \\
 \text{s.t.}            & \; \IP{c}{\lambda^{\pi_\theta}} \leq 0
\end{align*}
where the cost function is of the form
\begin{align}
 c(\ss) = \begin{cases}
           50,     & \ss = \texttt{Hole} \\
           -0.001, & \text{else}
          \end{cases},
 \tag{$\Reg_{\mathrm{lin}}$}
 \label{eq:AppendixCostFuncLinearPerformanceConstraint}
\end{align}
i.e. penalizing the four holes in the center of the environment \Cref{fig:SmallGridWorldVis}.
We ablate the penalty parameter $\beta \in \{5\cdot 10^{-4}, 0.001, 0.005, 0.01\}$, the batch sizes $B \in \{8,
 24, 72\}$, and the step sizes $\eta \in \{0.001, 0.01, 0.1\}$.
We want to highlight that episodes terminate upon reaching the bottom-right
tile. Consequently, the corresponding terminal state has an estimated
occupancy $\lambdaHat$ of almost zero for all methods, as it does not contribute
informative state-action visitation beyond termination.

The full ablation results of the objective are shown in
\Cref{fig:MaxEntropyLinConstraintAblationStudyObjective}, for the
constraint in \Cref{fig:MaxEntropyLinConstraintAblationStudyConstraint}
respectively.

\paragraph{Penalty Parameter $\beta$.}
Since our analysis requires $\beta \sim \epsilon^{-1}$, we validate
empirically that this is actually not required in practice. While bigger
$\beta$ might be indeed necessary to balance the tradeoff between
improvement on the objective function and the constraint violation, cf.
\Cref{fig:MaxEntropyLinConstraintAblationStudyConstraint}, we ablate in
\Cref{subsub:AppendixLinearConstraintLOWER_C} that the order of magnitude
of $\beta$ rather depends on the scale of the safety constraint and does
not necessarily have to be huge in practice. {Checkpoint convergence
  results for both the grid world and SafeCartpole are reported in
  \Cref{tab:BetaSensitivityConvergence} in the main text.}

\paragraph{Batch Size $B$.} While the penalty parameter $\beta$ plays a significant role ensuring
constraint violation, the batch size is elementary for the objective
function. Larger $B$ becomes more important if one wants to increase the
step size to speed-up the training process, cf.
\Cref{fig:MaxEntropyLinConstraintAblationStudyObjective}. Empirically, we
cannot observe that batches of $B \sim \epsilon^{-4}$ are necessary,
indicating that our estimates are most likely too conservative.

\paragraph{SGD Step Size $\eta$.}
The step size plays a crucial role, in particular to ensure convergence in
the objective. In particular the interplay between the batch size $B$ and
$\eta$ suggests, that both of them have to be chosen proportionally.

\paragraph{Implementation Details Primal-Dual \cite{yingPolicybasedPrimalDualMethods2025}.}
We implement the primal-dual method of \citet{yingPolicybasedPrimalDualMethods2025} as a baseline
and perform a hyperparameter search for the primal and dual learning rate
$\alpha_\theta, \alpha_\mu$ and the momentum term $\alpha_t$
in the notation
of \citet{yingPolicybasedPrimalDualMethods2025} Algorithm 1,
for parameters
$\alpha_\theta \in \{0.001, 0.005, 0.01, 0.05,0.1, 0.5, 1\}$,
$\alpha_\mu \in \{0.001, 0.005, 0.01, 0.05,0.1, 0.5, 1\}$, and
$\alpha_t \in \{0.001, 0.005, 0.01, 0.03, 0.05, 0.1, 1\}$.
We also ablate the batch size $B \in \{8, 24, 72\}$,
as for our PGP-method and pick the best configuration ($\alpha_\theta =
 0.01$, $\alpha_\mu= 0.001$, $\alpha = 1$, i.e. no momentum) displayed in
\Cref{fig:Figure1MaxEntropyPerformanceConstrainedVsUnconstrained}.
Interestingly, also for the primal-dual method, the batch size does not
significantly change the performance, which is why we focus in
\Cref{fig:Figure1MaxEntropyPerformanceConstrainedVsUnconstrained} on $B=8$
for the sake of comparability. In our experiments, we could not see an
improvement using their variance reduction momentum scheme:
the resulting policies exhibited the oscillating constraint violation
behavior and the best performance was achieved
with $\alpha_t \equiv 1$, $t \in \setm{T}$.

\subsubsection{Unconstrained Baseline}
\label{subsub:AppendixUnconstrainedBaseline}
For comparison, we also run the unconstrained baseline where we
maximize entropy without any safety constraints. In particular, we run our penalty algorithm
with $\beta = 0.0$ which corresponds to running unconstrained SGD to maximize the entropy
and evaluate the constraint violation with respect to
\eqref{eq:AppendixCostFuncLinearPerformanceConstraint}.
The results are shown in \Cref{fig:MaxEntropyUnconstrainedAblation}.

\begin{figure}
 \centering
 \begin{subfigure}{0.8\textwidth}
  \centering
  \includegraphics[width=\textwidth]{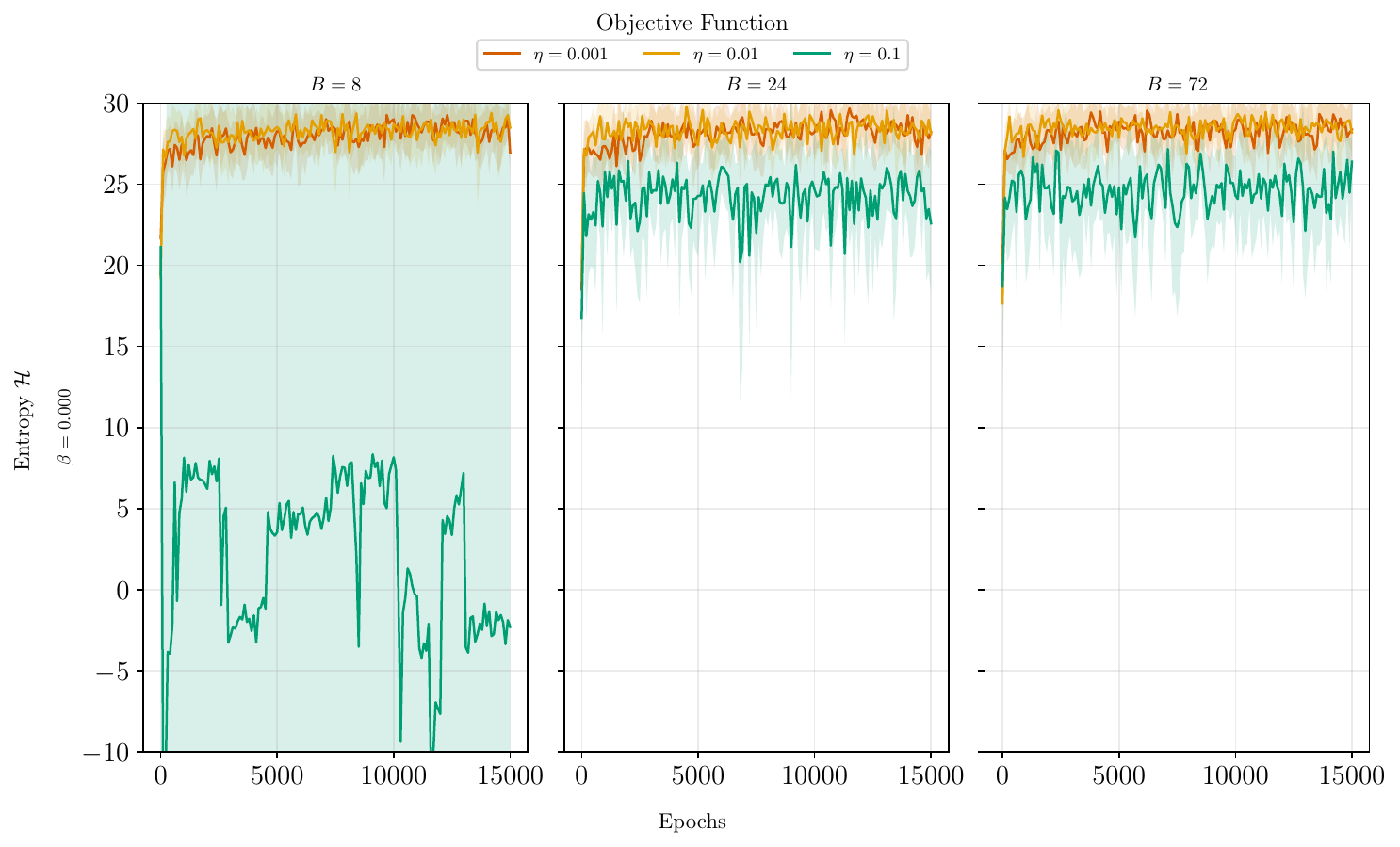}
  \caption{Unconstrained \textbf{entropy objective}.}
  \label{fig:MaxEntropyUnconstrainedAblationObjective}
 \end{subfigure}
 \hfill
 \begin{subfigure}{0.8\textwidth}
  \centering
  \includegraphics[width=\textwidth]{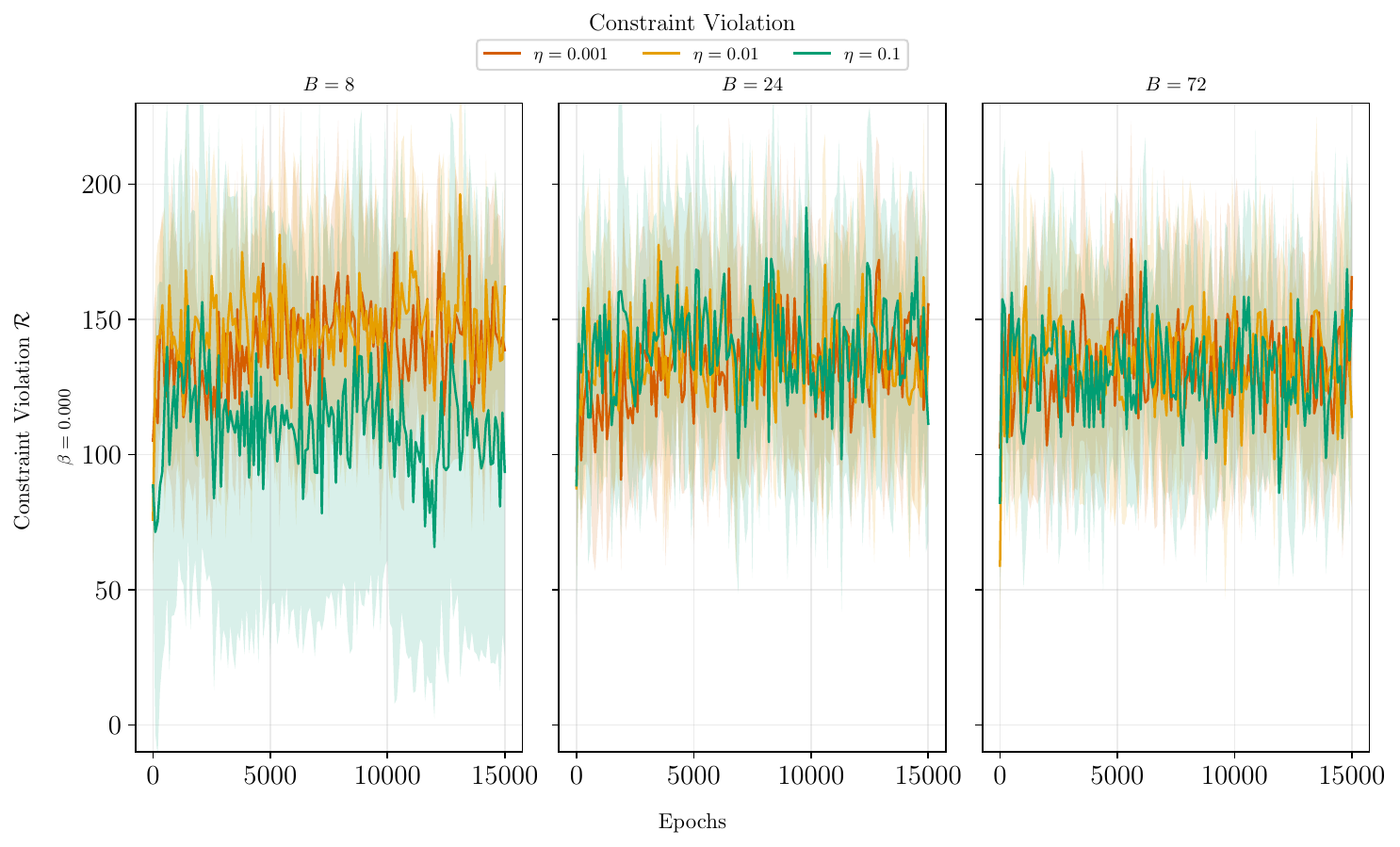}
  \caption{Unconstrained \textbf{constraint violation}.}
  \label{fig:MaxEntropyUnconstrainedAblationConstraint}
 \end{subfigure}
 \caption{Ablation study of \textbf{unconstrained entropy maximization} on the environment \Cref{fig:SmallGridWorldVis}
  with the (not considered) safety constraint \eqref{eq:AppendixCostFuncLinearPerformanceConstraint}.
  We ablate different batch sizes $B \in \{8, 24, 72\}$ and step sizes $\eta \in \{0.001, 0.01, 0.1\}$.
  Each curve shows the mean and one standard deviation, i.e. $\mu \pm \sigma$, for 10 different random seeds.}
 \label{fig:MaxEntropyUnconstrainedAblation}
\end{figure}

\begin{figure}
 \centering
 \includegraphics[width=\textwidth]{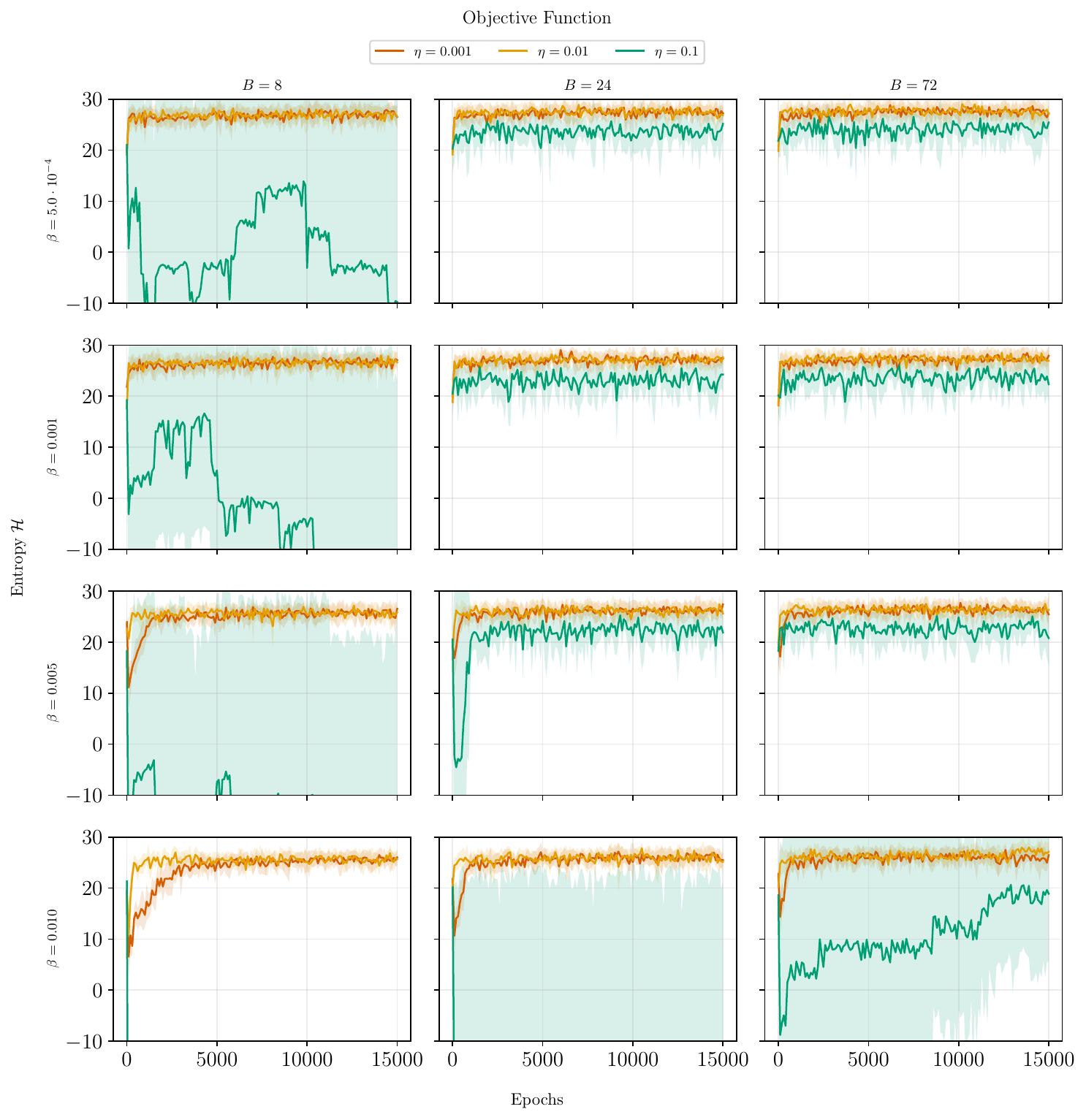}
 \caption{\textbf{Ablation study} of the \textbf{entropy function value} of our
  \Cref{algo:Penalty} under linear constraints \eqref{eq:AppendixCostFuncLinearPerformanceConstraint}
  running on the environment \Cref{fig:SmallGridWorldVis}.
  We ablate different batch sizes $B \in \{8, 24, 72\}$, penalty
  parameters $\beta \in \{5\cdot 10^{-4}, 0.001, 0.005, 0.01\}$
  and different step sizes $\eta \in \{0.001, 0.01, 0.1\}$
  of Stochastic Gradient Descent. Each curve shows the mean and one standard deviation
  of the entropy, i.e.
  $\mu \pm \sigma$, for 10 different random seeds.}
 \label{fig:MaxEntropyLinConstraintAblationStudyObjective}
\end{figure}

\begin{figure}
 \centering
 \includegraphics[width=\textwidth]{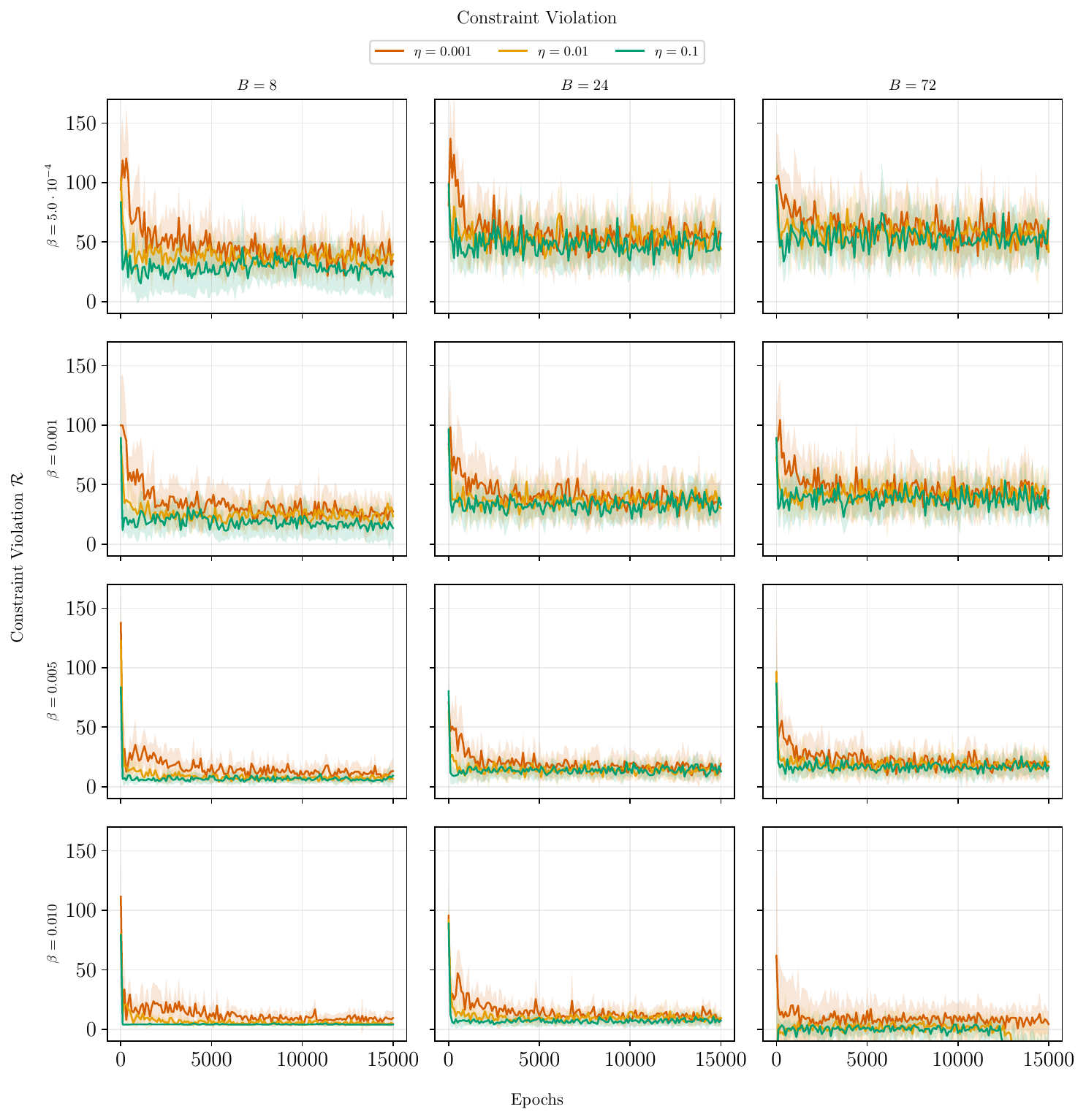}
 \caption{\textbf{Ablation study} of the \textbf{(linear) constraint function value}
  \Cref{eq:AppendixCostFuncLinearPerformanceConstraint} for our \Cref{algo:Penalty}
  running on the environment \Cref{fig:SmallGridWorldVis}.
  We ablate different batch sizes $B \in \{8, 24, 72\}$, penalty
  parameters $\beta \in \{5\cdot 10^{-4}, 0.001, 0.005, 0.01\}$
  and different step sizes $\eta \in \{0.001, 0.01, 0.1\}$
  of Stochastic Gradient Descent. Each curve shows the mean and one standard deviation
  of the linear constraint function, i.e.
  $\mu \pm \sigma$, for 10 different random seeds.}
 \label{fig:MaxEntropyLinConstraintAblationStudyConstraint}
\end{figure}

\subsubsection{Lower Penalty for Holes, Increased $\beta$ required}
\label{subsub:AppendixLinearConstraintLOWER_C}
As already suspected in
\Cref{sec:NumericalExperimentsLinearPerformanceConstraints}, we find that
choosing the penalty parameter $\beta$ does not have such a strong
dependency on $\sim\epsilon^{-1}$ but rather inversely proportionally
depends on the scale of the cost function. To test this hypothesis, we
modify the cost function
\eqref{eq:AppendixCostFuncLinearPerformanceConstraint} to a relaxed
version with a $-2$ penalty for the Holes, i.e.
\begin{align}
 c(\ss) = \begin{cases}
           2,      & \ss = \texttt{Hole} \\
           -0.001, & \text{else}
          \end{cases},
 \tag{$\Reg_{\mathrm{lin}, \mathrm{relax}}$}
 \label{eq:AppendixCostFuncLinearPerformanceConstraintLOWER_C}
\end{align}
as an alternative.
As mentioned in the analysis of \Cref{algo:Penalty} in
\Cref{sec:NumericalExperimentsLinearPerformanceConstraints},
we argue that the choice of $\beta$ in practice does not necessarily
scale with $\sim \epsilon^{-1}$, but should be rather chosen
depending on the scale of the reward, i.e. entropy maximization,
and the order of magnitude of the constraint, in order to
balance the constraint violation and the reward maximization.
In particular, if we adjust the safety formulation to
\eqref{eq:AppendixCostFuncLinearPerformanceConstraintLOWER_C},
i.e. a lower penalty for acting unsafe,
we require the penalty parameter $\beta$ to be larger to
achieve similar results, as shown in
\Cref{fig:MaxEntropyLinConstraintAblationStudyObjectiveLOWER_C}
and
\Cref{fig:MaxEntropyLinConstraintAblationStudyConstraintLOWER_C},
respectively, as we expected. Compared to the previous run, also in
this case the step size $\eta = 0.1$ seems to be too big to either
converge or deliver competitive results.

\begin{figure}
 \centering
 \begin{subfigure}[t]{0.32\textwidth}
  \centering
  \includegraphics[width=\textwidth]{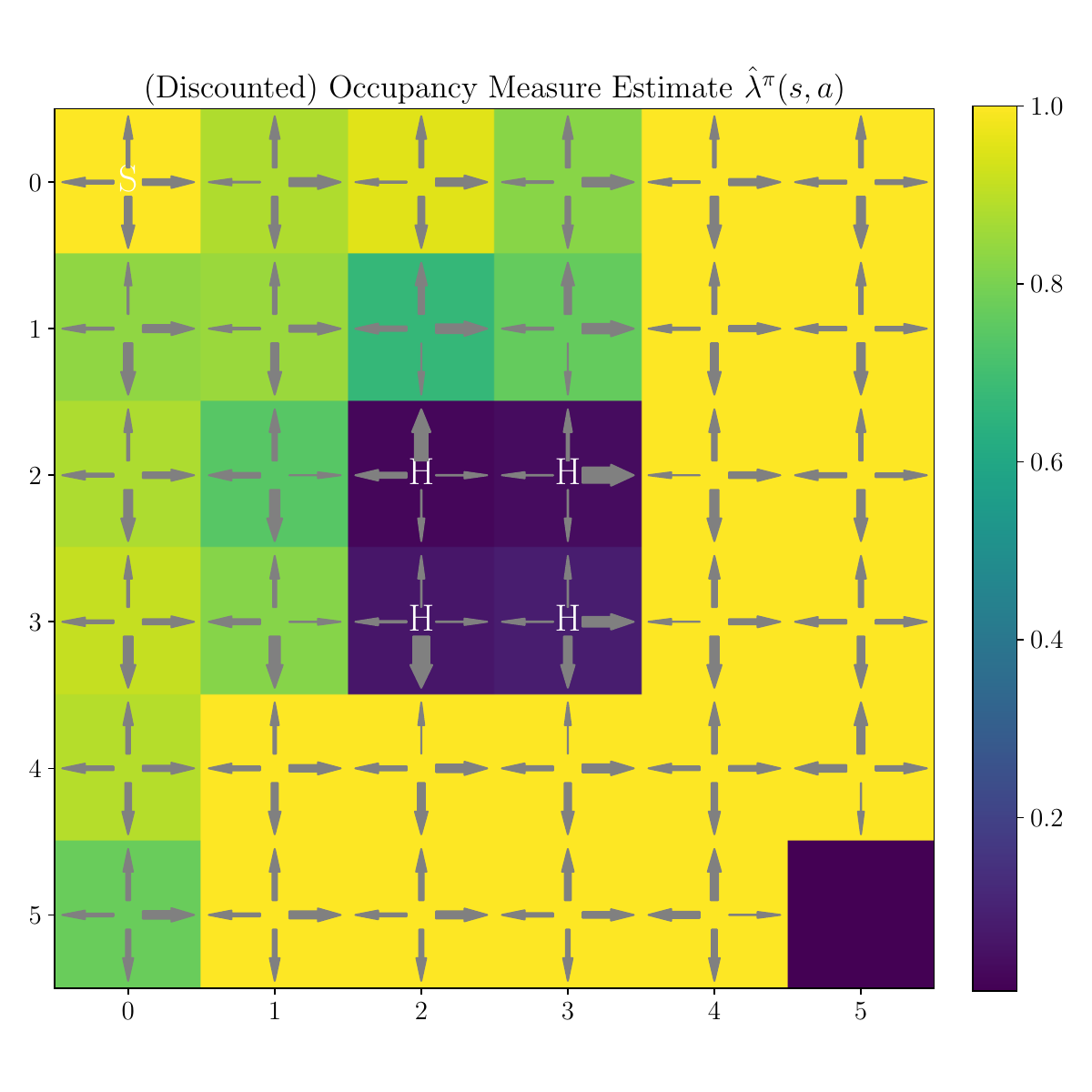}
  \caption{
   Final policy $\pi_{\theta^{(N)}}$ of executing Penalty Policy Gradient Method in \Cref{algo:Penalty}
   with $\beta= 0.005$ and the constraint
   \eqref{eq:AppendixCostFuncLinearPerformanceConstraint}.}
  \label{fig:VisPolicyMaxEntropyLinearConstraint}
 \end{subfigure}
 \hfill
 \begin{subfigure}[t]{0.32\textwidth}
  \centering
  \includegraphics[width=\textwidth]{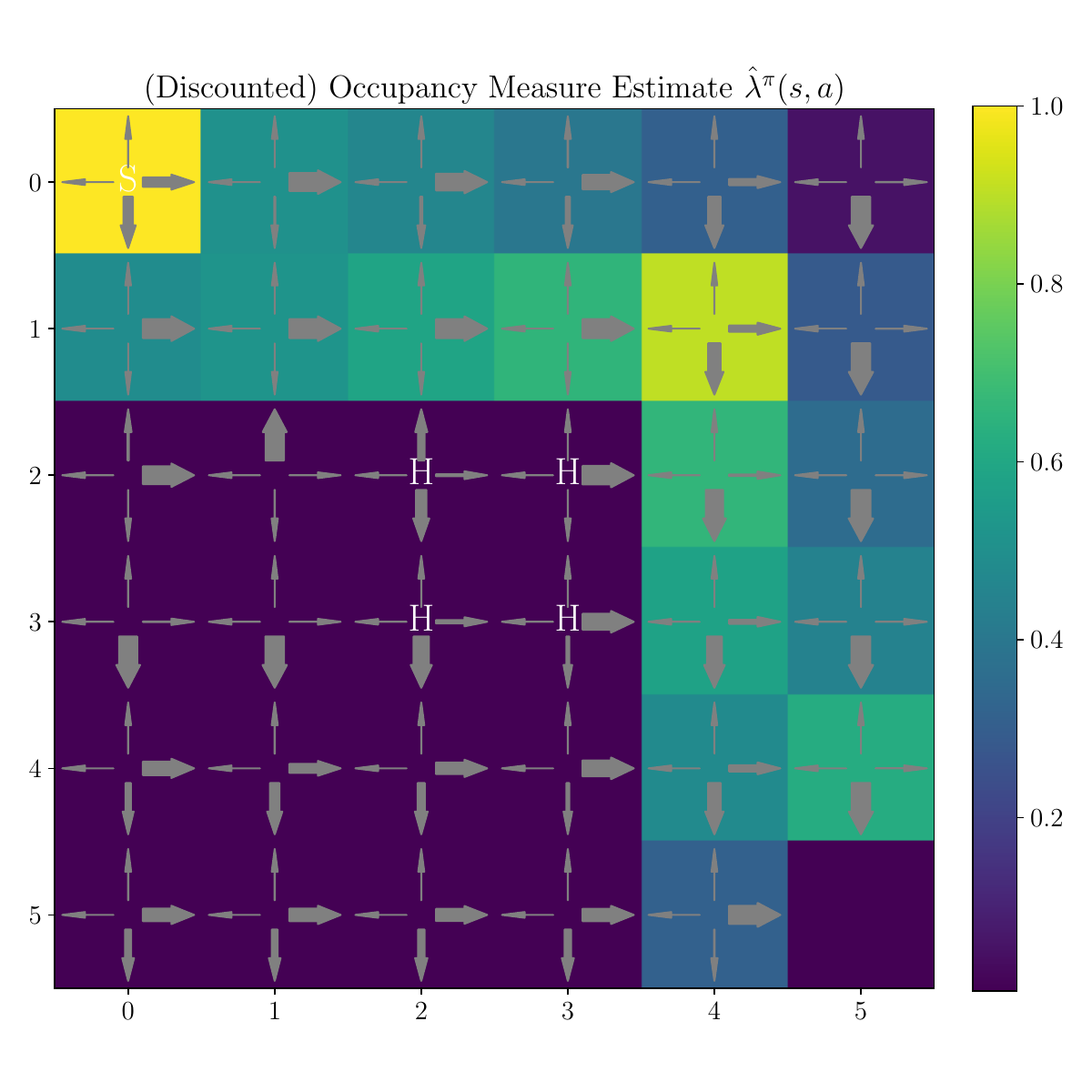}
  \caption{
   Final policy $\pi_{\theta^{(N)}}$ of executing the primal-dual method of \citet{yingPolicybasedPrimalDualMethods2025}
   with $\alpha_\theta= 0.001, \alpha_t \equiv 1$ (notation of their Algorithm 1) and the constraint
   \eqref{eq:AppendixCostFuncLinearPerformanceConstraint}.}
  \label{fig:VisPolicyMaxEntropyLinearConstraintPrimalDual}
 \end{subfigure}
 \hfill
 \begin{subfigure}[t]{0.32\textwidth}
  \centering
  \includegraphics[width=\textwidth]{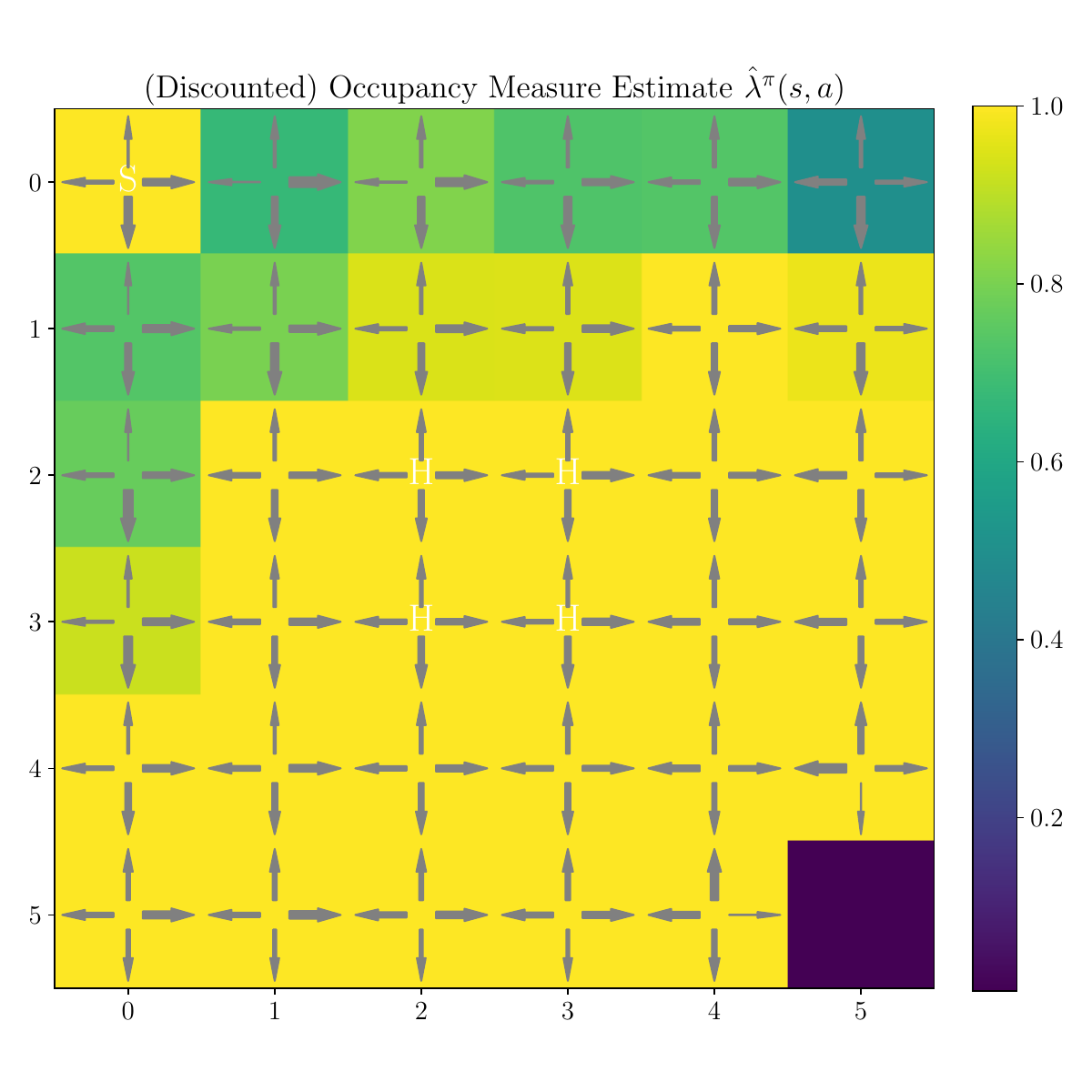}
  \caption{Final policy $\pi_{\theta^{(N)}}$ after running (unconstrained) SGD on \eqref{eq:MainProblemRL},
   i.e. \Cref{algo:Penalty} with $\beta = 0.0$.}
  \label{fig:VisPolicyMaxEntropyUnconstrained}
 \end{subfigure}
 \caption{Policies of \Cref{fig:Figure1MaxEntropyPerformanceConstrainedVsUnconstrained}:
 Comparison of the estimate of the
 state-action occupancy measure $\hat{\lambda}^{\pi^{(N)}}$ of
 the last iterate policies obtained by
 the Penalty Policy Gradient Method
 \Cref{algo:Penalty}, the primal-dual method of \citet{yingPolicybasedPrimalDualMethods2025},
 and running unconstrained SGD to
 maximize the entropy (without a safety constraint). For PGP and unconstrained SGD,
 we use a step-size of $\eta = 0.01$ and in all of the three cases we estimate the state-action
 occupancy measure via
 \eqref{eq:StateOccupancyMeasureMonteCarloEstimate}
 using a batch size of $B=8$.
 One can clearly observe that the penalized version respects the
 unsafe, penalized tiles in the center, while simply running
 SGD to maximize the entropy is not safe. The primal-dual algorithm exhibits
 constraint oscillations and yields a weaker performance, i.e. not fully exploring the
 space.}
 \label{fig:VisPolicyMaxEntropyComparison}
\end{figure}

\begin{figure}
 \centering
 \includegraphics[width=\textwidth]{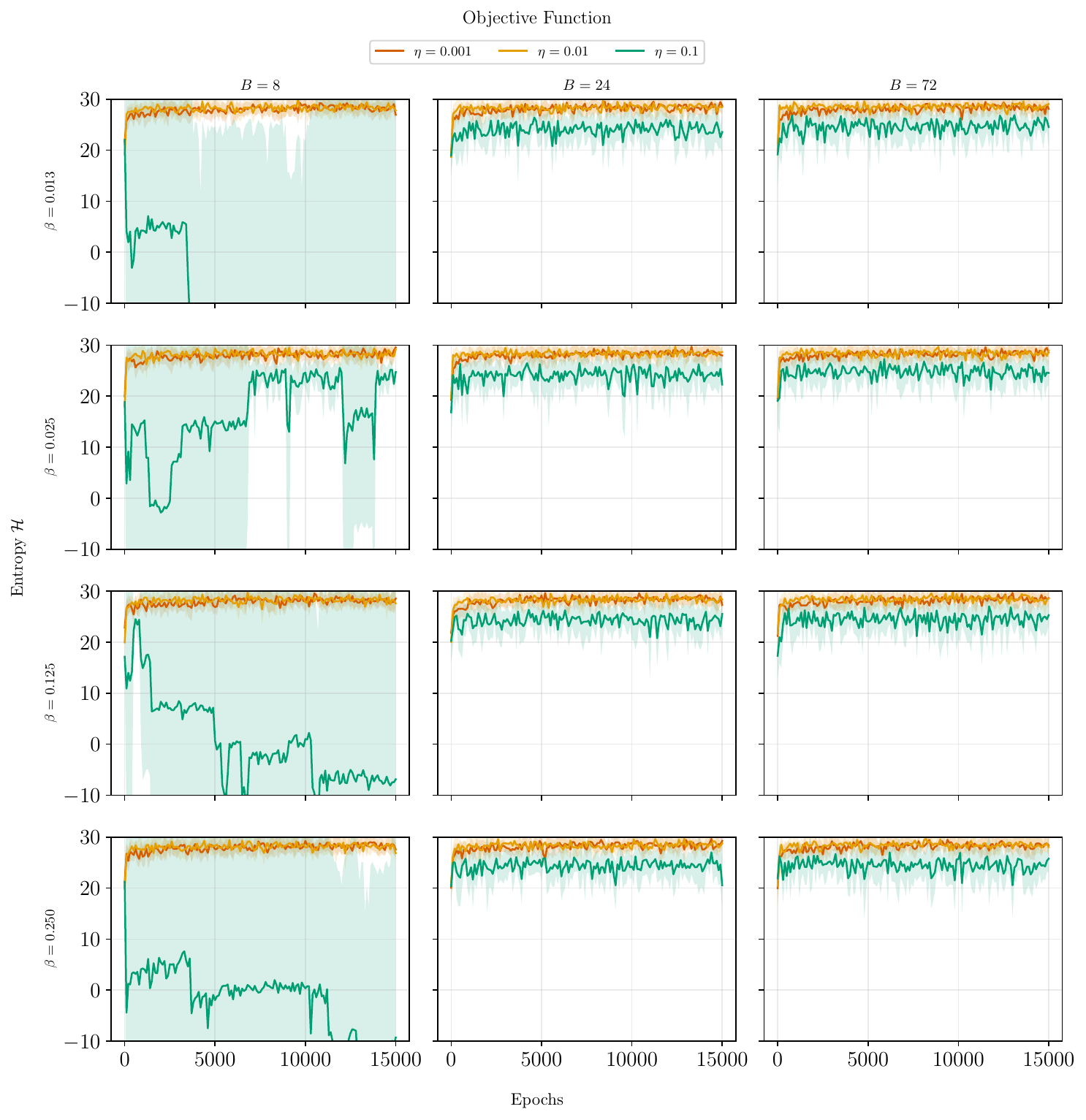}
 \caption{\textbf{Ablation study} of the \textbf{entropy function value} for our
  \Cref{algo:Penalty} under the \textbf{\underline{relaxed} (linear) constraint function value}
  \Cref{eq:AppendixCostFuncLinearPerformanceConstraintLOWER_C} for our \Cref{algo:Penalty}
  running on the environment \Cref{fig:SmallGridWorldVis}.
  We ablate different batch sizes $B \in \{8, 24, 72\}$, \textbf{larger penalty
   parameters} $\beta \in \{0.013, 0.025, 0.125, 0.25\}$
  and different step sizes $\eta \in \{0.001, 0.01, 0.1\}$
  of Stochastic Gradient Descent. Each curve shows the mean and one standard deviation
  of the entropy, i.e.
  $\mu \pm \sigma$, for 10 different random seeds.}
 \label{fig:MaxEntropyLinConstraintAblationStudyObjectiveLOWER_C}
\end{figure}
\begin{figure}

 \centering
 \includegraphics[width=\textwidth]{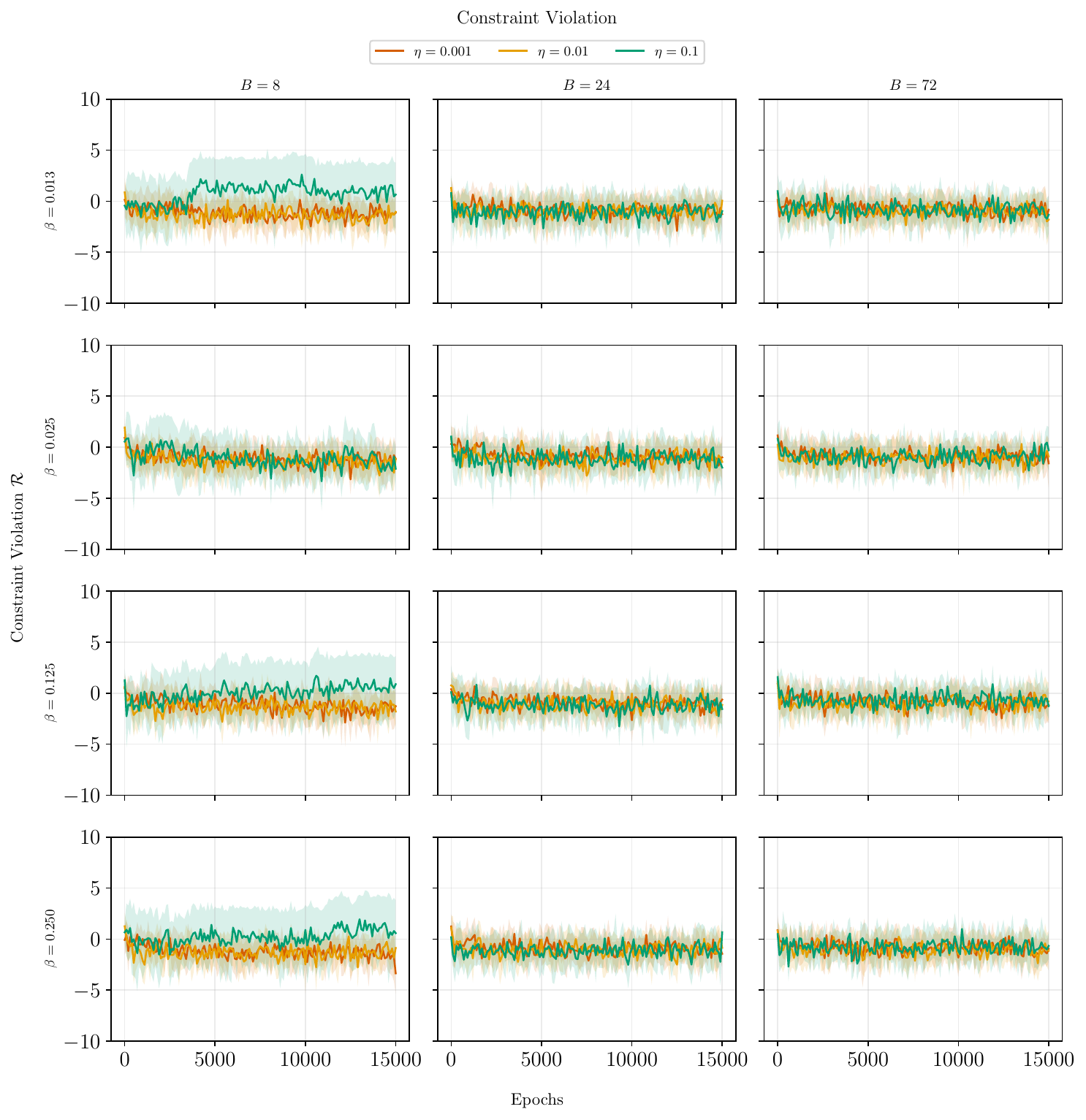}
 \caption{\textbf{Ablation study} of the \textbf{\underline{relaxed} (linear) constraint function value}
  \Cref{eq:AppendixCostFuncLinearPerformanceConstraintLOWER_C} for our \Cref{algo:Penalty}
  running on the environment \Cref{fig:SmallGridWorldVis}.
  We ablate different batch sizes $B \in \{8, 24, 72\}$, \textbf{larger penalty
   parameters} $\beta \in \{0.013, 0.025, 0.125, 0.25\}$
  and different step sizes $\eta \in \{0.001, 0.01, 0.1\}$
  of Stochastic Gradient Descent. Each curve shows the mean and one standard deviation
  of the linear constraint function, i.e.
  $\mu \pm \sigma$, for 10 different random seeds.}
 \label{fig:MaxEntropyLinConstraintAblationStudyConstraintLOWER_C}
\end{figure}

\subsection{Grid World: Reference Trajectory Constraint -- Norm Constraint}
\label{sec:AppendixGridWorldNormConstr}
In order to showcase our method's capability of performing maximum entropy exploration beyond linear
constraints, we focus on a safe exploration task under imitation learning constraints.
The reference policy $\pi_{\mathrm{ref}}$ is, in the notation of
\Cref{fig:SmallGridWorldVis}, first going all the way to the right and then
down until it terminates in the bottom right tile. We encode the
constrained problem with an $\ell_2$-norm constraint of the form
\begin{align}
 \Reg_{\ell_2}(\lambda^\pi) \define \norm{\lambda^\pi- \lambda^{\pi_{\mathrm{ref}}}}_2 \leq b,
 \tag{$\Reg_{\ell_2}$}
 \label{eq:PenaltyNormConstraint}
\end{align}
for a constraint violation budget $b \in \RR_{> 0}$.

We ablate different values of $b \in \{0.0, 0.2, 0.5, 1.0\}$ and run our
penalty method with $\beta = 0.5$ and a batch size of $B=8$.

\paragraph{Results.} The final policies for each $b$ are visualized in
\Cref{fig:MaxEntropyNormConstraintPolicy} and the corresponding training
progress in \Cref{fig:MaxEntropyNormConstraintObjectives}. One can clearly
see that as the shift $b$ increases, the constraint becomes looser; thus
the policy can deviate further from the reference policy and obtain a
higher entropy value.

\begin{figure*}[!htbp]
 \centering
 \setlength{\tabcolsep}{3pt}
 \begin{tabular}{cc}
  $b = 0.0$                                                                                             & $b = 0.2$ \\
  \includegraphics[width=0.45\linewidth]{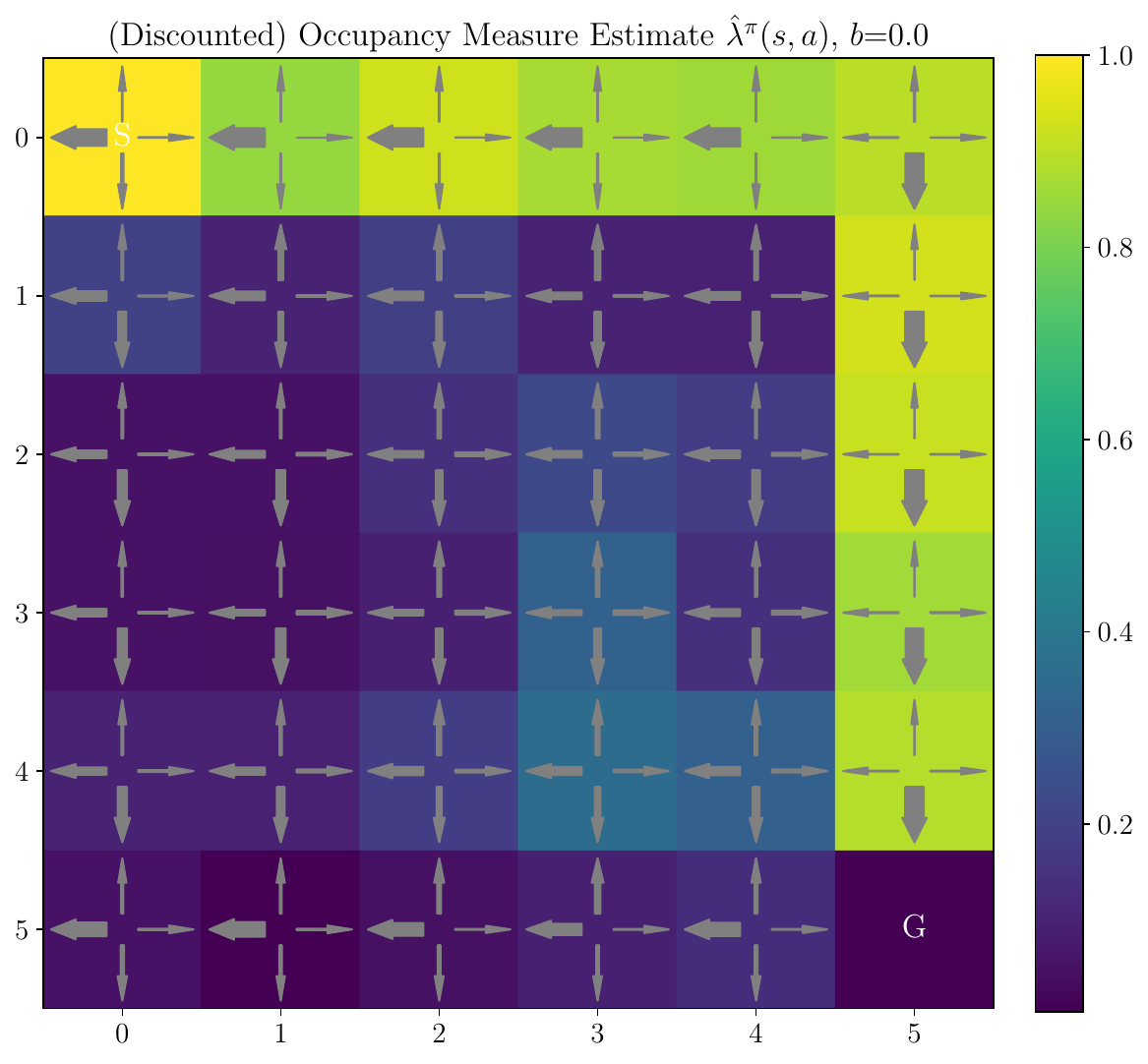} &
  \includegraphics[width=0.45\linewidth]{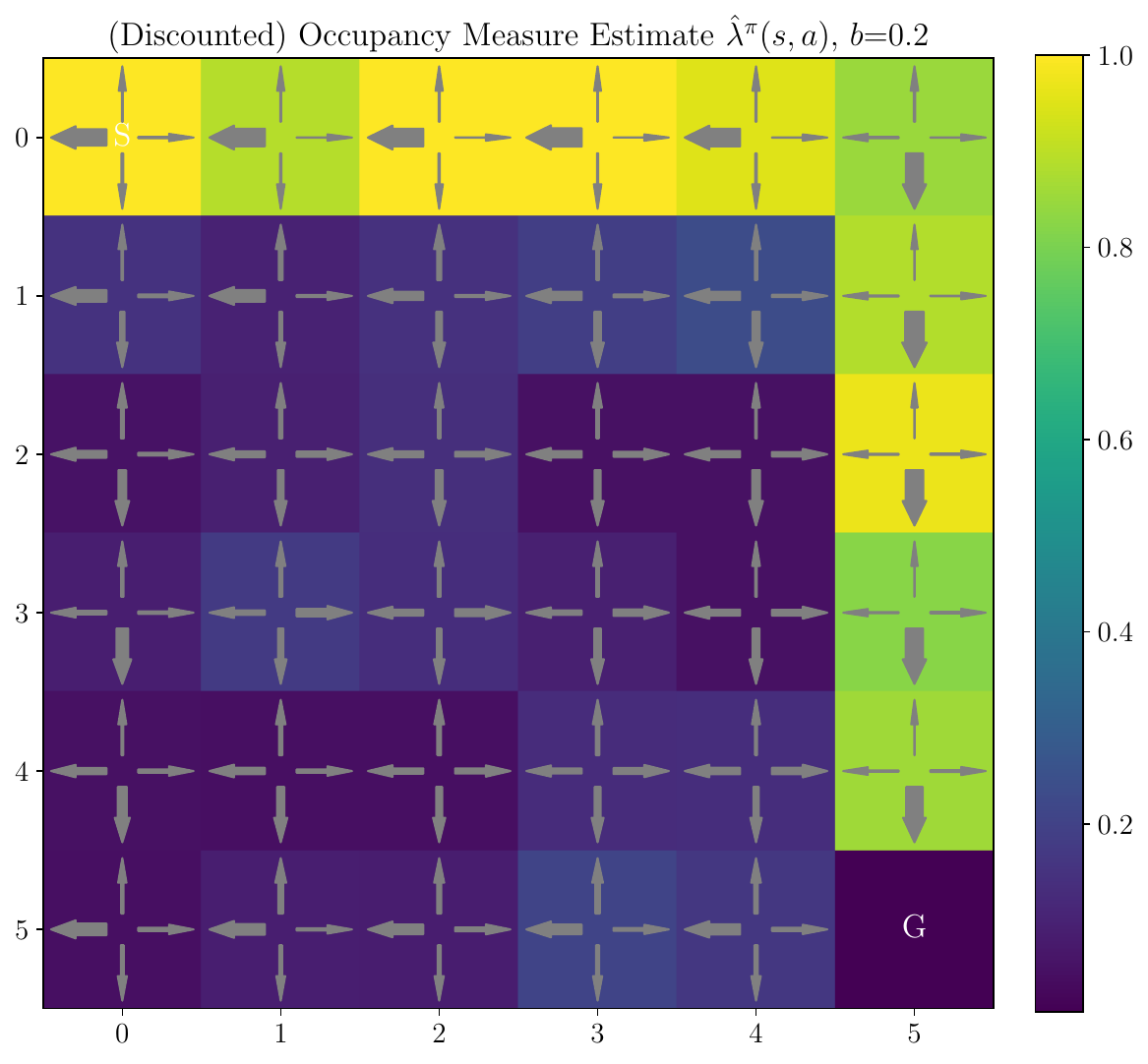}             \\
  $b = 0.5$                                                                                             & $b = 1.0$ \\
  \includegraphics[width=0.45\linewidth]{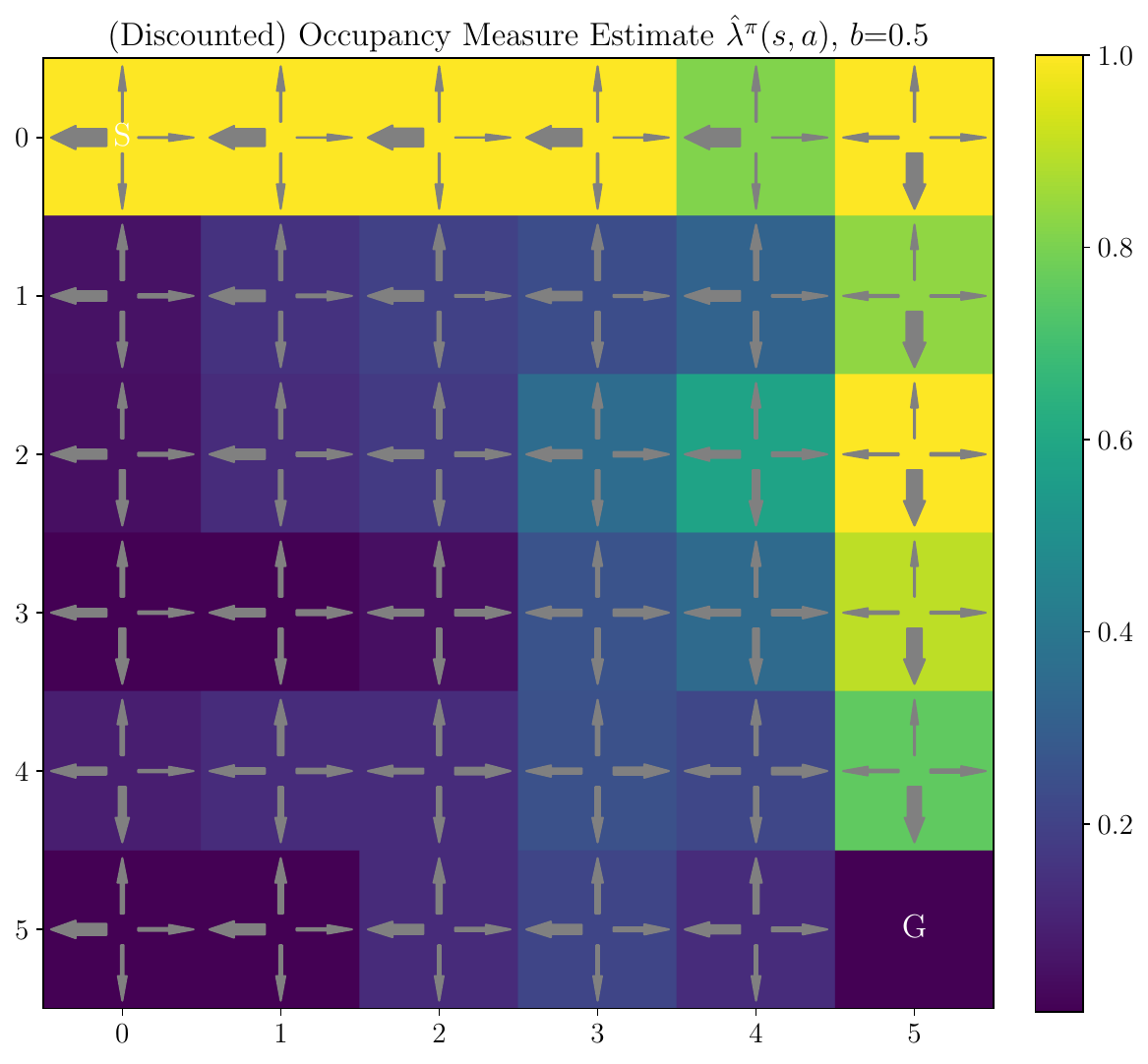} &
  \includegraphics[width=0.45\linewidth]{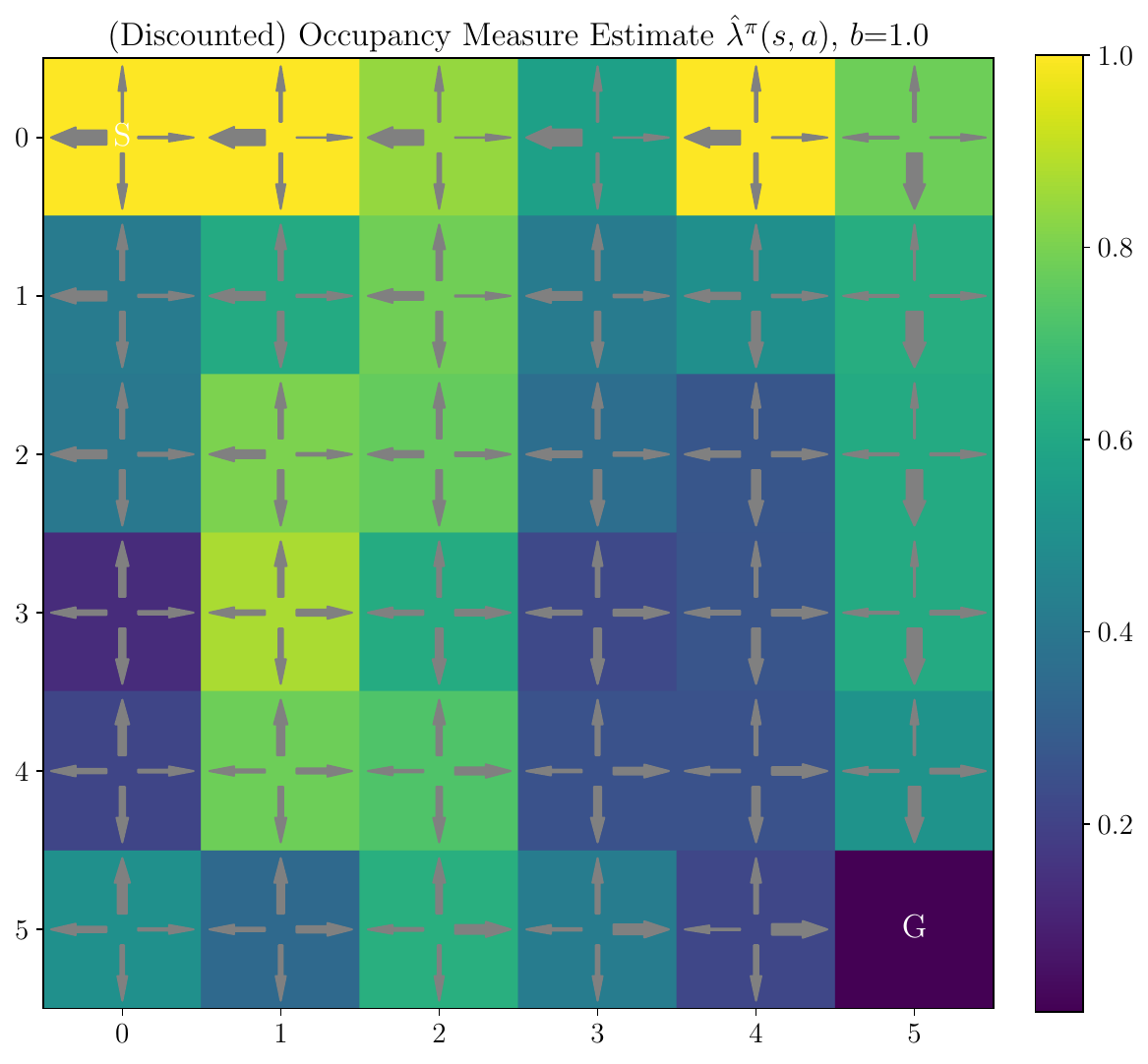}             \\
 \end{tabular}
 \caption{\textbf{Final policies} for maximum entropy with norm constraint \eqref{eq:PenaltyNormConstraint}
  for varying constraint shift parameter $b$.
  Each subfigure shows the learned final policy for different shifts
  $b = 0.0, 0.2, 0.5, 1.0$ (top-left to bottom-right) of the right-hand side of the constraint.
  Clearly, as $b$ increases, the constraint becomes looser, allowing the learned policy to
  deviate further from the reference policy.
 }
 \label{fig:MaxEntropyNormConstraintPolicy}
\end{figure*}

\begin{figure*}[!htbp]
 \centering
 \setlength{\tabcolsep}{3pt}
 \begin{tabular}{cc}
  $b = 0.0$                                                                                                           & $b = 0.2$ \\
  \includegraphics[width=0.45\linewidth]{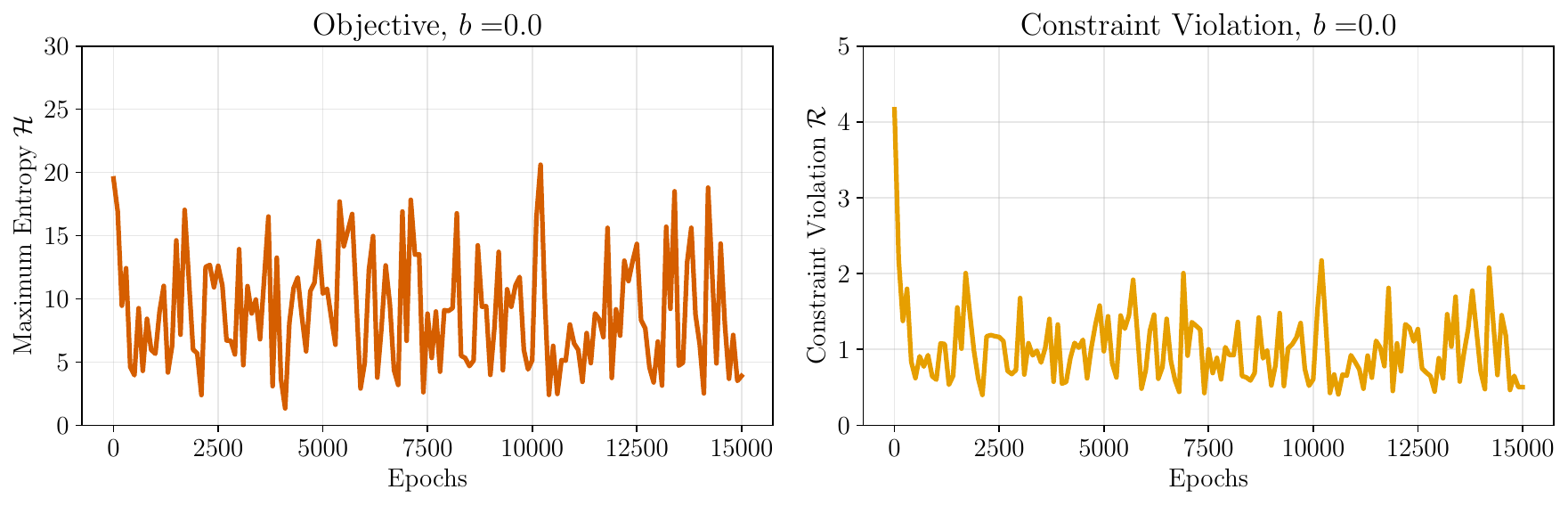} &
  \includegraphics[width=0.45\linewidth]{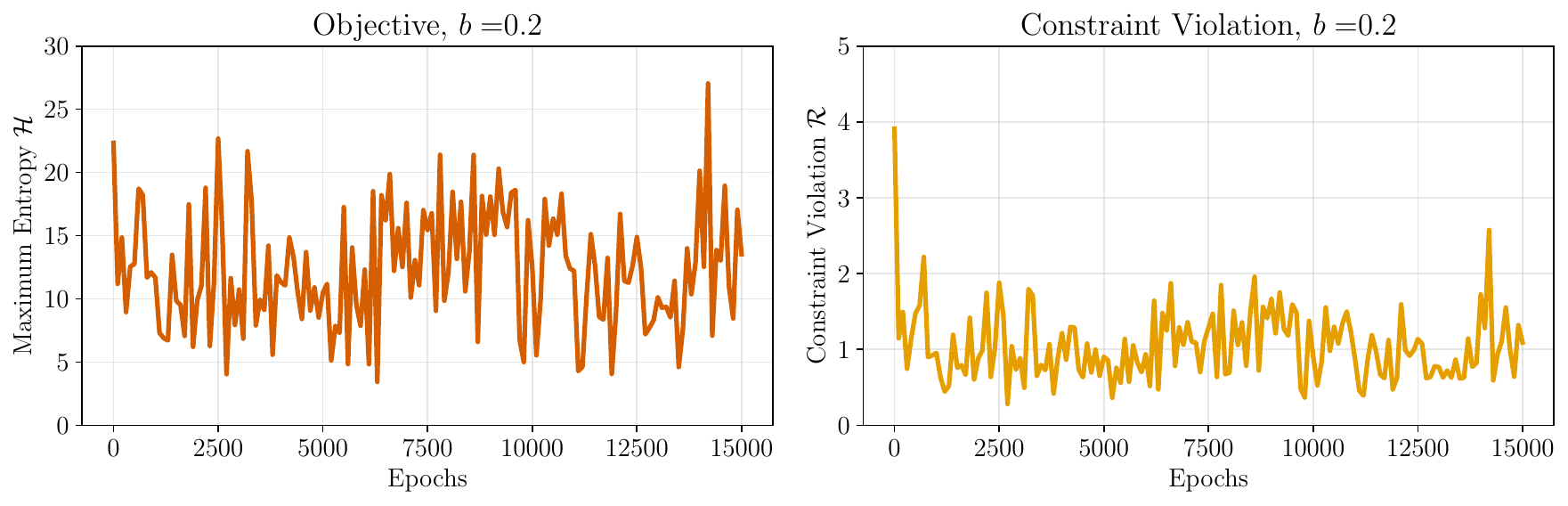}             \\
  $b = 0.5$                                                                                                           & $b = 1.0$ \\
  \includegraphics[width=0.45\linewidth]{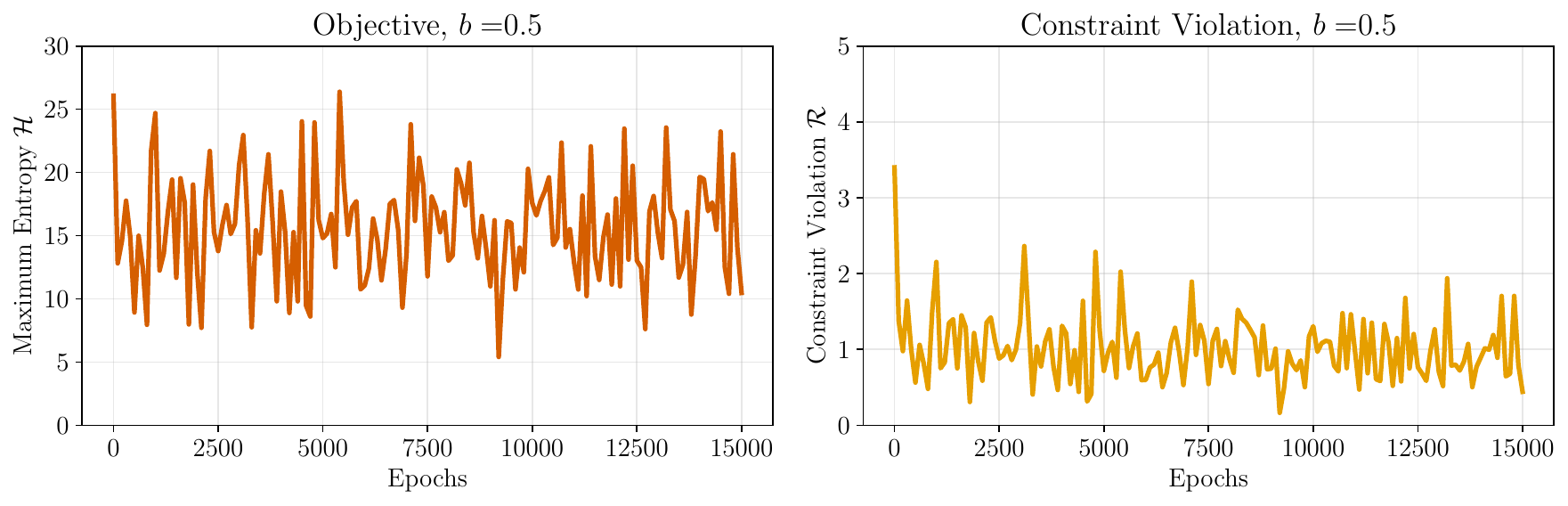} &
  \includegraphics[width=0.45\linewidth]{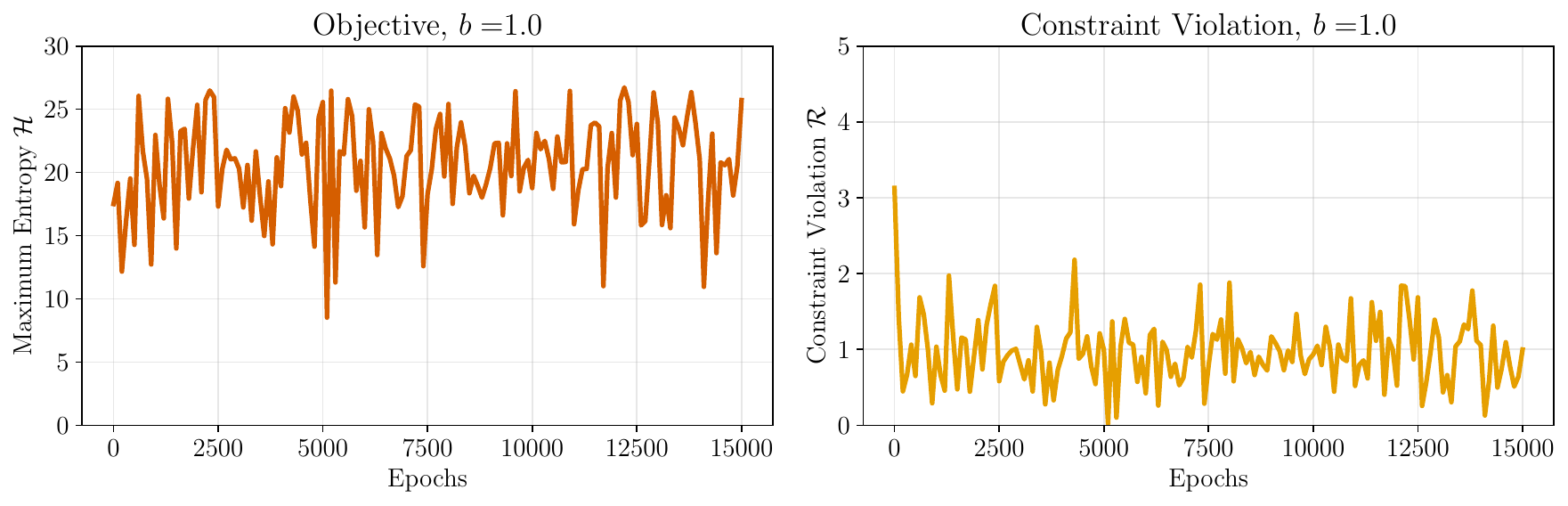}             \\
 \end{tabular}
 \caption{\textbf{Training progress} for maximum entropy with norm constraint \eqref{eq:PenaltyNormConstraint}
  for varying constraint shift parameter $b$.
  Each subfigure shows the training curves for $b = 0.0, 0.2, 0.5, 1.0$ (top-left to bottom-right).
  As $b$ increases, the constraint becomes looser, allowing the learned policy to deviate further from the reference policy.
 }
 \label{fig:MaxEntropyNormConstraintObjectives}
\end{figure*}

\subsection{Continuous State-Action Space}
\label{sec:AppendixContinuousSpaceExperiments}
We first discuss details on the estimation of the state-action occupancy
measure and then provide details on the \emph{PointMass} and the
\emph{SafeCartpole} experiment.

\subsubsection{Occupancy Measure Estimate}
\label{sec:AppendixLambdaEstimateContinuous}
When scaling to continuous state-action spaces in constrained maximum
entropy exploration, or more generally in convex RL, reliably estimating
the occupancy measure \eqref{eq:StateOccupancyMeasureMonteCarloEstimate} is
the main bottleneck. In our case, we use and implement the state-of-the-art
technique proposed by \citet{barakatScalableGeneralUtility2025} relying on
Maximum Likelihood Estimators (MLE) for learning the approximation $\lambdaHat$
compared to the Mean Squared Error, which previous techniques have been
proposing. Compared to our formulation of $\lambda^{\pi}: \SS \times \AA
 \to [0,1]$ as a function of the state and action space,
\citet{barakatScalableGeneralUtility2025} rewrite $\lambda^{\pi}$ of a
policy $\pi$ in the form
\begin{align*}
 \lambda^{\pi}(\ss, \aa) \define d^{\pi}(\ss) \cdot \pi(\aa\vert\ss),
\end{align*}
and only approximate the state-occupancy measure $d^{\pi}$ with a function approximator.

In particular, at every step of the algorithm, we collect a batch of
$N_{\lambdaHat} \in \NN$ samples $B_{\lambdaHat} \define \{\ss_i,
 \aa_i\}_{i\in \setm{N_{\lambdaHat}}}$, consisting of potentially several
rollouts and then approximate $d^{\pi}$ with a parametric class of
probability distributions
\begin{align*}
 \mathrm{GMM}(K) \define
 \left\{p_\wVec \define  \sum_{k=1}^{K}
 w_k \mathcal{N}(\mathbf{m}_k, \Sigma_k) :
 \wVec = (w_k, \mathbf{m}_k, \Sigma_k)_{k\in \setm{K}} \right. \\
 \left.\;\mathrm{s.t.}\;
 w_k \geq 0, \sum_{k=1}^{K} w_k = 1,
 \mathbf{m}_k \in \mathbb{R}^{d_s},
 \Sigma_k \in \mathbb{S}^{d_s}_{++},
 \right\}
\end{align*}
in our case Gaussian Mixture Models (GMMs), with $K=16$. The approximation
$\hat{d}^{\pi}$ is learned by maximizing the MLE loss,
or equivalently minimizing a Kullback Leibler divergence
\cite{stuartExpressivityMultimodalContrastive2026}, i.e.
\begin{align}
 \hat{d}^{\pi} \define p_{\wVec^*}, \quad
 p_{\wVec^*} \in \argmax_{p_\wVec \in \mathrm{GMM}(K)} \frac{1}{N_{\lambdaHat}}
 \sum_{i\in \setm{N_{\lambdaHat}}} \log p_\wVec(\ss_i).
 \tag{$\hat{d}^{\pi}$-Est}
 \label{eq:CartpoleStateOccupancyMeasEstimate}
\end{align}
Following \citet{barakatScalableGeneralUtility2025}, the estimate $\lambdaHat^{\pi}$
is then given by
\begin{align}
 \lambdaHat^{\pi}(\ss, \aa) \define \hat{d}^{\pi}(\ss) \cdot \pi(\aa \vert \ss).
 \tag{$\lambda^{\mathrm{MLE}}_H$-Est}
 \label{eq:StateOccupancyMeasureMLE}
\end{align}

\subsubsection{Imitation Learning Constraint: PointMass Exploration}
\label{sec:AppDetailsPointMass}
We provide additional implementation and computational details on
the \emph{PointMass} experiment.

\paragraph{State, Action \& Observation Space.} In the PointMass environment, the state, action and observation space are
\begin{align*}
 \SS \times \AA \times \Obs = \RR^2 \times \RR^2 \times \RR^2.
\end{align*}
At timestep $t$, each observation $o_t$ consists of the current $x$
and $y$ coordinate of the point mass. The action $\aa$ encodes the gain
along the $x$ and $y$-axis respectively.

\paragraph{Objective and Constraint: Definition \& Implementation.} First, we train a reference policy $\pi_{\mathrm{ref}}$ using the
Soft-Actor critic (SAC) algorithm of the \emph{StableBaselines3} package
\cite{raffinStableBaselines3ReliableReinforcement2021}. We fit the
reference GMM from stochastic SAC rollouts using the same discounted
weighted likelihood as in \eqref{eq:CartpoleStateOccupancyMeasEstimate}. On
a batch $B$, the apprenticeship constraint and its pseudo-reward are
\begin{align*}
 \widehat{\Reg}_{\rm KL}
  & \define\sum_{i\in B}\bar w_i
 \left(\log\lambdaHat^\pi(\ss_i,\aa_i)
 -\log\lambdaHat^{\rm ref}(\ss_i,\aa_i)\right)-\rMax, \\
 r_{\rm KL}(\ss_i,\aa_i)
  & \define1+\log\lambdaHat^\pi(\ss_i,\aa_i)
 -\log\lambdaHat^{\rm ref}(\ss_i,\aa_i).
\end{align*}
The budget $\rMax$ affects the constraint value but, being constant, does
not enter its pseudo-reward.

\paragraph{Analysis of Actions.}
\Cref{fig:PointMassControlTrajectory} shows that as the constraint budget
$\rMax$ increases, the actions of $\pi_\theta$ deviate progressively further
from those of $\pi_{\mathrm{ref}}$. This directly reflects the difficulty
of the task: the SAC reference policy consistently directs the agent toward
the goal $(0,0)$, so effective exploration requires substantial deviation
from $\pi_{\mathrm{ref}}$.

\begin{figure*}[!h]
 \centering
 \includegraphics[width=\textwidth]{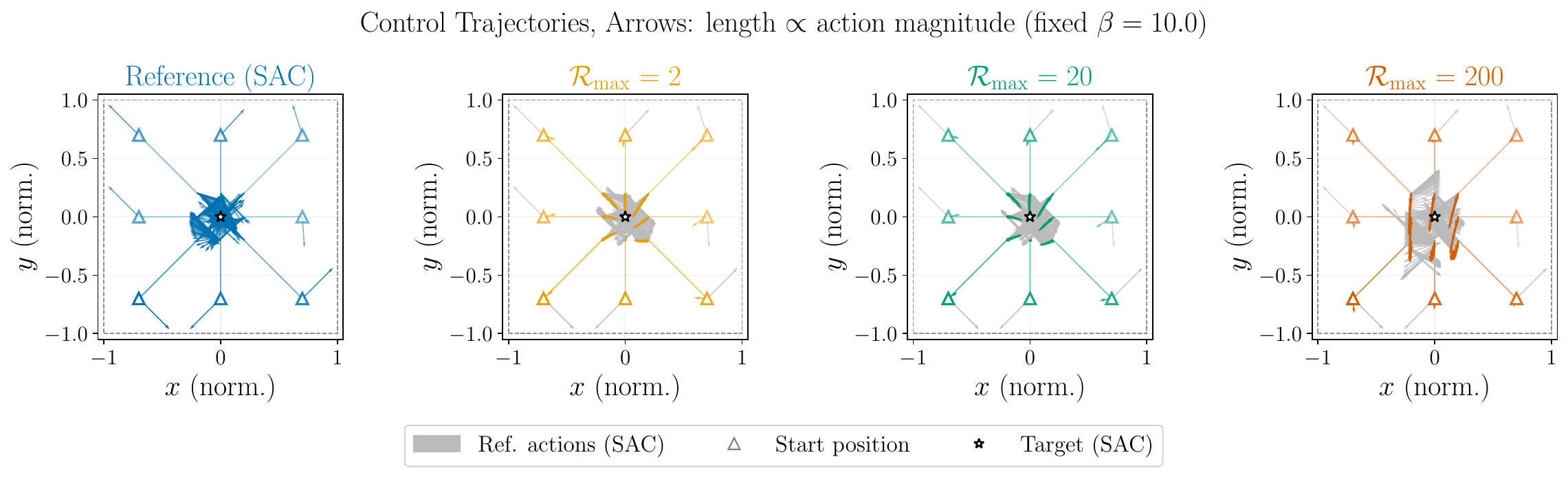}
 \caption{
  \textbf{PointMass: Control Trajectories.}
  Comparison of actions sampled from $\pi_\theta(\innerEmpty{} \mid \ss)$ for the
  reference SAC policy and
  PGP-trained policies with $\beta = 10$ and $\rMax \in \{2, 20, 200\}$,
  evaluated with the same fixed seed across 8 initial positions.
  Colored arrows indicate the actions selected by each policy at the current state.
  For PGP policies, the corresponding reference action $\pi_{\mathrm{ref}}(\innerEmpty{} \mid \ss)$ is
  overlaid in gray, illustrating the increasing deviation from the reference as $\rMax$ grows.
 } \label{fig:PointMassControlTrajectory}
 \vspace{-0.3cm}
\end{figure*}

\subsubsection{Safety Constrained: Cartpole Exploration}
\label{sec:AppDetailsCartpole}
We provide implementation and computational details on the \emph{SafeCartpole}
experiment.

\paragraph{State, Action \& Observation Space.}
In the SafeCartpole environment, the state, action and observation space
are
\begin{align}
 \SS \times \AA \times \Obs = \RR^4 \times \RR \times \RR^5,
\end{align}
where each observation $o_t$ at timestep $t \in [0, \Tmax]$ consists of
\begin{align}
 o_t = [x_t,
 \sin(\omega_t),
 \cos(\omega_t),
 \frac{\Dd}{\dt} x_t,
 \frac{\Dd}{\dt} \omega_t] \in \RR^5,
 \tag{$o_t$-Cartpole}
 \label{eq:CartpoleObservationSpace}
\end{align}
and $\omega_t$ is the current angle of the pole, $x_t$ the slider's
position with its corresponding velocity, and the last component
encodes the angular velocity of the
pole via the corresponding time derivatives of the angular observations.
The state consists of
\begin{align*}
 \ss_t = [x_t,
 \frac{\Dd}{\dt} x_t,
 \omega_t,
 \frac{\Dd}{\dt} \omega_t
 ] \in \RR^4,
\end{align*}
i.e. includes not only the position and the angle, but also
their corresponding derivatives.
Thus, effective exploration of the state action space requires
covering the position and velocity of the cart as well as
the angles of the pole, and also entails the angular velocities of the pole,
making the task significantly harder for constrained maximum entropy.
An action $\aa_t \in \RR$ is a real-valued action which is passed through
a motor gear and then acts on the system.

\paragraph{Objective and Constraint: Definition \& Implementation Details.}
After the state-action occupancy measure is estimated via the MLE in
\eqref{eq:CartpoleStateOccupancyMeasEstimate}, we can estimate the entropy
objective and the constraint.

The constraint function is given by
\begin{align*}
 c(\ss_t, \aa_t) \define \max(\abs{x_t} - 2.0,0),
\end{align*}
where $x_t$ encodes the slider's position at time $t \in [0, \Tmax]$. For the sake of
presentation, we also log the values $\abs{x_t} - 2.0$ to make the constraint
violation plot more expressive.
The entropy and the cost function can be then computed by point-wise multiplications
of the tensors given by: $\lambdaHat$ evaluated on the state and action pairs,
the logarithm of $\lambdaHat$ for the objective and
the constraint function values $(c(\ss_t, \aa_t))_t$ for each rollout,
respectively. For numerical stability, we perform a log-sum-exp rescaling
for these tensor computations. Afterwards we apply the squared penalty,
sum both terms and obtain the pseudo-reward for which we compute the gradient
using PyTorch's \cite{paszkePyTorchImperativeStyle2019a} automatic differentiation
framework. The resulting gradient is plugged in the \eqref{eq:PolicyGradientTheorem}
to compute the cumulative reward and optimize. We use a critic network as a baseline
for the REINFORCE algorithm to improve stability, i.e. \eqref{eq:PolicyGradientTheorem}
acts on the advantage function.

\paragraph{Detailed Analysis of the Results.}
We analyze separately:
\begin{itemize}
 \item \textbf{Entropy and Constraint Violation.}
       In \Cref{fig:CartpolePerformance} we observe a significant jump at $N=1000$
       in the entropy $\HH$ which, combined with
       \Cref{fig:CartpoleTrainingEvolution} and additional logging information,
       corresponds to the policy understanding that symmetric discovery
       drastically improves the exploration in the state-action space. Afterwards,
       we observe a slight drop in the entropy again, as the policy needs to
       relearn, or more precisely combine, to turn the pole upwards and discover
       the space in a symmetric way. As evident in
       \Cref{fig:CartpoleMuSigmaPosAngleOverTime} (bottom left), later episodes
       then learn how to push the penalty function right below the threshold to
       accelerate faster, i.e. get up the pole quicker and then discover even
       more. The final policy has almost perfect symmetry over the rollout horizon
       with varying angular velocities in the lower half of the angular part of
       the state space.
 \item  \textbf{Controls.} \Cref{fig:CartpoleMuSigmaPosAngleOverTime} shows that
       in particular early controls are slightly fuzzy and exhibit high(er) variance.
       In later iterations, i.e. $k \in [1000, 2500] \cap \NN$,
       the policy $\pi^{(k)}$ correctly learns
       a bang-bang control style, which is optimal to turn the pole upward, especially
       since we do not penalize actions in our setting. This behavior is for larger $k$
       refined and the final policy perfectly flips the pole to the upper half of
       the angular's state space and is confirmed as the correct behavior by
       the reduced variance $\sigma^2_{\theta}$, as it also improves on the tiny
       constraint violation at the end of the observation horizon.
\end{itemize}

\begin{figure*}[h!] %
 \centering
 \begin{subfigure}{0.32\textwidth}
  \centering
  \includegraphics[width=\linewidth, keepaspectratio]{
   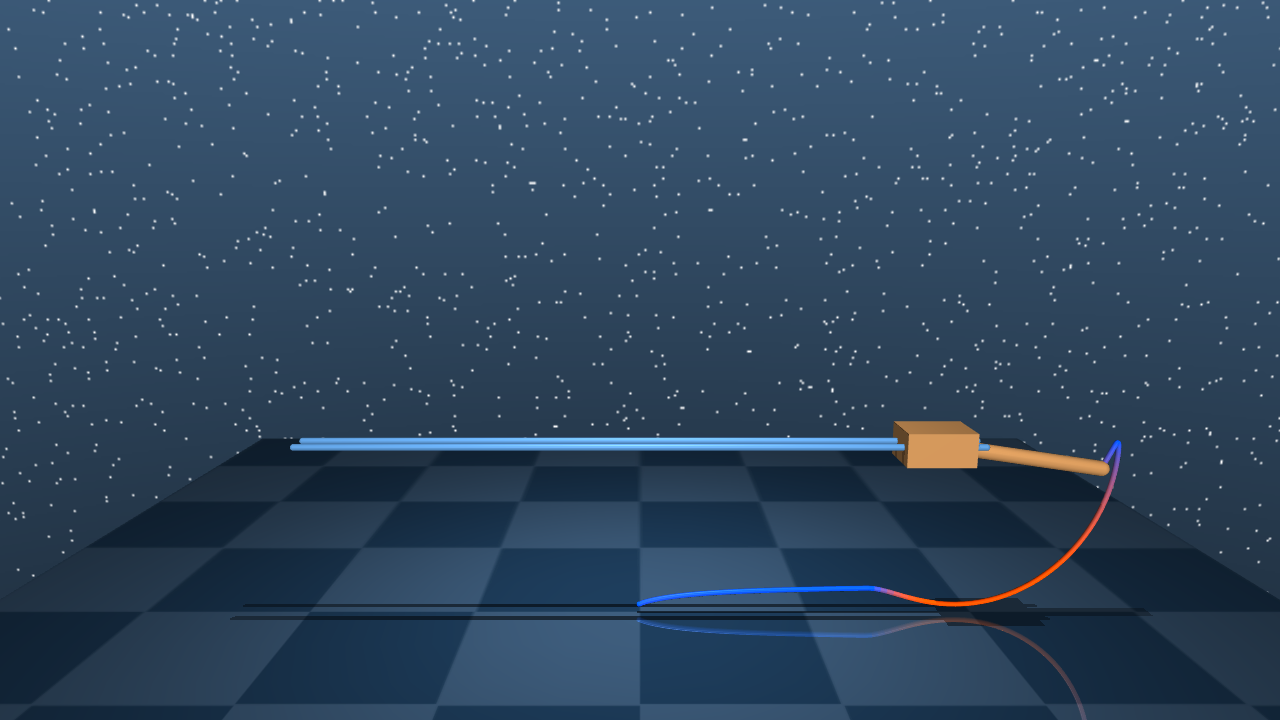
  }
  \caption{$t=4.8\mathrm{s}$}
  \label{fig:CartpoleFrame1}
 \end{subfigure}
 \hfill
 \begin{subfigure}{0.32\textwidth}
  \centering
  \includegraphics[width=\linewidth, keepaspectratio]{
   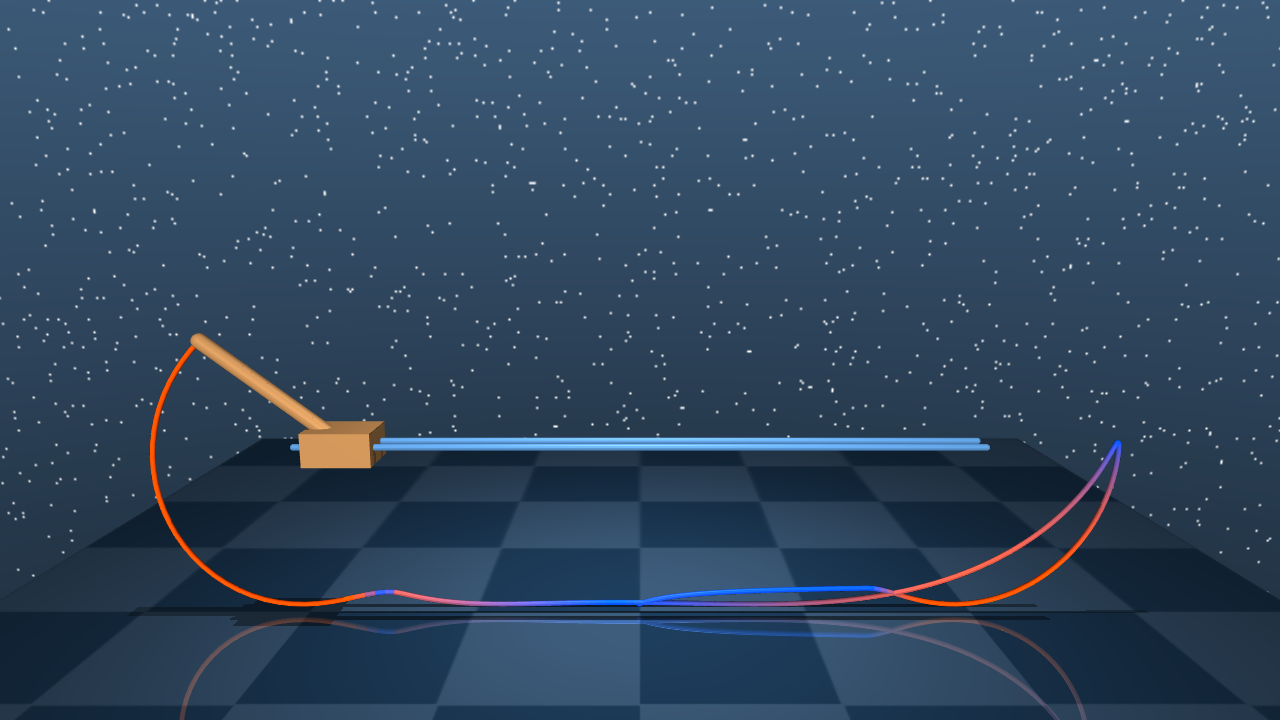
  }
  \caption{$t=7.2\mathrm{s}$}
  \label{fig:CartpoleFrame3}
 \end{subfigure}
 \hfill
 \begin{subfigure}{0.32\textwidth}
  \centering
  \includegraphics[width=\linewidth, keepaspectratio]{
   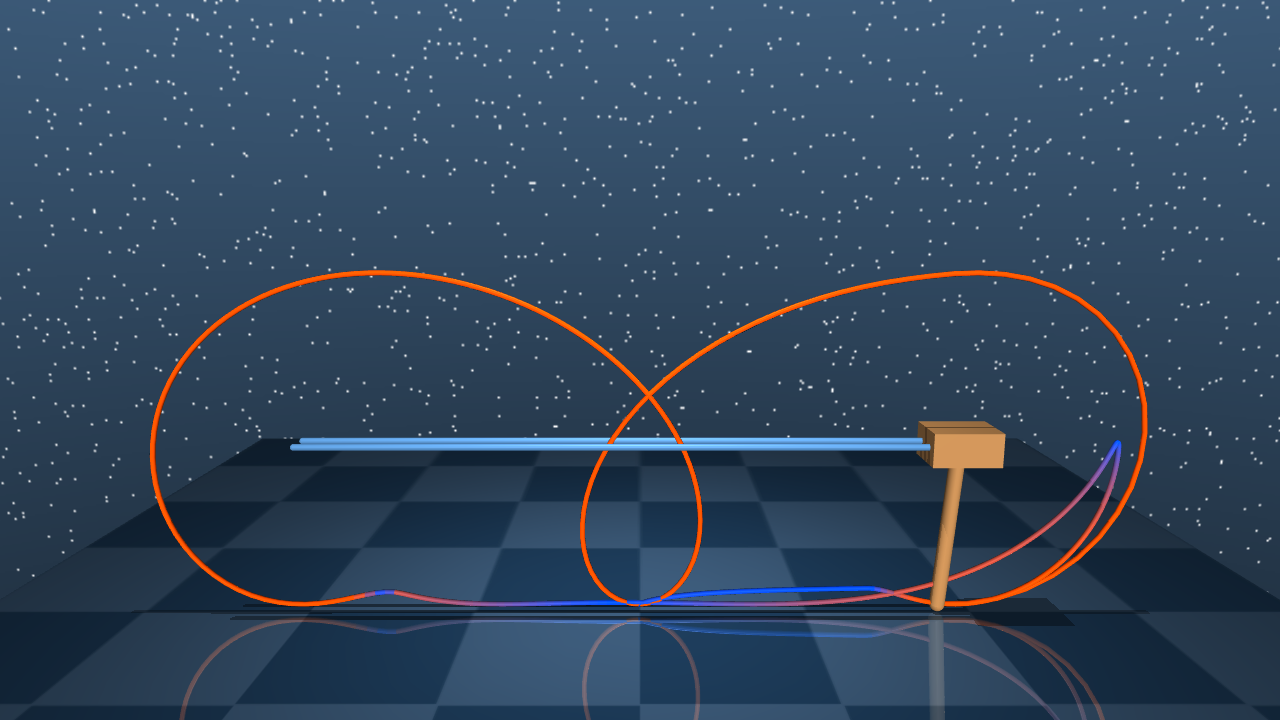
  }
  \caption{$t=9.6\mathrm{s}$}
  \label{fig:CartpoleFrame5}
 \end{subfigure}
 \caption{More fine-grained version of \Cref{fig:Figure1Cartpole}.
  Last iterate policy $\pi_{\theta^{(N)}}$ for
  \textbf{constrained maximum entropy exploration}
  on the SafeCartpole environment. The \textbf{tracer's color} encodes the
  (normalized) \textbf{pole's angular velocity} (dark blue for \markerBlue{lower values} and
  orange for \markerOrange{higher values}) which is part
  of the state space, i.e. has to be covered to maximize $\HH$.
  A more detailed visualization over several training iterations is available
  in \Cref{fig:CartpoleTrainingEvolution}.
  The corresponding \textbf{video} is available in the \textbf{supplementary material}.
 }
 \label{fig:CartpoleFinalPolicy}
\end{figure*}

\begin{figure}[h!] %
 \centering
 \includegraphics[width=\linewidth, keepaspectratio]{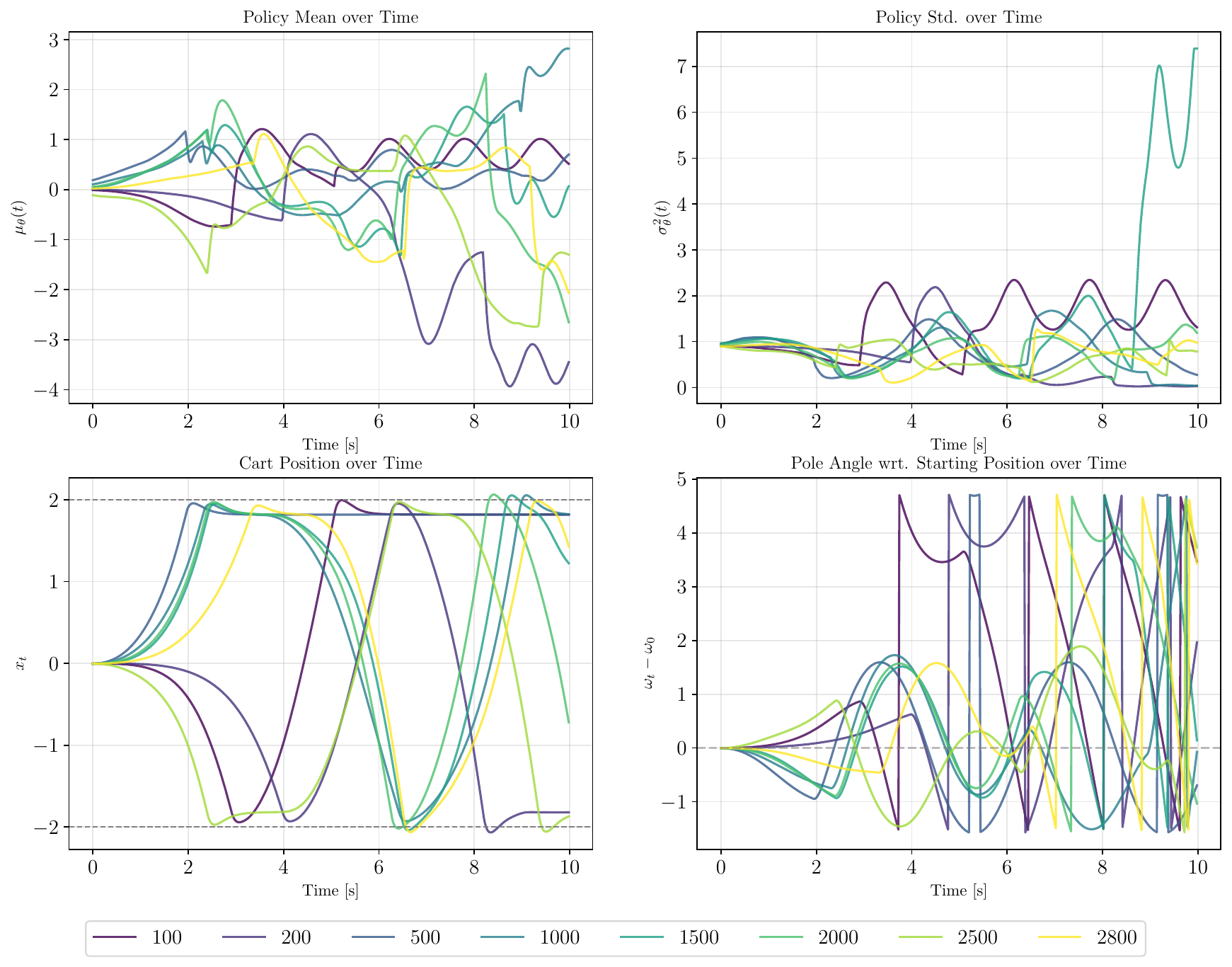}
 \caption{\textbf{Numbers} in the \textbf{legend} indicate the \textbf{iteration numbers}
 used in \Cref{fig:CartpolePerformance} and \Cref{fig:CartpoleTrainingEvolution}
 and correspond to the same policy and the same seed.\\
 \textbf{Top: BoxTanhGaussian-Policy.}
 We visualize the mean $\mu_{\theta}$ (left) and the standard deviation
 $\sigma^2_{\theta}$ (right) of the
 GMM policy $\pi$ over the rollout horizon.
 \\\textbf{Bottom: Cart Position and Pole Angle.} The left-hand side shows the
 cart's position over time with the dashed gray line indicating the threshold
 penalized by the constraint and the right-hand side shows the
 pole's angle relative to the angle of the starting position at $t=0$. Discontinuities
 in the visualization arise since the values are invariant under modulo $2\pi$.
 }
 \label{fig:CartpoleMuSigmaPosAngleOverTime}
\end{figure}

\begin{table}[p]
 \centering
 \scriptsize
 \setlength{\tabcolsep}{3pt}
 \renewcommand{\arraystretch}{0.92}
 \begin{minipage}{.49\linewidth}
  \centering
  \begin{tabular}{@{}ll@{}}
   \toprule
   Parameter             & Value          \\
   \midrule
   Observation Space     & $[-1.0,1.0]^2$ \\
   Simulation $\Delta_t$ & $0.02$         \\
   Control $\Delta_t$    & $0.02$         \\
   Time limit            & $20\mathrm{s}$ \\
   \bottomrule
  \end{tabular}
  \label{tab:PointMassConfig}
  \subcaption{{\textbf{PointMass environment.}}}
 \end{minipage}
 \hfill
 \begin{minipage}{.49\linewidth}
  \centering
  \begin{tabular}{@{}ll@{}}
   \toprule
   Parameter             & Value          \\
   \midrule
   Pole Length           & $[-0.1, 0.1]$  \\
   Gear                  & $[-1.0, 1.0]$  \\
   Slider Pos. Limit     & $[-2.0, 2.0]$  \\
   Simulation $\Delta_t$ & $0.01$         \\
   Control $\Delta_t$    & $0.01$         \\
   Time limit            & $10\mathrm{s}$ \\
   \bottomrule
  \end{tabular}
  \label{tab:CartpoleSwingupConfig}
  \subcaption{{\textbf{SafeCartpole environment.}}}
 \end{minipage}
 \vspace{0.4ex}

 \begin{minipage}[t]{.43\linewidth}
  \centering
  \begin{tabular}{@{}p{.55\linewidth}p{.38\linewidth}@{}}
   \toprule
   Parameter               & Value                        \\\midrule
   Iterations / discount   & $3000$ / $\gamma=0.99$       \\
   Transitions / iteration & 16384                        \\
   $B$                     & 8192                         \\
   GMM update              & 200 steps; batch 512         \\
   Independent evaluation  & Every 25 it.; 16+16 episodes \\
   Gradient clipping       & 1.0                          \\\bottomrule
  \end{tabular}
  \subcaption{\textbf{Training and evaluation.}}
  \label{tab:CartpoleTrainingParameters}
 \end{minipage}
 \hfill
 \begin{minipage}[t]{.55\linewidth}
  \centering
  \begin{tabular}{@{}p{.28\linewidth}p{.66\linewidth}@{}}
   \toprule
   Model                     & Architecture                      \\
   \midrule
   Policy $\pi$              &
   $\mu_{\mathrm{Net}},\Sigma_{\mathrm{Net}}$:
   [Lin(obs\_dim, 256), Lin(256,256), Lin(256,act\_dim)]         \\
   Occ. measure $\lambdaHat$ &
   GMM with $K=16$, $(\sigma_k)_{k\in\setm{K}}\subset[-5.0,2.0]$ \\
   Critic / VF baseline      &
   [Lin(obs\_dim, 256), Lin(256,256), Lin(256,1)]
   \\ \bottomrule
  \end{tabular}
  \label{tab:CartpoleModelArchitectures}
  \subcaption{\textbf{Model architectures.}}
 \end{minipage}
 \caption{{Configurations for the continuous-control experiments in
    \Cref{sec:NumericalExperimentsCartpole}. Environment settings follow the RWRL
    benchmark \cite{dulac-arnoldChallengesRealworldReinforcement2021}. The policy
    uses MLPs for the mean and diagonal covariance; VF denotes the critic
    baseline. PGP and PDPG use matched policy, GMM, transition, and seed budgets.}}
 \label{tab:ConfigurationCartpole}
\end{table}
\begin{figure}[p]
 \centering
 \scriptsize
 \setlength{\tabcolsep}{1.2pt}
 \renewcommand{\arraystretch}{1.0}
  \begin{tikzpicture}
  \node[anchor=south west,inner sep=0] (grid) at (0,0) {
   \includegraphics[width=0.90\textwidth]{
    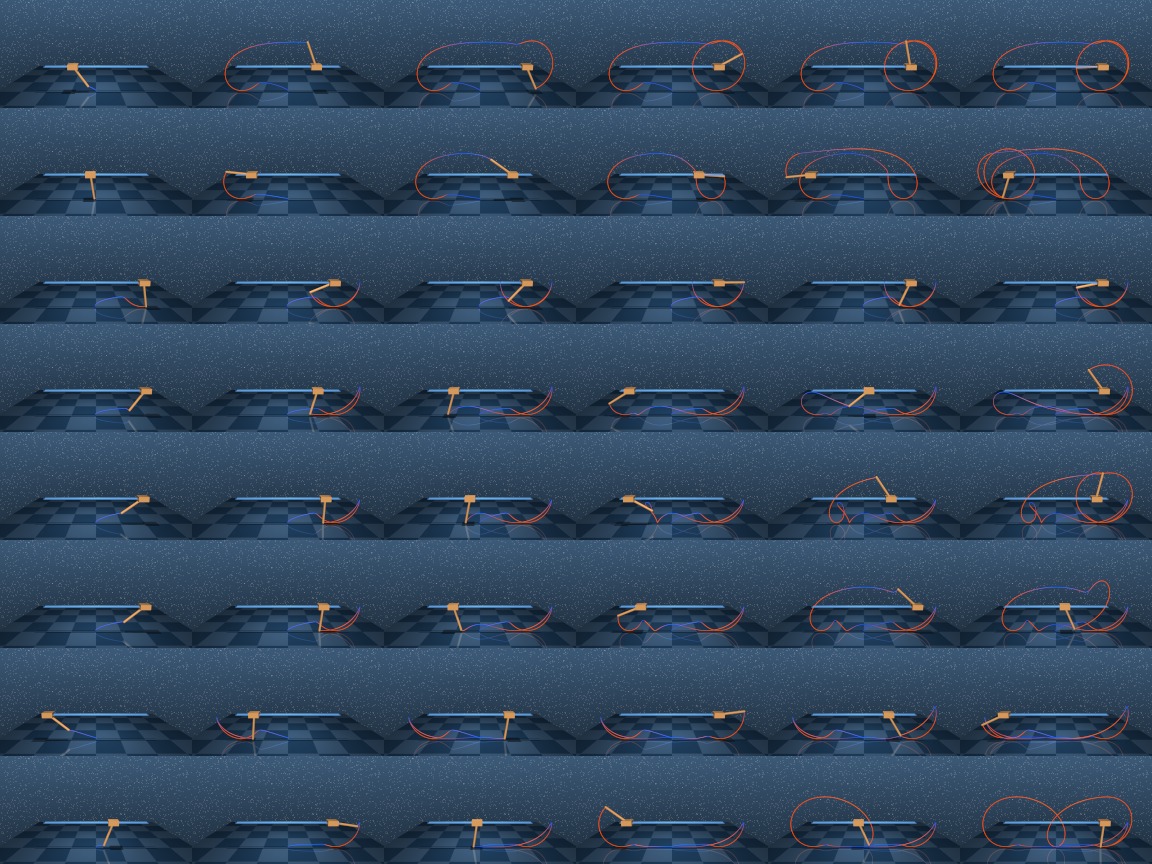}};
  \node[anchor=south east,font=\bfseries\scriptsize]
  at ($(grid.north west)+(-0.45em,0.35em)$) {Iteration};
  \node[anchor=south,font=\bfseries\scriptsize]
  at ($(grid.north west)!0.08333!(grid.north east)+(0,0.35em)$)
  {$t=2.4\,\mathrm{s}$};
  \node[anchor=south,font=\bfseries\scriptsize]
  at ($(grid.north west)!0.25!(grid.north east)+(0,0.35em)$)
  {$t=4.8\,\mathrm{s}$};
  \node[anchor=south,font=\bfseries\scriptsize]
  at ($(grid.north west)!0.41667!(grid.north east)+(0,0.35em)$)
  {$t=6.0\,\mathrm{s}$};
  \node[anchor=south,font=\bfseries\scriptsize]
  at ($(grid.north west)!0.58333!(grid.north east)+(0,0.35em)$)
  {$t=7.2\,\mathrm{s}$};
  \node[anchor=south,font=\bfseries\scriptsize]
  at ($(grid.north west)!0.75!(grid.north east)+(0,0.35em)$)
  {$t=8.4\,\mathrm{s}$};
  \node[anchor=south,font=\bfseries\scriptsize]
  at ($(grid.north west)!0.91667!(grid.north east)+(0,0.35em)$)
  {$t=9.6\,\mathrm{s}$};
  \node[anchor=east,font=\bfseries\scriptsize]
  at ($(grid.south west)!0.9375!(grid.north west)+(-0.45em,0)$) {100};
  \node[anchor=east,font=\bfseries\scriptsize]
  at ($(grid.south west)!0.8125!(grid.north west)+(-0.45em,0)$) {200};
  \node[anchor=east,font=\bfseries\scriptsize]
  at ($(grid.south west)!0.6875!(grid.north west)+(-0.45em,0)$) {500};
  \node[anchor=east,font=\bfseries\scriptsize]
  at ($(grid.south west)!0.5625!(grid.north west)+(-0.45em,0)$) {1000};
  \node[anchor=east,font=\bfseries\scriptsize]
  at ($(grid.south west)!0.4375!(grid.north west)+(-0.45em,0)$) {1500};
  \node[anchor=east,font=\bfseries\scriptsize]
  at ($(grid.south west)!0.3125!(grid.north west)+(-0.45em,0)$) {2000};
  \node[anchor=east,font=\bfseries\scriptsize]
  at ($(grid.south west)!0.1875!(grid.north west)+(-0.45em,0)$) {2500};
  \node[anchor=east,font=\bfseries\scriptsize]
  at ($(grid.south west)!0.0625!(grid.north west)+(-0.45em,0)$) {2800};
 \end{tikzpicture}
 \captionsetup{font=small}
 \caption{\textbf{Policy training evolution on SafeCartpole for a fixed seed.}
  Rows show training iterations and columns show trajectory times. The policy
  progressively learns to swing up the pole while respecting the safety
  constraint, cf. \Cref{fig:CartpoleMuSigmaPosAngleOverTime}. The tracer color
  encodes normalized pole angular velocity
  (\markerBlue{lower} to \markerOrange{higher}); corresponding videos are
  available in the supplementary material.}
 \label{fig:CartpoleTrainingEvolution}
\end{figure}
\clearpage

\end{document}